\documentclass{article}

\PassOptionsToPackage{numbers, sort&compress}{natbib}

 \usepackage[preprint]{neurips_2026}

\usepackage[utf8]{inputenc} %
\usepackage[T1]{fontenc}    %
\usepackage{hyperref}       %
\usepackage{url}            %
\usepackage{booktabs}       %
\usepackage{amsfonts}       %
\usepackage{nicefrac}       %
\usepackage{microtype}      %
\usepackage{xcolor}         %

\usepackage{graphicx,caption}
\usepackage{amsmath,amsfonts,amssymb}
\usepackage{xparse}
\usepackage{empheq}
\usepackage{mathtools}
\usepackage{svg}
\usepackage{enumitem}
\usepackage{algorithm}
\usepackage{algpseudocode}

\usepackage{cleveref}

\usepackage{amsthm}

\newtheorem{lemma}{Lemma}

\newtheorem{assumption}{Assumption}
\newtheorem{remark}{Remark}

\newtheorem{corollary}{Corollary}

\usepackage{thmtools}
\usepackage{thm-restate}

\crefname{thm}{Theorem}{Theorems}
\Crefname{thm}{Theorem}{Theorems}
\crefname{definition}{Definition}{Definitions}
\Crefname{definition}{Definition}{Definitions}
\crefname{proposition}{Proposition}{Propositions}
\Crefname{proposition}{Proposition}{Propositions}
\crefname{lemma}{Lemma}{Lemmas}
\Crefname{lemma}{Lemma}{Lemmas}
\crefname{corollary}{Corollary}{Corollaries}
\Crefname{corollary}{Corollary}{Corollaries}
\crefname{assumption}{Assumption}{Assumptions}
\Crefname{assumption}{Assumption}{Assumptions}
\crefname{remark}{Remark}{Remarks}
\Crefname{remark}{Remark}{Remarks}
\crefname{example}{Example}{Examples}
\Crefname{example}{Example}{Examples}
\crefname{algorithm}{Algorithm}{Algorithms}
\Crefname{algorithm}{Algorithm}{Algorithms}

\AddToHook{cmd/appendix/before}{%
    \crefalias{section}{appendix}%
    \crefalias{subsection}{appendix}
}

\newenvironment{proofsketch}{%
  \proof}{\endproof}

\usepackage[textsize=tiny]{todonotes}

\definecolor{blind_blue}{HTML}{547FEF}
\definecolor{blind_magenta}{HTML}{DC267F}
\definecolor{blind_orange}{HTML}{FE6100}
\newcommand{\colourbase}{black}
\newcommand{\applycolor}[1]{\textcolor{\colourbase}{#1}}

\newcommand{\mymacro}[3][0]{%
  \newcommand{#2}[#1]{{\applycolor{#3}}}%
}

\newcommand{\defn}[1]{\textbf{#1}}
\newcommand{\mb}{\mathbf}
\newcommand{\mc}{\mathcal}

\title{Invariant Learning Dynamics of Transformers \\ in Inductive Reasoning Tasks}

\author{%
  Tiberiu Musat\thanks{Correspondence at \texttt{tiberiu@musat.ai}, \texttt{tiago.pimentel@inf.ethz.ch}} \\
  ETH Zurich
  \And
  Tiago Pimentel \\
  ETH Zurich
  \And
  Nicolas Zucchet \\
  Stanford University
  \And
  Thomas Hofmann \\
  ETH Zurich
}

\newcommand{\R}{\mathbb{R}} %

\newcommand{\tp}[1]{{#1}^{\intercal}} %

\mymacro{\setAssoc}{\mathcal{R}}
\mymacro{\setAssocC}{\mathcal{R}^c}
\mymacro{\arity}{\rho} %

\mymacro{\vocabPairs}{k}
\mymacro{\vocab}{\mathcal{V}}
\mymacro[1]{\tok}{\mathbf{#1}}
\mymacro{\token}{\tok{x}}
\mymacro{\tokSeq}{\mathbf{w}}

\mymacro{\vocabC}{\vocab^c}
\mymacro{\vocabR}{\vocab^r}

\mymacro{\itm}{\mathbf{a}}
\mymacro{\lbl}{\mathbf{b}}

\mymacro{\embedSymbol}{\mathrm{Embed}}
\mymacro[1]{\embed}{\embedSymbol[#1]}
\mymacro[1]{\emb}{\mathbf{e}_{#1}}
\mymacro[1]{\embedPerturbed}{\mathrm{Embed_w}[#1]}
\mymacro{\embedDim}{d}
\mymacro{\headDim}{d_h}
\mymacro{\embedSpace}{\R^d}
\mymacro{\embedMatrix}{\mathbf E}

\mymacro{\seqPairs}{n}
\mymacro{\queryPos}{2\seqPairs + 1}

\mymacro{\seqEmbed}{\mathbf{T}}
\mymacro{\posEmbed}{\mathbf{P}}
\mymacro{\modelInput}{\mathbf{X}}
\mymacro{\modelOutput}{\mathbf{y}}
\mymacro[1]{\hidden}{\mathbf{H}^{(#1)}}
\mymacro[1]{\hiddenlast}{\mathbf{H}^{(#1)}_{:,\seqLen}}
\mymacro[2]{\hiddeni}{\mathbf{H}^{(#1)}_{:,#2}}
\mymacro{\hiddenOut}{\Bar{\mathbf{H}}}
\mymacro{\outputPredicted}{\mathbf{z}}
\mymacro{\loss}{\mathcal{L}}

\mymacro{\seqEmbedT}{\tilde{\mathbf{T}}}
\mymacro{\posEmbedT}{\tilde{\mathbf{P}}}
\mymacro{\modelInputT}{\tilde{\mathbf{X}}}
\mymacro{\modelOutputT}{\tilde{\mathbf{y}}}
\mymacro[1]{\hiddenT}{\tilde{\mathbf{H}}^{(#1)}}
\mymacro{\hiddenOutT}{\tilde{\Bar{\mathbf{H}}}}
\mymacro{\outputPredictedT}{\tilde{\mathbf{z}}}
\mymacro{\lossT}{\tilde{\mathcal{L}}}

\mymacro{\seqLen}{N}

\mymacro{\seqEmbedSpace}{\R^{\embedDim \times \seqLen}}
\mymacro{\posEmbedSpace}{\R^{\embedDim \times \seqLen}}
\mymacro{\modelInputSpace}{\R^{\embedDim \times \seqLen}}
\mymacro{\outputSpace}{\embedSpace}
\mymacro{\hiddenSpace}{\R^{\embedDim \times \seqLen}}

\mymacro[1]{\weights}{\mathbf{W}^{(#1)}}
\mymacro[1]{\weightsQ}{\mathbf{W}^{(#1)}_{Q}}
\mymacro[1]{\weightsK}{\mathbf{W}^{(#1)}_{K}}
\mymacro[1]{\weightsV}{\mathbf{W}^{(#1)}_{V}}
\mymacro[1]{\weightsP}{\mathbf{W}^{(#1)}_{P}}
\mymacro[1]{\values}{\mathbf{V}^{(#1)}}
\mymacro{\weightSpace}{\R^{\embedDim \times \embedDim}}
\mymacro{\valueSpace}{\R^{\embedDim \times \embedDim}}
\mymacro{\weightsThirdSpace}{\R^{\embedDim \times \embedDim}}

\mymacro{\weightsOut}{\Bar{\mathbf{W}}}
\mymacro{\weightsOutSubSpace}{\Bar{\mathbb{S}}}

\mymacro[1]{\weightsICL}{\mathbf{W}^{(#1)}_{\mathrm{ICL}}}
\mymacro{\weightsOutICL}{\Bar{\mathbf{W}}_{\mathrm{ICL}}}
\mymacro{\out}{y}

\mymacro{\id}{\mathbf{I}}
\mymacro{\idR}{\mathbf{I}^{(r)}}
\mymacro{\idC}{\mathbf{I}^{(c)}}
\mymacro{\idP}{\mathbf{I}^{(p)}}
\mymacro{\idT}{\mathbf{I}^{(t)}}
\mymacro{\commCorrel}{\mathbf{C}}
\mymacro[1]{\weightsSubSpace}{\mathbb{S}^{(#1)}}

\mymacro{\posEmbedSet}{\mathcal{P}}
\mymacro{\posCorrel}{\mathbf{M}}
\mymacro{\correlSpace}{\R^{\embedDim \times \embedDim}}

\mymacro{\orthogonalMatrix}{\mathbf{E}}

\NewDocumentCommand{\softmax}{g}{
  \applycolor{\sigma
  \IfValueT{#1}{\left( #1 \right)}}
}

\mymacro{\inputSeq}{\mathbf{w}^\mathrm{in}}
\mymacro{\outputSeq}{\mathbf{w}^\mathrm{out}}

\mymacro{\oneInput}{\mathbf{x}}
\mymacro{\oneHidden}{\mathbf{h}}
\mymacro{\rank}{\mathrm{rank}}
\mymacro{\vecSpan}{\mathrm{span}}
\mymacro{\tokDim}{d_\mathrm{t}}
\mymacro{\posDim}{d_\mathrm{p}}

\mymacro{\tokSpace}{\mathbb{U}^{\mathrm{(t)}}}
\mymacro{\tokCommSpace}{\mathbb{U}^{\mathrm{(c)}}}
\mymacro{\tokRareSpace}{\mathbb{U}^{\mathrm{(r)}}}
\mymacro{\posSpace}{\mathbb{U}^{\mathrm{(p)}}}

\mymacro{\identity}{\mathbf{I}}
\mymacro{\zero}{\mathbf{0}}

\mymacro{\basisTok}{\mathbf{Q}^\mathrm{(t)}}
\mymacro{\basisPos}{\mathbf{Q}^\mathrm{(p)}}
\mymacro{\basisTokSpace}{\mathbf{Q}^\mathrm{x}}

\mymacro[1]{\permGroup}{S_{#1}}
\mymacro{\permC}{\pi_C}
\mymacro{\permR}{\pi_R}
\mymacro{\perm}{\pi}

\mymacro{\densityRare}{f_R}
\mymacro{\distToksRare}{\mathcal{D}^r}
\mymacro{\tokenR}{\token_r}
\mymacro{\tokenC}{\token_c}

\mymacro{\embedR}{\mathbf{e}^r}

\mymacro{\numFreqs}{m}
\mymacro{\freq}{\omega}
\mymacro{\offset}{\phi}
\mymacro{\offsetSpace}{\R^{\seqPairs}}
\mymacro{\sinBasis}{\mathbf{q}}
\mymacro{\cosBasis}{\mathbf{r}}

\mymacro{\imgCtxLen}{T}
\mymacro{\numLayers}{L}
\mymacro{\numHeads}{H}

\newcommand*{\circled}[1]{\tikz[baseline=(char.base)]{        \node[shape=circle,draw,inner sep=1pt] (char) {\normalfont{\small #1}};}}

\begin{document}

\maketitle

\begin{abstract}

    We present a theoretical framework to explain the emergence of inductive
    reasoning abilities in Transformer language models. While previous works on
    Transformer learning dynamics have so far been mostly tied to specific tasks,
    we study a generalized class of inductive tasks that unifies several synthetic
    tasks known in the literature, including in-context $n$-grams and multi-hop
    reasoning. In this class, we theoretically prove that the training dynamics of
    attention models can be confined to a highly interpretable, low-dimensional invariant
    manifold. On this manifold, the learning dynamics are captured by a handful of
    interpretable coordinates rather than millions of parameters, making both
    theoretical and empirical analysis more tractable. Using this framework, we
    characterize how data statistics govern the competition between in-context and
    in-weights learning, we study how random initializations determine the
    `winning' circuit when multiple solutions are possible, and we demonstrate that
    the coordinate frame associated with the manifold can be used to automatically
    detect which circuits have been learned in trained models. By casting circuit
    formation as a low-dimensional dynamical phenomenon, we take a step toward a
    predictive theory of how Transformers learn.%
\end{abstract}

\begin{figure}[h]
    \centering
    \captionsetup{width=\linewidth}
    \includegraphics[width=\linewidth]{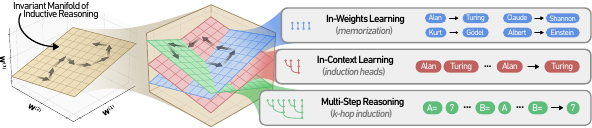}
    \caption{The \emph{Invariant Manifold of Inductive Reasoning} (IMIR) is an interpretable low-dimensional subspace of the parameter space which training trajectories never leave. Each induction circuit resides and evolves within a subspace of the IMIR, depicted schematically as a colored plane.}
    \label{fig:circuit-space}
\end{figure}

\section{Introduction}

While great progress has been made toward understanding the circuits (i.e.
subnetworks, or weight structures) present in large language models
\citep{zheng2024attention, olah2020zoom, wang2022interpretability,
    conmy2023towards},
we are still far from understanding the learning dynamics that underlie their
formation. %
As a result, skill acquisition in language models remains mostly unpredictable
\citep{wei2022emergent}, which complicates model development and poses AI
safety issues. %
A better understanding of circuit emergence could reduce these risks, provide
practical insights into how to interpret the internal mechanics of trained
models, as well as help design better training pipelines
\citep{simon2026there}.
In this work, we focus on the emergence of circuits responsible for
\defn{inductive abilities}: general-purpose capabilities that allow a language
model to recognize patterns, infer rules, and perform abstract reasoning on the
fly based on the provided context, rather than relying on static, memorized
patterns.\looseness=-1

Inductive abilities of transformers have been studied extensively by prior work, both empirically and theoretically. 
However, inductive abilities are typically studied in isolation, limited to specific tasks such as in-context $n$-grams \citep{akyurek2024context, edelman2024evolution, varre2025learning} or $k$-hop induction \citep{sanford2024transformers, musat2024mechanism, wang2024buffer, allen-zhu2026physics}.
Without a unifying theoretical framework, our understanding progresses slowly,
and three important phenomena remain particularly puzzling.
\circled{1}~\textbf{Circuit~competition} is a fundamental aspect of language model training
\citep{singh2025strategy,merrill2023tale,zhong2023clock}. However, most previous theoretical works assume staged
learning algorithms where only one component is trained at a time, while others are kept frozen \citep{bietti2024birth, nichani2024transformers}. A few works study the dynamical interaction of different components during training, but are limited to small transformers with only one \citep{zucchet2025emergence} or two \citep{musat2026emergence} attention heads. \circled{2}~\textbf{Data distributional properties}
strongly influence the emergence of in-context learning \citep{chan2022data,
    kim2025training, singh2023transient}, but a mechanistic understanding is
currently limited to phenomenological models \citep{reddy2023mechanistic}.
\circled{3}~\textbf{Randomly initialized} networks contain high-performing
subnetworks that can be trained in isolation, according to the \emph{lottery
    ticket hypothesis} \citep{frankle2018lottery}, but we currently lack a
mechanistic understanding of why specific initializations yield `winning'
tickets.

Our paper's central contribution is the \textbf{Invariant Manifold of Inductive
    Reasoning (IMIR)}: a low-dimensional subspace of
the parameter space which exhibits self-contained learning dynamics. 
A model whose parameters belong to the IMIR will remain on the IMIR when updated via gradient descent, a behavior known in dynamical systems theory as an \emph{invariant manifold}. 
We theoretically prove the existence of the IMIR and provide its mathematical description (\cref{sec:theory}) for a generalized class of induction tasks (\cref{sec:unified-class}). 
Importantly, our framework considers all parameter interactions in a model during training, making it more tractable to study phenomena like circuit competition. 
Moreover, each dimension of the IMIR is highly interpretable, corresponding to a specific elementary action within a specific transformer component. 
This allows us to find reduced and interpretable descriptions of the transformer circuits responsible for inductive reasoning abilities.

In \cref{sec:circuit-competition}, we leverage our framework to conduct a
circuit competition study (addressing~\circled{1}). Concretely, we first
investigate the circuit competition between in-context learning and in-weights
learning (\cref{sec:icl-vs-iwl}), where we theoretically prove the data dependency of
in-context learning (addressing~\circled{2}). Second, we investigate the impact
of random initialization on the competition between several in-context learning
circuits within a single transformer model; we find that the winning circuit
is determined by the model's initialization, with sharp phase transitions between the regions of initialization space favoring each circuit (partially addressing~\circled{3} in this simplified setup). 
In
\cref{sec:autointerp}, we show how the highly interpretable structure of the
IMIR can be used for automatic detection of inductive circuits. Finally, in
\cref{sec:discussion}, we discuss other possible applications of our framework
and future research directions.

Overall, our work provides a theoretical framework for understanding the
emergence of inductive circuits and opens the door for future work on the
learning dynamics of large language models.

\section{A Unified Class of Inductive Tasks}
\label{sec:unified-class}

\mymacro{\ctx}{\mathbf{x}}
\mymacro{\tupleSize}{k}

A great number of inductive abilities have been studied in the existing
literature \citep{olsson2022context, sanford2024transformers,
    akyurek2024context, wang2022interpretability, musat2024mechanism,
    alur2025impossibility}. 
However, models' inductive abilities are usually investigated
separately across these tasks, which makes it difficult to arrive at general results. We show that
many inductive tasks, while being relatively different at the surface level,
all share some intrinsic structure requiring in-context manipulation of token
tuples. This will enable us to derive general results that apply to a large
variety of tasks.
We describe our task class in \cref{sec:task-class} and we give some
examples of induction tasks. Then, we describe how the task depends on token relations and block positions in \cref{sec:data-invariance}.

\subsection{The Block-List Task Class}
\label{sec:task-class}

In natural language next-token prediction tasks, most words depend only on a tiny fraction of their context: 
the immediately preceding block of words and a few relevant blocks spread throughout its (potentially much larger) context. 
These relevant blocks may correspond to sentences, clauses, phrases, or even short spans such as consecutive word pairs.
Inspired by this observation about the dependency structure in natural language, we now
define a general class of \defn{block-list tasks} where a next token must be
predicted based on a list of $\seqPairs$ blocks of at most $\tupleSize$ tokens
each:
\begin{equation}
    [\ctx^1_{1} \ldots \ctx^1_{\tupleSize_1}]
    \;
    [\ctx^2_{1} \ldots \ctx^2_{\tupleSize_2}]
    \;\ldots\;
    [\ctx^{\seqPairs}_{1} \ldots \ctx^{\seqPairs}_{\tupleSize_\seqPairs}]
    \;\;\longrightarrow\;\;
    \ctx^{\seqPairs}_{\tupleSize_\seqPairs + 1}
\end{equation}
where $[\ctx^{\seqPairs}_{1} \ldots \ctx^{\seqPairs}_{\tupleSize_\seqPairs}]$ denotes the block of tokens immediately preceding the token to be predicted. The other blocks $[\ctx^{i}_{1} \ldots \ctx^{i}_{\tupleSize_i}]$, with $i < \seqPairs$, denote potentially (but not necessarily) relevant blocks of tokens. 
Under our task's definition, each $\tupleSize_i$ satisfies $\tupleSize_i \leq \tupleSize$, and we denote the language vocabulary as $\vocab$, so that $[\ctx^{i}_{1} \ldots \ctx^{i}_{\tupleSize_i}] \in \vocab^{\tupleSize_i}$. 
Many important next-token prediction tasks fit this description, and we now give three examples.

\label{sec:example-tasks}

\paragraph{In-context associative recall \citep{olsson2022context, nichani2024transformers, bietti2024birth}.}%
In this task, we have sequences with the pairwise structure $[\ctx^1_{1} \ctx^1_{2}]
    \;
    [\ctx^2_{1} \ctx^2_{2}]
    \;\ldots\;
    [\ctx^{\seqPairs}_{1}]
    \rightarrow\
    \ctx^{\seqPairs}_{2}$
and must predict token $\ctx^{\seqPairs}_{2}$ by finding a previous occurrence\footnote{If multiple occurrences are allowed, a common approach is to predict the same distribution over the next token.} $\ctx^i_1$ of token $\ctx^{\seqPairs}_{1}$ with $i<\seqPairs$ in the sequence and copying the token following it (i.e., $\ctx^i_{2}$).
In our general formulation, we have blocks of two tokens in the context
($\tupleSize_i=2$ for $i<\seqPairs$), while the query token forms a
single-token block. Notably, this is one of the most studied inductive tasks.

\paragraph{$k$-hop induction \citep{sanford2024transformers, musat2024mechanism, wang2024buffer,allen-zhu2026physics}.}
This task generalizes the in-context associative recall task by considering chained pairs of associations.
It has a similar structure to the previous task; however, once $\ctx^i_{1}$ is identified in the context, we may use $\ctx^i_{2}$ for a second-lookup operation.
Formally, we must find a chain of $k$ bigrams such that:
$(\ctx^{\seqPairs}_{1} = \ctx^{i_1}_{1}) \rightarrow (\ctx^{i_1}_{2} = \ctx^{i_2}_{1}) \rightarrow (\ctx^{i_2}_{2} = \ctx^{i_3}_{1}) \cdots \rightarrow \ctx^{i_k}_{2}$.

\paragraph{In-context $n$-grams \citep{akyurek2024context, edelman2024evolution, varre2025learning}.}
A different way to generalize the in-context associative recall task is by considering relations between $n$-grams of arbitrary length $\ell$.
In this task, we thus have sequences with the structure
$[\ctx^1_{1} \ctx^1_{2} \cdots \ctx^1_{\ell}]\;
    [\ctx^2_{1} \ctx^2_{2} \cdots \ctx^2_{\ell}]
    \;\ldots\;
    [\ctx^{\seqPairs}_{1} \cdots \ctx^{\seqPairs}_{\ell-1}]
    \rightarrow\
    \ctx^{\seqPairs}_{\ell}$
and must predict token $\ctx^{\seqPairs}_{\ell}$ by finding a previous occurrence of its entire prefix $\ctx^{\seqPairs}_{1} \cdots \ctx^{\seqPairs}_{\ell-1}$ in the sequence.

Several other tasks also fit this general block-list structure, among them:
indirect object identification \citep{wang2022interpretability}, conditional
retrieval \citep{musat2024mechanism}, inverting permutations
\citep{alur2025impossibility}, the in-context recall tasks used to study the emergence of sparse attention \citep{zucchet2025emergence},
in-context language learning \citep{akyurek2024context}, and reasoning with
fragments of natural language \citep{schlegel-etal-2022-transformers}. We
present these tasks in \cref{fig:block_list}.

\begin{figure}[t]
    \centering
    \includegraphics[width=\linewidth]{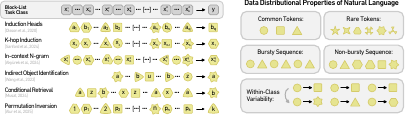}
    \caption{We define a block-list class of language tasks where the next token must be predicted based on a list of token tuples. We then study the training dynamics in this class of tasks.}
    \label{fig:block_list}
\end{figure}

\subsection{Data Symmetries}
\label{sec:data-invariance}

In inductive reasoning tasks, predictions typically depend on high-level structure, rather than on arbitrary choices of tokens or absolute positions.
We now describe a few symmetries present in our data, which may impose such structure.

\paragraph{Token Symmetries.} What a model must learn in an inductive task is the relation between tokens, not the identity of any particular token. For instance, in the associative recall task of \cref{sec:task-class}, the rule ``return the token that previously followed the query'' is unaffected if we consistently relabel the vocabulary: the instance $\mathrm{a\,b}\,\ldots\,\mathrm{a} \rightarrow \mathrm{b}$ and its relabeling $\mathrm{u\,v}\,\ldots\,\mathrm{u} \rightarrow \mathrm{v}$ are solved by the same mechanism. A model that depends on this relational structure, rather than on specific token identities, is precisely the kind of model that generalizes to tokens unseen during training, which is a defining feature of in-context learning. We formalize this property as an invariance of the data distribution under a group of token relabelings. Consider a valid sequence of tokens $\ctx \in \vocab^\seqLen$, where $\seqLen = \sum_{i=1}^{\seqPairs} \tupleSize_i$. Then, applying an admissible token permutation $\perm \in S_\vocab$---where $S_\vocab$ denotes the symmetric group on $\vocab$, i.e., the set of all bijections $\perm : \vocab \to \vocab$---should result in another valid sequence $\perm(\ctx) = \perm(\ctx_1)\ldots\perm(\ctx_\seqLen)$. However, what counts as an admissible permutation?

We define admissible token permutations based on the relational structure that a task imposes on the underlying vocabulary. Not every relabeling preserves the task: when the prediction relies on fixed relations between tokens, such as memorized associations, a permutation that does not respect those relations yields a different task. The admissible permutations are therefore exactly those that preserve this structure. 
Let $\setAssoc \subseteq \vocab^\arity$ be an $\arity$-ary relation---i.e., a set of $\arity$-tuples of tokens---characterizing basic skills required to solve the task.
For instance, in the associative recall task, $\setAssoc$ is a set of associated token pairs (a binary relation, $\arity = 2$, formalized below); in a task whose context encodes memorized transitions of a Markov chain, $\setAssoc$ is the set of allowed transitions $(\token, \token')$.
The relation $\setAssoc$ induces a subgroup of the symmetric group $S_{\vocab}$ called the automorphism group of $\setAssoc$, which is defined as
\begin{equation}
\operatorname{Aut}(\vocab, \setAssoc) := \{ \perm \in S_\vocab \;:\; \mathbf{w} \in \setAssoc \Leftrightarrow \perm(\mathbf{w}) \in \setAssoc \}.
\end{equation}
Its elements are precisely the permutations preserving the relational structure of $\setAssoc$.

To simplify the exposition throughout the remainder of this work, we specialize $\setAssoc$ to be an irreflexive symmetric binary relation of distinct, \emph{unordered} token pairs:
\begin{align}
    \setAssoc =
    \{  \{\itm_1,\lbl_1\},  \dots, \{\itm_\vocabPairs,\lbl_\vocabPairs\} \},
    \qquad \itm_i,\lbl_i \in \vocab,
\end{align}
where, for simplicity of presentation, we assume for now that every token belongs to exactly one association, i.e., $\bigcup_{i=1}^{\vocabPairs} \{ \itm_i,\lbl_i\} = \vocab$ and $2\vocabPairs = |\vocab|$.\footnote{When we later distinguish common from rare tokens (\cref{sec:data-properties}), only common tokens carry associations and we instead have that $\bigcup_{i=1}^{\vocabPairs} \{ \itm_i,\lbl_i\} = \vocabC$, with rare tokens remaining unpaired.}
We show how our theory can be extended to arbitrarily complex task structures in \cref{sec:relational-structure}.

\paragraph{Positional Symmetries.} In natural language, the
relevant blocks $[\token_1^i \cdots \token_{k_i}^i]$ may appear at any position
in a document, separated by an arbitrary number of irrelevant
words. We can simulate this effect by giving each block a random positional offset $\offset_i \in \mathbb{N}$, as if blocks were spread out in a very large hypothetical context. We can ensure that blocks don't overlap by enforcing $\offset_{i+1} \ge \offset_i + \tupleSize_i$. Then, the positional embedding of the token $\ctx^i_j$ is determined by its shifted position $\offset_i + j$. This generalizes the encoding of positions and includes the special case with no offset: $\offset_{i+1} = \offset_i + \tupleSize_i$.

\section{Invariant Manifolds of Learning Dynamics}
\label{sec:theory}

In this section, we introduce our key theoretical contribution, the identification and characterization of the \defn{Invariant Manifold of Inductive Reasoning (IMIR)}, a low-dimensional subspace to which parameters remain confined when trained via gradient descent on any block-list task.
In \cref{sec:architecture-main}, we describe the transformer architecture and training setup considered throughout the paper.
In \cref{sec:IMIR-theory}, we present our main theorem along with a brief proof
sketch. In \cref{sec:imir-interp}, we provide an interpretation of the IMIR's
structure.\footnote{
We note that \cref{thm:invariant-manifold} establishes that the IMIR is \emph{invariant}: a trajectory initialized on the IMIR remains on it throughout training.
We do not prove the stronger property that the IMIR is \emph{attractive}; we only show this empirically in our experiments.
}

\subsection{Transformer Architecture}
\label{sec:architecture-main}

Our models are based on an attention-only transformer, which we describe here.
Let $\seqLen = \sum_{i=1}^{\seqPairs} \tupleSize_i$ denote the length of the
input sequence. First, we map every token $\token \in \vocab$ to an embedding
vector $\emb{\token} \in \embedSpace$. Applying this embedding to our input
sequence and target output, we obtain the model input $\seqEmbed \in \R^{\embedDim \times \seqLen}$ and output $\modelOutput \in \outputSpace$. Second, we use sinusoidal positional
embeddings $\posEmbed \in \posEmbedSpace$, following
\citet{vaswani2017attention}, with random block-level offsets as described in \cref{sec:data-invariance}. As is standard, we then combine token
and position embeddings to obtain the model inputs: $\hidden{0} = \seqEmbed +
    \posEmbed$.

We pass these inputs through a multi-head attention-only
transformer with $\numLayers$ attention layers and $\numHeads$ attention heads
per layer. The activations at each layer $\ell \in [\numLayers]$ are given by:
\begin{equation}
    \hidden{\ell} \;=\; \; \hidden{\ell-1} \; + \;\; \sum_{h=1}^{\numHeads} \; \weightsP{\ell,h} \weightsV{\ell,h} \, \hidden{\ell-1} \softmax{\tp{\hidden{\ell-1}} \tp{\weightsQ{\ell,h}} \, \weightsK{\ell,h} \; \hidden{\ell-1}},
\end{equation}
where: (i) $\weightsK{\ell,h}, \weightsQ{\ell,h}, \weightsV{\ell,h} \in \R^{\headDim \times \embedDim}$ and $\weightsP{\ell,h} \in \R^{\embedDim \times \headDim}$ denote the key, query, value, and projection matrices of attention head $h \in [\numHeads]$ in layer $\ell \in [\numLayers]$;
(ii) $\embedDim,\headDim \in \mathbb{N}$ denote the embedding and attention head dimension;
(iii) $\hidden{\ell} \in \hiddenSpace$ denotes the residual stream at layer $\ell \in [\numLayers]$, and
(iv)~$\softmax : \R^{\seqLen \times \seqLen} \rightarrow \R^{\seqLen \times \seqLen}$ denotes the column-wise softmax function with auto-regressive masking.
Finally, we obtain our predicted next-token by passing the hidden activations
of the last token through a final linear layer $\tilde{\mathbf{y}} =
    \weightsOut \, \hiddenlast{\numLayers}$, where $\weightsOut \in \weightSpace$,
followed by a weight-tied unembedding layer. We train the model using the
cross-entropy loss. We give a more detailed description of the architecture in
\cref{sec:architecture}.

\paragraph{Beyond Attention-Only Transformers.} Standard Transformers additionally include feed-forward networks (FFN) and layer normalizations (LN). We omit these components from the core analysis in order to focus on attention layers, which are primarily responsible for inductive capabilities. Nevertheless, our theoretical results extend naturally to standard Transformers with FFNs and LNs, as we discuss in \cref{sec:beyond-attention-only}.

\subsection{Selection, Writing, and Action Bases}
\label{sec:defn_bases_matrices}

\mymacro{\orthoGroup}{\mathcal{O}}

\mymacro{\EV}{\mathbb{E}}
\mymacro{\proj}{\mathrm{proj}}
\mymacro{\weightsAll}{\mathbf{W}}
\mymacro{\weightsSubSpaceAll}{\mathbb{S}}

\mymacro[1]{\valActions}{\mathcal{V}_{#1}}
\mymacro{\tokActions}{\mathcal{T}}
\mymacro{\posActions}{\mathcal{P}}

Our main result (\cref{thm:gradient-confinement,thm:invariant-manifold} below) states that, when a transformer is trained on a block-list task, gradient descent confines its key--query and output weights to the span of a small set of basis matrices.
In this section, we construct these basis matrices, each of which encodes one elementary action that is possibly useful for solving a block-list task. Our construction depends on three intermediate bases:
\defn{selection matrices} are learnable mappings of tokens or positions (e.g., ``map each position to the previous''), \defn{writing bases} track how each head writes information into the residual stream (and thus how later layers can read it back), and \defn{action bases} track how chained attention heads transform the original embeddings (extending writing bases to multiple layers).\looseness=-1

When selecting contextual information, an attention head must compare transformed query and key representations.
In our block-list task setup, this can be done in two ways:
tokens can be compared directly to themselves via their identity, 
or they can be compared to their symmetry-related counterparts.
These operations can be represented via the two matrices:
$\idT \in \weightSpace$, the identity transformation acting only on the subspace of token embeddings, and
$\commCorrel \in \weightSpace$, a projection matrix that maps the embeddings of symmetrically
associated tokens to one another (associated via $\setAssoc$),  mapping the embedding of each token $\itm_i$ to the embedding of its associated partner $\lbl_i$.\footnote{
Formally, $\commCorrel\,\emb{\itm_i} = \emb{\lbl_i}$ and $\commCorrel\,\emb{\lbl_i} = \emb{\itm_i}$ for all $\{ \itm_i, \lbl_i \} \in \setAssoc$, while collapsing all other directions to zero.
See \cref{sec:relational-structure} for an extension to non-symmetric token relations.}
Similarly, in our setup, positions can be either compared directly to themselves, or mapped to one of their (at most $\tupleSize - 1$) previous positions,
which can be represented by matrices: 
$\idP \in \weightSpace$, which denotes the identity transformation acting only on the subspace of positions; and
$\posCorrel$, which maps position embeddings to their immediate predecessor, i.e. $\posCorrel \posEmbed_{:,i+1} = \posEmbed_{:,i}$, and whose matrix powers $\posCorrel^{j}$ thus shift positions back by $j$ time steps.
Based on these observations, we define the following sets of \defn{selection matrices} for tokens and positions, respectively:
\begin{align}
    \tokActions \;=\; \{\; \idT, \commCorrel  \;\},
     &  &
    \posActions \;=\; \{\; \idP, \posCorrel, \ldots, \posCorrel^{\tupleSize - 1} \;\},
\end{align}

Now, we note that each attention head can only write information to the residual stream via its output--value matrix transformation. 
Together with the identity defined by the transformer's residual connections, we say a layer $\ell \in [\numLayers]$ spans the
following \defn{writing basis} on the residual stream:
\begin{align}
    \valActions{\ell} \;=\; \Big\{\; \id \;\Big\} \,\cup\, \left\{\;
    \weightsP{\ell,h} \, \weightsV{\ell,h}
    \;:\;
    h \in [\numHeads]
    \;\right\}.
\end{align}

Notably, after several self-attention layers, a representation may accumulate
information written by the output-value projections of any set of heads in
previous layers. We thus define the \defn{action basis} $\valActions{1:\ell}$
of a collection of layers $1:\ell \in [\numLayers]$, as well as its
\defn{inverse action basis} $\valActions{\ell:1}$:
\begin{align}
    \valActions{1:0} = \{\, \id \,\}
     &  &
    \valActions{1:\ell} = \valActions{\ell} \; \valActions{\ell - 1} \; \ldots \; \valActions{1}
     &  &
    \valActions{\ell:1} = \left\{\;
    \mathbf{V}^+
    \;:\;
    \mathbf{V} \in \valActions{1:\ell}
    \;\right\}
\end{align}
where the product of two matrix sets $\mathcal{A}$ and $\mathcal{B}$ is defined as $\mathcal{A} \mathcal{B} = \{\, \mathbf{AB} : \mathbf{A} \in \mathcal{A},\, \mathbf{B} \in \mathcal{B} \,\}$ and $\mathbf{V}^+$ denotes the Moore-Penrose pseudo-inverse of $\mathbf{V}$. Note that $\valActions{1:\ell}$ contains exactly $(\numHeads + 1)^\ell$ matrices, one for each subset of heads in layers $1$ to $\ell$ with at most one head per layer.

Finally, each attention head may construct keys and queries using any of the
representations that have accumulated up to that layer. This is achieved by
rotating a selection matrix on the left side for the query and on the right
side for the key, aligning the desired subspaces.
Likewise, the final output
layer may select any token representation by reading from the appropriate
subspace. Therefore, we define the \textbf{basis weights} $\mathcal{W}_{\ell}$
at layer $\ell \in [\numLayers]$ and the \textbf{basis output weights}
$\bar{\mathcal{W}}$ as
\begin{align}
    \mathcal{W}_\ell = \valActions{1:\ell-1} \,
    \big( \tokActions \cup \posActions \big) \,
    \valActions{\ell-1:1}
     &  &
    \bar{\mathcal{W}} = \tokActions \, \valActions{\numLayers:1}.
\end{align}
Intuitively, each basis weight in $\mathcal{W}_\ell$ composes three operations: its right factor (in $\valActions{\ell-1:1}$) reads from one position's residual stream a representation written there by some subset of earlier heads, its middle factor (in $\tokActions \cup \posActions$) applies an elementary token or position action to it, and its left factor (in $\valActions{1:\ell-1}$) matches the result against a representation written into another position's residual stream by a (possibly different) subset of earlier heads.\footnote{%
The inverse is needed on the right side to read information from an earlier head which encoded it in its write space. E.g., if a first-layer head stored some information $\emb{\token}$ using the projection $\values{1} = \weightsP{1,1}\weightsV{1,1}$, one must decode it by using the inverse, to recover value $\emb{\token}$ via $\values{1}^{+}\values{1}\emb{\token}$. The left side does not need an inverse since it is used in the attention head through a transpose, $\tp{\big(\values{1}\,\emb{\token}\big)} = \tp{\emb{\token}}\,\tp{\values{1}}$, which, by our assumptions (the value maps are orthogonal projections), already equals the required pseudo-inverse: $\tp{\values{1}} = \values{1}^{+}$.%
}
Similarly, each basis output weight in $\bar{\mathcal{W}}$ reads a representation written by a subset of heads and maps it to an output token, either directly or through its association.\looseness=-1

\newcommand{\calX}{\mathcal{X}}

\subsection{Invariant Dynamics Theory}
\label{sec:IMIR-theory}

We now move forward to our main theorem.
Before that, we highlight that our result relies on a few assumptions, which we state informally here and formally in \cref{sec:assumptions}. 
First, we put forward the following \emph{data-related assumptions}: 
(i)~token associations in $\setAssoc$ are fully interchangeable, so the set of valid inputs $\calX$, where $\ctx \in \calX \subseteq \vocab^N$, is closed under association-preserving token permutations; 
(ii) block positions carry random offsets; and
(iii)~the embeddings of tokens which do not belong to any association in $\setAssoc$, termed here rare tokens, are lexinvariant \citep{huang2023lexinvariant} (i.e., they are re-sampled for every sequence).
We believe these assumptions to be straightforward and motivated by the traditional structure of block-list tasks.
Second, we put forward the following \emph{architecture-related assumptions}:
(iv)~token and position embeddings
are orthonormal;
(v)~we analyze the merged key--query product $\tp{\weightsQ{\ell}}\weightsK{\ell}$ as a single learnable matrix; and 
(vi)~the output--value maps are fixed projections onto mutually orthogonal subspaces (following the associative-memory \citep{bietti2024birth} and disentangled-residual-stream view \citep{nichani2024transformers, friedman2023learning}).
The extension of our results to models with learnable value projection matrices is sketched in appendices \ref{sec:multi-head} and \ref{sec:circuit-folding}.

\mymacro{\manifold}{\mathcal{S}}

\begin{restatable}[Gradient Confinement]{thm}{thmGradientConfinement}
    \label{thm:gradient-confinement}
    Consider the population-wise training dynamics of a model satisfying \cref{asm:embeddings,asm:merged-kq,asm:ortho-values} on a block-list task satisfying \cref{asm:data-symmetry,asm:data-offsets,asm:lexinvariance}.
    Further, let $\manifold$ be the space of transformers with $\tp{\mathbf{W}^{(\ell, h)}_Q} \; \mathbf{W}^{(\ell, h)}_K \;\in\; \vecSpan(\mathcal{W}_\ell)$ and $\weightsOut \in \vecSpan(\bar{\mathcal{W}})$ and let 
    $\weightsAll$ be the tuple of all weights in this transformer.
    If a transformer has weights in $\manifold$, then its expected gradient is confined to $\manifold$, i.e.:
    \begin{align}
        \weightsAll \in \manifold 
        \quad \implies \quad
        \EV \big[\nabla_{\weightsAll}\loss(\weightsAll) \big] \in \manifold.
    \end{align}
\end{restatable}
\begin{proofsketch}
    Our proof relies on two key symmetries of the data, regarding token embeddings and position embeddings, respectively. First, the only correlations between tokens are those between associated tokens. This means that the data distribution is preserved under any permutation of tokens that preserves common associations. Second, blocks have random offsets to simulate an arbitrary number of irrelevant tokens between each pair of consecutive blocks. This means that the data distribution is preserved when changing the positional offsets of blocks.
    We apply these two transformations (token permutation and position offset
    shifts) to the model input and compute how this transforms the hidden
    representations and model output. A key result is that all attention scores are
    perfectly invariant to our transformations, hence the hidden and output
    transformations are essentially identical to our input transformations (applied
    to different subspaces). Importantly, passing permuted tokens through the model is equivalent to permuting its outputs, which implies that the loss is also invariant to our
    transformations for any input sequence.
    Finally, since the data distribution is invariant to our transformations, the expected
    gradients must also be invariant to the same transformations. We compute
    exactly how the gradient is transformed under our input transformations and we
    show that any gradients invariant under those transformations must fit the
    structure described above. We provide a comprehensive proof in
    \cref{sec:invariant-space-proof}.
\end{proofsketch}

\begin{corollary}[Invariant Manifold]
    \label{thm:invariant-manifold}
    Under the same assumptions, $\manifold$ is an \textbf{invariant manifold}\footnote{Not to be confused with \emph{invariant
        subspaces}, referring to vector subspaces preserved by linear maps.} under the population-wise training dynamics: if $\weightsAll \in \manifold$ at some training step, then $\weightsAll \in \manifold$ for the rest of training. In other words, the weight structure described above is preserved during training with gradient descent.
\end{corollary}
\begin{proof}
    A gradient-descent step updates $\weightsAll \mapsto \weightsAll - \eta\, \EV[\nabla_{\weightsAll}\loss(\weightsAll)]$. By \cref{thm:gradient-confinement} the update direction lies in $\manifold$, and the constraints defining $\manifold$ are linear in the key--query and output weights, so $\weightsAll \in \manifold$ implies the updated weights remain in $\manifold$.
\end{proof}

\subsection{The Geometry of Reasoning Circuits}
\label{sec:imir-interp}

Each of the basis weights in $\mathcal{W}_{\ell}$ and $\bar{\mathcal{W}}$ corresponds to a simple and interpretable action over specific subspaces of embeddings. 
For attention layers, such interpretable actions might be, for
example, ``attending to the previous position'' or ``attending to the token retrieved by head 1'', while an output action might be something like ``read the token retrieved by head 3''.
When the model is on the IMIR, by decomposing it into its
components corresponding to each weight matrix in $\mathcal{W}_{\ell}$ and $\bar{\mathcal{W}}$, we can describe the model as a sum of highly interpretable actions. 
When the model is not on the IMIR, projecting it onto the IMIR provides a partial description.
We give a more comprehensive and intuitive description of the weight spaces in \cref{sec:IMIR-intuitive}. 
We also give a theoretical motivation for the weight spaces from representational symmetries in \cref{sec:IMIR-symmetry}, discussing why they capture all actions relevant for inductive tasks.

\section{Understanding Circuit Competition}
\label{sec:circuit-competition}

\subsection{Data Distributional Properties of Natural Language}
\label{sec:data-properties}

\mymacro{\freqR}{f_r}
\mymacro{\freqC}{f_c}
\mymacro{\freqWCV}{f_{v}}
\mymacro{\burstiness}{b}
\mymacro{\noiseAmount}{\varepsilon}
\mymacro{\noise}{\eta}
\mymacro{\dataDist}{\mathcal{D}}

One important aspect not discussed above is how to sample the tokens $\ctx$ used in the tasks. 
Notably, prior work \citep{chan2022data,kim2025training,reddy2023mechanistic,zucchet2025emergence} has shown that this choice can have a strong impact on the training dynamics of language models, and, in particular, that standard properties of token distributions in natural languages (e.g., Zipfian distributions or burstiness) can accelerate learning.
We incorporate a number of natural language distributional properties into
our setup as follows:
\begin{enumerate}[leftmargin=*]
    \item \textbf{Tokens differ in frequency.}
        Token frequencies in natural language typically follow a long-tailed distribution \citep{piantadosi2014zipf, zipf1949principle}. 
        To model this and bring our setup closer to natural language (while keeping it analytically tractable), we consider a subset of the tokens in our vocabulary to form a set of \defn{common tokens} $\vocabC$ with total probability of occurrence  $\mathbb P(v \in \vocabC) = 1- f_r$ where $f_r \in [0;1]$. 
        The remaining tokens are termed \defn{rare tokens}: $\vocabR = \vocab \setminus \vocabC$, with $\mathbb P(v \in \vocabR) = f_r$. 
        We assume that rare tokens form a virtually infinite set, where no rare token is ever observed twice.\footnote{
        In practice, we model such a virtually ``infinite'' vocabulary using a finite vocabulary and a lexinvariant assumption \citep{huang2023lexinvariant}, where a rare token's embeddings are repeatedly re-sampled at every batch, inhibiting the model from memorizing it.}\looseness=-1
    \item \textbf{Tokens can be bursty.} Documents in natural language tend to contain a concept or association multiple times, or not at all. In other words, in natural language, concepts do not show up uniformly, but rather in `bursts'. Consider, for example, the word `burstiness' itself, which appears several times in this document, while most works do not contain it at all. We incorporate this behavior in our task class by assuming that block pairs present in the context have an average multiplicity of $\burstiness \ge 1$. For example, an input sequence with $\burstiness = 2$ might look like "$\mathrm{ab,gh,gh,ab,g\rightarrow h}$". We write $\dataDist(\burstiness)$ for the resulting data distribution at burstiness $\burstiness$, making the dependence explicit when we compare gradients across burstiness levels.
    \item \textbf{Tokens are polysemous.}
          The same token can present different meanings and associations in different documents.
          For example, in computer science books, the name `Alan' might be commonly followed by the surname `Turing', but at times it might also be followed by a different surname such as `Baker'.\footnote{
          Alan Baker (1939–2018) was an English mathematician, known for his work on effective methods in number theory, in particular those arising from transcendental number theory.} 
          In our setup, we model this aspect by assuming that the next token to be predicted after a common token differs from its common association with frequency $\freqWCV \in [0,1]$.
\end{enumerate}

Past work has shown that in-context learning abilities emerge optimally in
transformers trained with many rare tokens, high burstiness, and high
within-class variability \citep{chan2022data,kim2025training}. In the absence
of these properties, in-context learning emerges after significantly more
training, or not at all.
The same properties were shown to enable learning in a 3-parameter
phenomenological model of induction-head formation
\citep{reddy2023mechanistic}. Finally, burstiness was shown to accelerate the
emergence of sparse attention in single-layer transformers
\citep{zucchet2025emergence}. We exemplify these properties in
\cref{fig:block_list}.

\subsection{The Role of Data: In-Context Learning \emph{vs.}\ In-Weights Learning}
\label{sec:icl-vs-iwl}

\newcommand{\mathcomment}[1]{\text{\textcolor{gray}{#1}}}

We now use our \cref{thm:invariant-manifold} to study a very simple setting, but which already exhibits circuit competition: a 2-layer, single-head attention-only transformer trained to predict the second token of an in-context bigram.  
The IMIR for this architecture is 28-dimensional,\footnote{The IMIR has 28-dimensions for the space of key-query products and final projections, not the full parameter space.} with 4 dims for the first layer actions, 16 dims for the second layer, and 8 dims for output layer:
\begin{subequations}
\begin{align}
    \tp{\weightsQ{1,1}} \weightsK{1,1} &\in
    \vecSpan(\underbrace{\{\idT, \commCorrel, \idP, \posCorrel\}}_{\tokActions \cup \posActions}) \\
    \tp{\weightsQ{2,1}} \weightsK{2,1} &\in 
    \vecSpan(\underbrace{\{\identity, \weightsP{1, 1}\weightsV{1, 1}\}}_{\valActions{1:\ell-1}}
    \underbrace{\{\idT, \commCorrel, \idP, \posCorrel\}}_{\tokActions \cup \posActions\;}
    \underbrace{\{\identity, \big(\weightsP{1, 1}\weightsV{1, 1}\big)^{\!+}\}}_{\valActions{\ell-1:1}}) \\
    \weightsOut &\in 
    \vecSpan(\underbrace{\{\idT, \commCorrel\}}_{\tokActions} \underbrace{\{\identity, \big(\weightsP{2, 1}\weightsV{2, 1}\big)^{\!+}\}\{\identity, \big(\weightsP{1, 1}\weightsV{1, 1}\big)^{\!+}\}}_{\valActions{\numLayers:1}})
\end{align}
\end{subequations}
As mentioned above, this IMIR allows us to additively decompose a model's learned weights into these interpretable directions, which we do and plot in \cref{fig:icl-vs-iwl} (center-bottom) (we also plot each parameter individually in \cref{sec:icl-vs-iwl-all-circuits}).
Notably, from this figure, we see that learning appears to concentrate on four directions, three of which are useful for \emph{in-context learning} (ICL) and one for \emph{in-weights learning} (IWL):
\begin{subequations}
\begin{align}
    \alpha:\;\; & \tp{\weightsQ{1,1}} \weightsK{1,1} \,\propto\, \posCorrel,
                &
    \gamma:\;\; & \weightsOut    \,\propto\, \idT\,\big(\weightsP{2, 1}\weightsV{2, 1}\big)^{\!+},
    \\
    \beta:\;\;  & \tp{\weightsQ{2,1}} \weightsK{2,1} \,\propto\, \idT\,\big(\weightsP{1, 1}\weightsV{1, 1}\big)^{\!+},
                &
    \delta:\;\; & \weightsOut    \,\propto\, \commCorrel.
\end{align}
\end{subequations}
Note that $(\alpha, \beta, \gamma)$ assemble the canonical induction head: the $\alpha$ subspace allows the first layer's head to pay attention to a
previous token, the $\beta$ enables the second layer's head to match the lookup token $\ctx^n_1$ to a previous one, and then $\gamma$ reads the information out and extracts the retrieved label.
Together they thus implement in-context learning, the only solution that works on the rare part of the data. 
The fourth direction $\delta$ accounts for in-weights learning, routing the input token to its canonical partner through a memorized association map $\commCorrel$.
Causal ablations confirming that $(\alpha, \beta, \gamma)$ alone yields ICL and $\delta$ alone yields IWL are reported in \cref{sec:icl-vs-iwl-ablations}.

\begin{figure}[t]
    \centering
    \includegraphics[width=\linewidth]{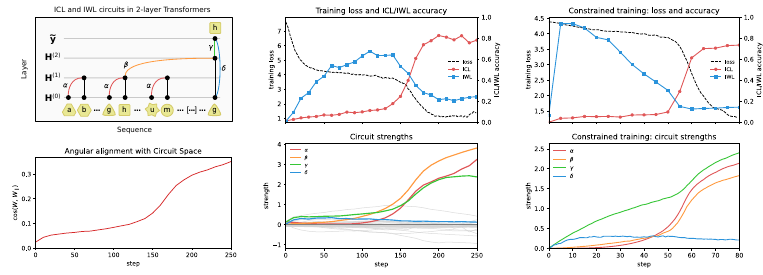}
    \caption{In two-layer transformers, ICL and IWL circuits compete during training. Top left: ICL has three circuit units ($\alpha$, $\beta$, and $\gamma$, forming an induction head) and IWL has one circuit unit ($\delta$, predicting the next token directly from the current token). Bottom left: the parameters become aligned with the IMIR during training. Center: the emergence of ICL/IWL abilities and circuits. Right: emergence of ICL/IWL during isolated training (only $\alpha$, $\beta$, $\gamma$, $\delta$) exhibits similar dynamics. Training details in \cref{sec:icl-vs-iwl-training}.\looseness=-1}
    \label{fig:icl-vs-iwl}
\end{figure}

\paragraph{IWL emerges first and starves ICL of gradient.}
The direction $\delta$ is a one-dimensional fix that already solves every query answerable by a memorized canonical association---i.e., common-token queries without within-class variation, which make up a $(1 - \freqR)(1 - \freqWCV)$ fraction of the data. 
It is the easiest direction for SGD to discover, empirically saturating well before $(\alpha, \beta, \gamma)$ have moved appreciably (\cref{fig:icl-vs-iwl}, center).
Once $\delta$ is in place, then \cref{thm:icl-rare-tokens} (presented below) starts taking effect, with ICL gradients---i.e., gradients with respect to the $(\alpha, \beta, \gamma)$ coordinates---being suppressed exactly on the data handled by IWL,
\begin{equation}
    \nabla_\mathrm{ICL} \loss(\weightsAll + \weightsAll_\mathrm{IWL})
    \;\approx\;
    \big(\freqR + (1 - \freqR)\,\freqWCV\big)\;
    \nabla_\mathrm{ICL} \loss(\weightsAll).
\end{equation}
The induction head therefore emerges later than it would in a
rare-only ($\freqR = 1$) world, with a delay that grows as $\freqR$
decreases or as $\freqWCV \to 0$. Per-circuit dynamics across data
regimes are in \cref{sec:icl-vs-iwl-all-circuits}.

\begin{restatable}{thm}{thmICLrareTokens}
    \label{thm:icl-rare-tokens}
    In the presence of an IWL circuit with logit scale $\delta$, the ICL circuit
    receives a population gradient suppressed by a data-dependent factor:
    \begin{equation}
        \EV\big[\nabla_\mathrm{ICL} \loss(\weightsAll + \weightsAll_\mathrm{IWL})\big]
        \;=\;
        \Big(
        \freqR + (1 - \freqR) \freqWCV
        \Big) \;
        \EV\big[\nabla_\mathrm{ICL} \loss(\weightsAll)\big]
        \;+\; O\!\big(|\vocab|\, e^{-\delta} + |\vocab|^{-1}\big),
    \end{equation}
    where $\weightsAll$ denotes the ICL-only configuration (no in-weights learning).
    The error term vanishes in the saturated-IWL ($\delta \to \infty$) and
    large-vocabulary ($|\vocab| \to \infty$) limit.
\end{restatable}
\begin{proofsketch}
    The IWL solution solves common queries whose in-context partner equals their
    canonical partner, a fraction $\freqC (1 - \freqWCV)$ of the data; there the
    loss is near-zero and, up to a residual of order $|\vocab|\,e^{-\delta}$, the
    ICL direction receives no gradient. On the remaining $\freqR + (1-\freqR)\freqWCV$
    of the data the IWL solution is either inactive (rare query) or points at the
    wrong token (scrambled partner), so the ICL gradient is unchanged up to a
    finite-vocabulary cross-talk of order $|\vocab|^{-1}$. We give a detailed
    proof in \cref{sec:proof-rare-tokens}.
\end{proofsketch}

\paragraph{Burstiness amplifies the ICL gradient.}
\Cref{thm:icl-rare-tokens} describes how the rare-token frequency $\freqR$ and the
within-class variability $\freqWCV$ reduce the gradient received by the ICL circuit.
Burstiness, which we introduced in \cref{sec:data-properties}, has the opposite
effect. In a bursty sequence, the block relevant to the query appears $\burstiness$
times in the context, and each occurrence provides the same in-context evidence to
the induction circuit, whereas the distractor blocks provide no consistent signal.
The ICL gradient therefore grows with $\burstiness$, which accelerates the emergence
of the induction head.

\begin{restatable}{thm}{thmICLburstiness}
    \label{thm:icl-burstiness}
    In the distractor-dominated regime, where the number of context blocks grows
    with the burstiness held fixed ($\seqPairs \to \infty$ with $\burstiness$ fixed,
    so $\seqPairs \gg \burstiness$), the ICL circuit receives a population gradient
    proportional to burstiness:
    \begin{equation}
        \EV_{\dataDist(\burstiness)}\big[\nabla_\mathrm{ICL} \loss(\weightsAll)\big]
        \;=\;
        \burstiness \;
        \EV_{\dataDist(1)}\big[\nabla_\mathrm{ICL} \loss(\weightsAll)\big]
        \;\big(1 \;+\; O\!\big(\burstiness / \seqPairs\big)\big),
    \end{equation}
    where $\EV_{\dataDist(\burstiness)}$ denotes expectation over the data
    distribution at burstiness $\burstiness$.
\end{restatable}
\begin{proofsketch}
    Burstiness acts as a multiplier on the relevant blocks: the ICL gradient
    decomposes into per-block signals that are identical for the $\burstiness$
    relevant blocks and zero in expectation for distractors, so the relevant
    signal is multiplied by $\burstiness$. The only cross-block coupling is the
    shared layer-$2$ softmax denominator $Z_q$ at the query position, which
    depends on $\burstiness$ only through a relative $O(\burstiness/\seqPairs)$
    shift; since every per-block signal scales as $1/Z_q$, this yields a
    \emph{multiplicative} $\big(1 + O(\burstiness/\seqPairs)\big)$ correction
    rather than an additive one. This is crucial, because the leading term is itself
    of order $\burstiness/\seqPairs$. We give a detailed proof in
    \cref{sec:proof-burstiness}.
\end{proofsketch}

\paragraph{The induction head is a winning ticket.}
The Lottery Ticket Hypothesis (LTH)~\citep{frankle2018lottery} posits that a
randomly initialized network contains a sparse subnetwork which, trained in
isolation, matches the dense network's performance. Our framework identifies
one such subnetwork explicitly: the four-parameter $(\alpha, \beta, \gamma,
    \delta)$ IMIR sub-manifold. We test the prediction directly via
\emph{constrained training}: at every SGD step we project the model back onto
the subspace spanned by $\alpha, \beta, \gamma, \delta$. The resulting dynamics
closely match those of the unconstrained network (\cref{fig:icl-vs-iwl},
right).

\vspace{-5pt}
\subsection{The Role of Initialization: Competing Induction Heads}
\label{sec:competing-induction-heads}
\vspace{-3pt}

\begin{figure}[t]
    \centering
    \includegraphics[width=\linewidth]{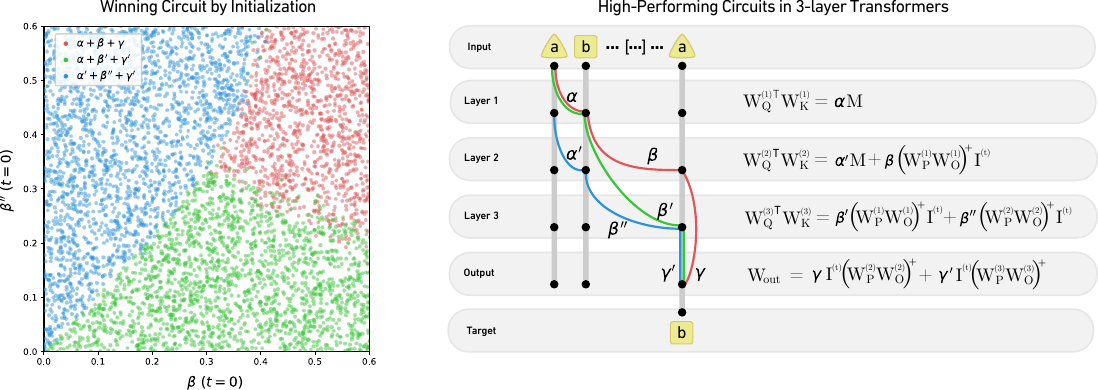}
    \vspace{-15pt}
    \caption{Training a 3-layer Transformer to solve in-context bi-grams yields three different solutions. On the left, we show the winning solution at various initializations. On the right, we show the precise mechanism of each possible solution with formulas for model weights.
    }
    \label{fig:icl-competition}
    \vspace{-5pt}
\end{figure}

An induction head can use any pair of attention heads in different layers.
Thus, in a transformer with three layers and one head per layer, trained on the same task, three distinct induction-head circuits are possible, as depicted in \cref{fig:icl-competition} (right). 
A randomly initialized model must converge to one of them (or a composition of them). 
To better understand how the initialization influences the learned circuit, we vary two parameters (denoted as $\beta$ and $\beta''$) and we plot the `winning' circuit in \cref{fig:icl-competition} (left). 
We give the full training details in \cref{sec:competing-induction-heads-training} and we plot the full training dynamics for 12 seeds in \cref{sec:competing-induction-heads-dynamics}.\looseness=-1

\vspace{-3pt}
\paragraph{Winning mechanics are complex to predict.}
In the LTH literature, winning subnetworks are typically identified by weight magnitude after training, and whether they can instead be predicted from properties of the initialization---such as initial weight magnitude---remains actively debated \citep{frankle2020pruning, paul2022unmasking, liu2024survey}.
The simplest initialization-based predictor in our setup would be that each circuit wins when its own units start large. 
We can analyze this predictor by looking at \cref{fig:icl-competition} (left).
In fact, this predictor aligns with the fact that the ($\alpha', \beta'', \gamma'$)-circuit (layers 2 and 3, blue) tends to emerge when $\beta''$ is initialized with a large magnitude. 
However, the ($\alpha, \beta', \gamma'$)-circuit (layers 1 and 3, green) seems to benefit from a larger $\beta$, which is not part of this circuit. 
Moreover, the ($\alpha, \beta, \gamma$)-circuit (layers 1 and 2, red) seems to require not just a large $\beta$ (part of the circuit), but also a large $\beta''$ (not part of the circuit).
Together, these results suggest the history is not that simple.

\vspace{-3pt}
\paragraph{The story of a circuit battle.} 
A closer look at the phase transitions in \cref{fig:icl-competition} (left) also sheds light on our transformer's learning dynamics.
First, note that these phase transitions are very sharp between the blue and green solutions and between blue and red, but not between green and red. 
This is because the red and green circuits are cooperative: both share $\alpha$, while $\beta$ and $\beta'$ are in different layers. Meanwhile, blue is competing with both red and green, trying to use the same layers in different ways. This might also explain why red requires a large $\beta''$: it stimulates the blue circuit, which in turn inhibits the green circuit.
Notably, the learning dynamics plots in \cref{sec:competing-induction-heads-dynamics} appear to support this analysis.\looseness=-1

\vspace{-5pt}
\section{Automated Circuit Detection for Induction Tasks}
\label{sec:autointerp}
\vspace{-3pt}

Our theoretical framework offers two benefits that may improve circuit discovery in trained models. 
First, the IMIR has a low dimensionality compared to the full parameter space, which greatly reduces the circuit search space.
Second, while most interpretability works identify circuits at the level of attention heads \citep{elhage2021mathematical, wang2022interpretability, conmy2023towards, sharma2025llms}, the IMIR allows us to identify circuits at the more granular level of precise weight structures, enabling us to understand the exact mechanism of each head.\footnote{The bases present in our IMIR are also similar to the variables produced by the D-RASP decompilation procedure in \citet{friedman2023learning, weiss2021thinking}.}

To illustrate these benefits, we apply our framework to the question of
Automated Circuit Discovery \citep{conmy2023towards} of induction circuits. We
train deep attention-only single-head transformers on the two-hop induction
task \citep{sanford2024transformers} (training details in
\cref{sec:autointerp-details}). We discover the learned circuit using a simple
algorithm: we project the learned parameters to the IMIR and we ablate the
dimensions that don't impact performance. We depict the recovered circuits
schematically in \cref{fig:autointerp}. We describe the complete algorithm in
\cref{sec:autointerp-algo} and the complete circuits in
\cref{sec:autointerp-circuits}.

One interesting finding is that the attention paths depicted in
\cref{fig:autointerp} are not implemented using only a minimal set of necessary
directions, but also leverage a few additional `aiding' directions. The
additional directions marginally increase the accuracy by reducing the
attention to irrelevant tokens. This kind of structure would be invisible to
existing automated circuit discovery algorithms with granularity at the
attention head level.

\begin{figure}[t]
    \centering
    \captionsetup{width=0.9\linewidth}
    \includegraphics[width=0.8\linewidth]{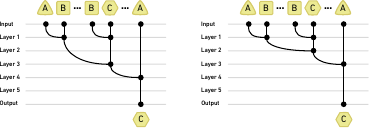}
    \caption{Two auto-discovered circuits in deep Transformers for two-hop induction.}
    \label{fig:autointerp}
    \vspace{-10pt}
\end{figure}

\vspace{-2pt}
\section{Conclusion and Future Work}
\label{sec:discussion}
\vspace{-1pt}

In this paper, we identified an invariant manifold for a class of transformers trained on block-list tasks.
This manifold allows us to both study transformers' training dynamics, and to decompose trained models' weights into interpretable subspaces.
We now discuss three directions for future work in which our
theoretical framework might prove useful in addition to the use cases demonstrated in this paper.\looseness=-1

\paragraph{Improving Reasoning Abilities.} 
Mechanistic interpretability has already proven useful for the design of improved neural network architectures \citep{xiao2023efficient,    willette2025delta, darcet2023vision, wang2024buffer}. 
For example, the selection mechanism in \emph{Mamba} architectures \citep{gu2023mamba} is inspired by induction heads \citep{olsson2022context}. 
A better understanding of inductive reasoning abilities (and their emergence) may thus enable the development of improved learning algorithms. 
In particular, our theoretical framework provides a geometric description of a large class of induction circuits.
Useless circuits can be identified and pruned more effectively, while remaining circuits can be described using a reduced number of parameters, reducing the memory footprint and inference latency. 
Moreover, it is known that learning induction heads is more difficult in long sequences \citep{musat2026emergence}, while deep circuits require either an implicit curriculum or an exponential amount of data \citep{musat2024mechanism, wang2025learning}. 
Can we design neural architectures without these limitations?\looseness=-1

\paragraph{Accelerating Experiments at Scale.} In deep learning, many important phenomena, such as data dependency of ICL
\citep{chan2022data} or generalization beyond overfitting
\citep{power2022grokking}, began as empirical observations in synthetic tasks,
before they were later shown to hold in general settings
\citep{kim2025training, liu2022omnigrok}. There is a simple reason for this:
training large models on full datasets is slow and costly, whereas uncovering
novel phenomena may require running a great number of experiments quickly. 
Our
framework can further accelerate experimentation in two ways. First, by
constraining learning to the IMIR, or even a reduced subset of interesting
dimensions, training can be greatly optimized, reducing experiment time and
computational demands. Second, by providing the geometric description of
meaningful directions, our framework can aid hypothesis generation and
understanding of training dynamics.

\paragraph{Invariant Manifolds in Other Tasks.} In this work, we have shown how invariant manifolds can be used to understand
training dynamics and model internals. While invariant manifolds of learning
dynamics have been studied before \citep{joudaki2025barriers,
    zhang2021validating, musat2026emergence}, to the best of our knowledge, this
work is the first to identify the geometry of the invariant manifolds
responsible for learning a large class of tasks. If inductive tasks exhibit a
low-dimensional invariant manifold, then maybe other tasks also do. Uncovering
them could lead to a better understanding of learning dynamics in deep neural
networks.%

\bibliographystyle{unsrtnat} %
\bibliography{refs} %

\begin{thebibliography}{49}
\providecommand{\natexlab}[1]{#1}
\providecommand{\url}[1]{\texttt{#1}}
\expandafter\ifx\csname urlstyle\endcsname\relax
  \providecommand{\doi}[1]{doi: #1}\else
  \providecommand{\doi}{doi: \begingroup \urlstyle{rm}\Url}\fi

\bibitem[Zheng et~al.(2025)Zheng, Wang, Huang, Song, Yang, Tang, Xiong, and
  Li]{zheng2024attention}
Zifan Zheng, Yezhaohui Wang, Yuxin Huang, Shichao Song, Mingchuan Yang,
  Bo~Tang, Feiyu Xiong, and Zhiyu Li.
\newblock Attention heads of large language models: A survey.
\newblock \emph{Patterns}, 6\penalty0 (2), 2025.

\bibitem[Olah et~al.(2020)Olah, Cammarata, Schubert, Goh, Petrov, and
  Carter]{olah2020zoom}
Chris Olah, Nick Cammarata, Ludwig Schubert, Gabriel Goh, Michael Petrov, and
  Shan Carter.
\newblock Zoom in: An introduction to circuits.
\newblock \emph{Distill}, 5\penalty0 (3):\penalty0 e00024--001, 2020.

\bibitem[Wang et~al.(2023)Wang, Variengien, Conmy, Shlegeris, and
  Steinhardt]{wang2022interpretability}
Kevin Wang, Alexandre Variengien, Arthur Conmy, Buck Shlegeris, and Jacob
  Steinhardt.
\newblock Interpretability in the wild: a circuit for indirect object
  identification in {GPT-2} small.
\newblock In \emph{The Eleventh International Conference on Learning
  Representations}, 2023.

\bibitem[Conmy et~al.(2023)Conmy, Mavor-Parker, Lynch, Heimersheim, and
  Garriga-Alonso]{conmy2023towards}
Arthur Conmy, Augustine Mavor-Parker, Aengus Lynch, Stefan Heimersheim, and
  Adri{\`a} Garriga-Alonso.
\newblock Towards automated circuit discovery for mechanistic interpretability.
\newblock \emph{Advances in Neural Information Processing Systems},
  36:\penalty0 16318--16352, 2023.

\bibitem[Wei et~al.(2022)Wei, Tay, Bommasani, Raffel, Zoph, Borgeaud, Yogatama,
  Bosma, Zhou, Metzler, et~al.]{wei2022emergent}
Jason Wei, Yi~Tay, Rishi Bommasani, Colin Raffel, Barret Zoph, Sebastian
  Borgeaud, Dani Yogatama, Maarten Bosma, Denny Zhou, Donald Metzler, et~al.
\newblock Emergent abilities of large language models.
\newblock \emph{Transactions on Machine Learning Research}, 2022.

\bibitem[Simon et~al.(2026)Simon, Kunin, Atanasov, Boix-Adser{\`a}, Bordelon,
  Cohen, Ghosh, Guth, Jacot, Kamb, et~al.]{simon2026there}
Jamie Simon, Daniel Kunin, Alexander Atanasov, Enric Boix-Adser{\`a}, Blake
  Bordelon, Jeremy Cohen, Nikhil Ghosh, Florentin Guth, Arthur Jacot, Mason
  Kamb, et~al.
\newblock There will be a scientific theory of deep learning.
\newblock \emph{arXiv preprint arXiv:2604.21691}, 2026.

\bibitem[Aky{\"u}rek et~al.(2024)Aky{\"u}rek, Wang, Kim, and
  Andreas]{akyurek2024context}
Ekin Aky{\"u}rek, Bailin Wang, Yoon Kim, and Jacob Andreas.
\newblock In-context language learning: Architectures and algorithms.
\newblock In \emph{International Conference on Machine Learning}. PMLR, 2024.

\bibitem[Edelman et~al.(2024)Edelman, Edelman, Goel, Malach, and
  Tsilivis]{edelman2024evolution}
Benjamin~L Edelman, Ezra Edelman, Surbhi Goel, Eran Malach, and Nikolaos
  Tsilivis.
\newblock The evolution of statistical induction heads: In-context learning
  {M}arkov chains.
\newblock \emph{Advances in Neural Information Processing Systems}, 37, 2024.

\bibitem[Varre et~al.(2025)Varre, Y{\"u}ce, and Flammarion]{varre2025learning}
Aditya Varre, Gizem Y{\"u}ce, and Nicolas Flammarion.
\newblock Learning in-context n-grams with transformers: Sub-n-grams are
  near-stationary points.
\newblock In \emph{International Conference on Machine Learning}. PMLR, 2025.

\bibitem[Sanford et~al.(2024)Sanford, Hsu, and
  Telgarsky]{sanford2024transformers}
Clayton Sanford, Daniel Hsu, and Matus Telgarsky.
\newblock Transformers, parallel computation, and logarithmic depth.
\newblock In \emph{International Conference on Machine Learning}. PMLR, 2024.

\bibitem[Musat(2025)]{musat2024mechanism}
Tiberiu Musat.
\newblock Mechanism and emergence of stacked attention heads in multi-layer
  transformers.
\newblock In \emph{The Thirteenth International Conference on Learning
  Representations}, 2025.

\bibitem[Wang et~al.(2024)Wang, Wang, Zhang, Zhou, Jin, Hu, Sun, Li, Zhang, and
  Xu]{wang2024buffer}
Zhiwei Wang, Yunji Wang, Zhongwang Zhang, Zhangchen Zhou, Hui Jin, Tianyang Hu,
  Jiacheng Sun, Zhenguo Li, Yaoyu Zhang, and Zhi-Qin~John Xu.
\newblock Understanding the language model to solve the symbolic multi-step
  reasoning problem from the perspective of buffer mechanism.
\newblock \emph{arXiv preprint arXiv:2405.15302}, 2024.

\bibitem[Allen-Zhu(2025)]{allen-zhu2026physics}
Zeyuan Allen-Zhu.
\newblock Physics of language models: Part 4.1, architecture design and the
  magic of canon layers.
\newblock In \emph{The Thirty-ninth Annual Conference on Neural Information
  Processing Systems}, 2025.
\newblock URL \url{https://openreview.net/forum?id=kxv0M6I7Ud}.

\bibitem[Singh et~al.(2025)Singh, Moskovitz, Dragutinovic, Hill, Chan, and
  Saxe]{singh2025strategy}
Aaditya~K Singh, Ted Moskovitz, Sara Dragutinovic, Felix Hill, Stephanie~CY
  Chan, and Andrew~M Saxe.
\newblock Strategy coopetition explains the emergence and transience of
  in-context learning.
\newblock In \emph{International Conference on Machine Learning}. PMLR, 2025.

\bibitem[Merrill et~al.(2023)Merrill, Tsilivis, and Shukla]{merrill2023tale}
William Merrill, Nikolaos Tsilivis, and Aman Shukla.
\newblock A tale of two circuits: Grokking as competition of sparse and dense
  subnetworks.
\newblock \emph{arXiv preprint arXiv:2303.11873}, 2023.

\bibitem[Zhong et~al.(2023)Zhong, Liu, Tegmark, and Andreas]{zhong2023clock}
Ziqian Zhong, Ziming Liu, Max Tegmark, and Jacob Andreas.
\newblock The clock and the pizza: Two stories in mechanistic explanation of
  neural networks.
\newblock \emph{Advances in neural information processing systems},
  36:\penalty0 27223--27250, 2023.

\bibitem[Bietti et~al.(2023)Bietti, Cabannes, Bouchacourt, Jegou, and
  Bottou]{bietti2024birth}
Alberto Bietti, Vivien Cabannes, Diane Bouchacourt, Herve Jegou, and Leon
  Bottou.
\newblock Birth of a transformer: A memory viewpoint.
\newblock \emph{Advances in Neural Information Processing Systems}, 36, 2023.

\bibitem[Nichani et~al.(2024)Nichani, Damian, and Lee]{nichani2024transformers}
Eshaan Nichani, Alex Damian, and Jason~D Lee.
\newblock How transformers learn causal structure with gradient descent.
\newblock In \emph{International Conference on Machine Learning}. PMLR, 2024.

\bibitem[Zucchet et~al.(2025)Zucchet, d'Angelo, Lampinen, and
  Chan]{zucchet2025emergence}
Nicolas Zucchet, Francesco d'Angelo, Andrew~K Lampinen, and Stephanie~CY Chan.
\newblock The emergence of sparse attention: impact of data distribution and
  benefits of repetition.
\newblock \emph{Advances in Neural Information Processing Systems}, 38, 2025.

\bibitem[Musat et~al.(2025)Musat, Pimentel, Noci, Stolfo, Sachan, and
  Hofmann]{musat2026emergence}
Tiberiu Musat, Tiago Pimentel, Lorenzo Noci, Alessandro Stolfo, Mrinmaya
  Sachan, and Thomas Hofmann.
\newblock On the emergence of induction heads for in-context learning, 2025.
\newblock URL \url{https://arxiv.org/abs/2511.01033}.

\bibitem[Chan et~al.(2022)Chan, Santoro, Lampinen, Wang, Singh, Richemond,
  McClelland, and Hill]{chan2022data}
Stephanie Chan, Adam Santoro, Andrew Lampinen, Jane Wang, Aaditya Singh, Pierre
  Richemond, James McClelland, and Felix Hill.
\newblock Data distributional properties drive emergent in-context learning in
  transformers.
\newblock \emph{Advances in Neural Information Processing Systems},
  35:\penalty0 18878--18891, 2022.

\bibitem[Kim et~al.(2025)Kim, Kim, Kwon, Yang, Jung, and Cha]{kim2025training}
Minsung Kim, Dong-Kyum Kim, Jea Kwon, Nakyeong Yang, Kyomin Jung, and Meeyoung
  Cha.
\newblock How training data shapes the use of parametric and in-context
  knowledge in language models.
\newblock \emph{arXiv preprint arXiv:2510.02370}, 2025.

\bibitem[Singh et~al.(2023)Singh, Chan, Moskovitz, Grant, Saxe, and
  Hill]{singh2023transient}
Aaditya Singh, Stephanie Chan, Ted Moskovitz, Erin Grant, Andrew Saxe, and
  Felix Hill.
\newblock The transient nature of emergent in-context learning in transformers.
\newblock \emph{Advances in neural information processing systems},
  36:\penalty0 27801--27819, 2023.

\bibitem[Reddy(2024)]{reddy2023mechanistic}
Gautam Reddy.
\newblock The mechanistic basis of data dependence and abrupt learning in an
  in-context classification task.
\newblock In \emph{The Twelfth International Conference on Learning
  Representations}, 2024.

\bibitem[Frankle and Carbin(2019)]{frankle2018lottery}
Jonathan Frankle and Michael Carbin.
\newblock The lottery ticket hypothesis: Finding sparse, trainable neural
  networks.
\newblock In \emph{International Conference on Learning Representations}, 2019.

\bibitem[Olsson et~al.(2022)Olsson, Elhage, Nanda, Joseph, DasSarma, Henighan,
  Mann, Askell, Bai, Chen, Conerly, Drain, Ganguli, Hatfield-Dodds, Hernandez,
  Johnston, Jones, Kernion, Lovitt, Ndousse, Amodei, Brown, Clark, Kaplan,
  McCandlish, and Olah]{olsson2022context}
Catherine Olsson, Nelson Elhage, Neel Nanda, Nicholas Joseph, Nova DasSarma,
  Tom Henighan, Ben Mann, Amanda Askell, Yuntao Bai, Anna Chen, Tom Conerly,
  Dawn Drain, Deep Ganguli, Zac Hatfield-Dodds, Danny Hernandez, Scott
  Johnston, Andy Jones, Jackson Kernion, Liane Lovitt, Kamal Ndousse, Dario
  Amodei, Tom Brown, Jack Clark, Jared Kaplan, Sam McCandlish, and Chris Olah.
\newblock In-context learning and induction heads.
\newblock \emph{Transformer Circuits Thread}, 2022.
\newblock
  https://transformer-circuits.pub/2022/in-context-learning-and-induction-heads/index.html.

\bibitem[Alur et~al.(2025)Alur, Hays, Raghavan, and
  Shah]{alur2025impossibility}
Rohan Alur, Chris Hays, Manish Raghavan, and Devavrat Shah.
\newblock The impossibility of inverse permutation learning in transformer
  models.
\newblock \emph{arXiv preprint arXiv:2509.24125}, 2025.

\bibitem[Schlegel et~al.(2022)Schlegel, Pavlov, and
  Pratt-Hartmann]{schlegel-etal-2022-transformers}
Viktor Schlegel, Kamen Pavlov, and Ian Pratt-Hartmann.
\newblock Can transformers reason in fragments of natural language?
\newblock In Yoav Goldberg, Zornitsa Kozareva, and Yue Zhang, editors,
  \emph{Proceedings of the 2022 Conference on Empirical Methods in Natural
  Language Processing}, pages 11184--11199, Abu Dhabi, United Arab Emirates,
  December 2022. Association for Computational Linguistics.
\newblock \doi{10.18653/v1/2022.emnlp-main.768}.
\newblock URL \url{https://aclanthology.org/2022.emnlp-main.768/}.

\bibitem[Vaswani et~al.(2017)Vaswani, Shazeer, Parmar, Uszkoreit, Jones, Gomez,
  Kaiser, and Polosukhin]{vaswani2017attention}
Ashish Vaswani, Noam Shazeer, Niki Parmar, Jakob Uszkoreit, Llion Jones,
  Aidan~N Gomez, {\L}ukasz Kaiser, and Illia Polosukhin.
\newblock Attention is all you need.
\newblock \emph{Advances in neural information processing systems}, 30, 2017.

\bibitem[Huang et~al.(2023)Huang, Zelikman, Chen, Wu, Valiant, and
  Liang]{huang2023lexinvariant}
Qian Huang, Eric Zelikman, Sarah Chen, Yuhuai Wu, Gregory Valiant, and Percy~S
  Liang.
\newblock Lexinvariant language models.
\newblock \emph{Advances in Neural Information Processing Systems},
  36:\penalty0 23990--24012, 2023.

\bibitem[Friedman et~al.(2023)Friedman, Wettig, and Chen]{friedman2023learning}
Dan Friedman, Alexander Wettig, and Danqi Chen.
\newblock Learning transformer programs.
\newblock \emph{Advances in Neural Information Processing Systems},
  36:\penalty0 49044--49067, 2023.

\bibitem[Piantadosi(2014)]{piantadosi2014zipf}
Steven~T Piantadosi.
\newblock Zipf’s word frequency law in natural language: A critical review
  and future directions.
\newblock \emph{Psychonomic bulletin \& review}, 21\penalty0 (5):\penalty0
  1112--1130, 2014.

\bibitem[Zipf(1949)]{zipf1949principle}
George~Kingsley Zipf.
\newblock \emph{Human Behavior and the Principle of Least Effort: An
  Introduction to Human Ecology}.
\newblock Addison-Wesley Press, Cambridge, MA, 1949.

\bibitem[Frankle et~al.(2021)Frankle, Dziugaite, Roy, and
  Carbin]{frankle2020pruning}
Jonathan Frankle, Gintare~Karolina Dziugaite, Daniel~M Roy, and Michael Carbin.
\newblock Pruning neural networks at initialization: Why are we missing the
  mark?
\newblock In \emph{International Conference on Learning Representations}, 2021.

\bibitem[Paul et~al.(2023)Paul, Chen, Larsen, Frankle, Ganguli, and
  Dziugaite]{paul2022unmasking}
Mansheej Paul, Feng Chen, Brett~W Larsen, Jonathan Frankle, Surya Ganguli, and
  Gintare~Karolina Dziugaite.
\newblock Unmasking the lottery ticket hypothesis: What's encoded in a winning
  ticket's mask?
\newblock In \emph{The Eleventh International Conference on Learning
  Representations}, 2023.

\bibitem[Liu et~al.(2024)Liu, Zhang, He, Wang, Xiao, Ye, Zhou, Ku, and
  Hui]{liu2024survey}
Bohan Liu, Zijie Zhang, Peixiong He, Zhensen Wang, Yang Xiao, Ruimeng Ye, Yang
  Zhou, Wei-Shinn Ku, and Bo~Hui.
\newblock A survey of lottery ticket hypothesis.
\newblock \emph{arXiv preprint arXiv:2403.04861}, 2024.

\bibitem[Elhage et~al.(2021)Elhage, Nanda, Olsson, Henighan, Joseph, Mann,
  Askell, Bai, Chen, Conerly, et~al.]{elhage2021mathematical}
Nelson Elhage, Neel Nanda, Catherine Olsson, Tom Henighan, Nicholas Joseph, Ben
  Mann, Amanda Askell, Yuntao Bai, Anna Chen, Tom Conerly, et~al.
\newblock A mathematical framework for transformer circuits.
\newblock \emph{Transformer Circuits Thread}, 2021.
\newblock https://transformer-circuits.pub/2021/framework/index.html.

\bibitem[Sharma et~al.(2025)Sharma, Rogers, Shapira, and Bau]{sharma2025llms}
Arnab~Sen Sharma, Giordano Rogers, Natalie Shapira, and David Bau.
\newblock Llms process lists with general filter heads.
\newblock \emph{arXiv preprint arXiv:2510.26784}, 2025.

\bibitem[Weiss et~al.(2021)Weiss, Goldberg, and Yahav]{weiss2021thinking}
Gail Weiss, Yoav Goldberg, and Eran Yahav.
\newblock Thinking like transformers.
\newblock In \emph{International Conference on Machine Learning}, pages
  11080--11090. PMLR, 2021.

\bibitem[Xiao et~al.(2024)Xiao, Tian, Chen, Han, and Lewis]{xiao2023efficient}
Guangxuan Xiao, Yuandong Tian, Beidi Chen, Song Han, and Mike Lewis.
\newblock Efficient streaming language models with attention sinks.
\newblock In \emph{The Twelfth International Conference on Learning
  Representations}, 2024.

\bibitem[Willette et~al.(2025)Willette, Lee, and Hwang]{willette2025delta}
Jeffrey Willette, Heejun Lee, and Sung~Ju Hwang.
\newblock Delta attention: Fast and accurate sparse attention inference by
  delta correction.
\newblock \emph{Advances in Neural Information Processing Systems}, 38, 2025.

\bibitem[Darcet et~al.(2024)Darcet, Oquab, Mairal, and
  Bojanowski]{darcet2023vision}
Timoth{\'e}e Darcet, Maxime Oquab, Julien Mairal, and Piotr Bojanowski.
\newblock Vision transformers need registers.
\newblock In \emph{The Twelfth International Conference on Learning
  Representations}, 2024.

\bibitem[Gu and Dao(2024)]{gu2023mamba}
Albert Gu and Tri Dao.
\newblock Mamba: Linear-time sequence modeling with selective state spaces.
\newblock In \emph{First Conference on Language Modeling}, 2024.

\bibitem[Wang et~al.(2025)Wang, Nichani, Bietti, Damian, Hsu, Lee, and
  Wu]{wang2025learning}
Zixuan Wang, Eshaan Nichani, Alberto Bietti, Alex Damian, Daniel Hsu, Jason~D
  Lee, and Denny Wu.
\newblock Learning compositional functions with transformers from easy-to-hard
  data.
\newblock In \emph{Proceedings of the Thirty Eighth Conference on Learning
  Theory}. PMLR, 2025.

\bibitem[Power et~al.(2022)Power, Burda, Edwards, Babuschkin, and
  Misra]{power2022grokking}
Alethea Power, Yuri Burda, Harri Edwards, Igor Babuschkin, and Vedant Misra.
\newblock Grokking: Generalization beyond overfitting on small algorithmic
  datasets.
\newblock \emph{arXiv preprint arXiv:2201.02177}, 2022.

\bibitem[Liu et~al.(2023)Liu, Michaud, and Tegmark]{liu2022omnigrok}
Ziming Liu, Eric~J Michaud, and Max Tegmark.
\newblock Omnigrok: Grokking beyond algorithmic data.
\newblock In \emph{The Eleventh International Conference on Learning
  Representations}, 2023.

\bibitem[Joudaki et~al.(2025)Joudaki, Lanzillotta, Razlighi, Mirzadeh,
  Alizadeh, Hofmann, Farajtabar, and Faghri]{joudaki2025barriers}
Amir Joudaki, Giulia Lanzillotta, Mohammad~Samragh Razlighi, Iman Mirzadeh,
  Keivan Alizadeh, Thomas Hofmann, Mehrdad Farajtabar, and Fartash Faghri.
\newblock Barriers for learning in an evolving world: Mathematical
  understanding of loss of plasticity.
\newblock \emph{arXiv preprint arXiv:2510.00304}, 2025.

\bibitem[Zhang et~al.(2021)Zhang, Jin, Zhang, Zhou, Zhao, Ren, Liu, Wu, Jin,
  and Dou]{zhang2021validating}
Zeru Zhang, Jiayin Jin, Zijie Zhang, Yang Zhou, Xin Zhao, Jiaxiang Ren, Ji~Liu,
  Lingfei Wu, Ruoming Jin, and Dejing Dou.
\newblock Validating the lottery ticket hypothesis with inertial manifold
  theory.
\newblock \emph{Advances in neural information processing systems},
  34:\penalty0 30196--30210, 2021.

\bibitem[Huang et~al.(2026)Huang, Bakalova, Bhattamishra, Merrill, and
  Hahn]{huang2026discovering}
Xinting Huang, Aleksandra Bakalova, Satwik Bhattamishra, William Merrill, and
  Michael Hahn.
\newblock Discovering interpretable algorithms by decompiling transformers to
  rasp, 2026.
\newblock URL \url{https://arxiv.org/abs/2602.08857}.

\end{thebibliography}

\appendix

\newpage
\section{Details of the theoretical results}

\subsection{Overview of the assumptions}
\label{sec:assumptions}

\Cref{thm:gradient-confinement} (and its \cref{thm:invariant-manifold})
holds under the six assumptions stated formally below, each followed by its motivation. 
The first three constrain the \emph{data} distribution (\cref{asm:data-symmetry,asm:lexinvariance,asm:embeddings}); 
the last three fix the \emph{architecture} and parameterization we analyze (\cref{asm:merged-kq,asm:ortho-values}). 
The same assumptions are summarized informally in words just above \cref{thm:gradient-confinement}.

\begin{assumption}[$\setAssoc$-invariance of common tokens]
    \label{asm:data-symmetry}
    Partition the vocabulary into a set of common and rare tokens: $\vocabC \cup \vocabR = \vocab$.
    The correlations among common tokens are limited to a set of symmetric, balanced, non-overlapping associations $\setAssoc \subset [\vocabC]^2$, so that the data distribution is unchanged under a $\perm$ permutation.
\end{assumption}
In other words, $\tokSeq$ and $\perm(\tokSeq)$ are equally likely for any string $\tokSeq \in \vocab^*$ and any permutation $\perm \in \permGroup{\vocab}$ that fixes rare tokens and preserves associations, i.e.\ $\perm(\token) \in \vocabR$ for all $\token \in \vocabR$ and $\{\perm(\itm), \perm(\lbl)\} \in \setAssoc$ for all $\{\itm, \lbl\} \in \setAssoc$. 
We do this to emphasize inductive abilities and limit the room for memorization, considering languages in which common associations are fully interchangeable. 
    
\begin{assumption}[Independent positional offsets]
    \label{asm:data-offsets}
    Each block carries an independent positional offset $\offset_{j,b}$, drawn i.i.d.\ and marginally uniform, so that the data distribution is unchanged under a global per-frequency offset shift $\offset_{j,b} \mapsto \offset_{j,b} + \Delta\offset_j$.
\end{assumption}

Our data uses these random offsets to simulate the presence of an arbitrary number of irrelevant tokens that may separate consecutive blocks in a real document.

\begin{assumption}[Lexinvariance of rare tokens]
    \label{asm:lexinvariance}
    Rare-token embeddings are resampled i.i.d.\ before each sequence from a
    distribution whose mean and second moment match those of the common tokens,
    $\mu = \tfrac{1}{|\vocabC|}\sum_{\token \in \vocabC}\emb{\token}$ and
    $S = \tfrac{1}{|\vocabC|}\sum_{\token \in \vocabC}\emb{\token}\tp{\emb{\token}}$.
\end{assumption}

We target the limit $|\vocabR| \gg |\vocabC|$ in which individual rare-token associations cannot be memorized. 
Following \citet{huang2023lexinvariant}, resampling the embeddings every sequence inhibits memorization and lets us handle rare tokens distributionally.

\begin{assumption}[Orthonormal centered embeddings]
    \label{asm:embeddings}
    The common-token embeddings $\{\emb{\token}\}_{\token\in\vocabC}$ and the
    sinusoidal basis vectors $\{\sinBasis_j, \cosBasis_j\}_{j\in[\numFreqs]}$ are
    mutually orthogonal, with $\tp{\emb{\tok a}}\emb{\tok b} =
        \tp{\emb{\tok a}}\sinBasis_j = \tp{\emb{\tok a}}\cosBasis_j = 0$ for distinct
    $\tok a, \tok b \in \vocabC$ and all $j \in [\numFreqs]$; all common tokens
    share a common norm, $\|\emb{\tok a}\|_2 = \|\emb{\tok b}\|_2$; and the common
    embeddings are \emph{centered}, $\sum_{\token\in\vocabC}\emb{\token} = \zero$.
    In particular $\tokSpace \perp \posSpace$.
\end{assumption}

We adopt the \emph{associative-memory} view of
\citet{bietti2024birth}, in which representations are near-orthogonal; this is
realized by large embedding dimension $\embedDim \gg 1$ (where random vectors are
near-orthogonal), by data whitening, or by one-hot encodings
\citep{nichani2024transformers}. Centering is the standard mean-subtraction
preprocessing, and it is what makes the token commutant in
\cref{proofsec:gradient-blocks} exactly $\vecSpan\{\idT, \commCorrel\}$: without
it a residual rank-one term $\mathbf{1}\tp{\mathbf{1}}$ survives and the manifold is
only approximately invariant. With centering, the matched mean of
\cref{asm:lexinvariance} is $\mu = \zero$.

\begin{assumption}[Merged key--query parameterization]
    \label{asm:merged-kq}
    Queries and keys enter the loss only through the merged product
    $\weights{\ell,h} = \tp{\weightsQ{\ell,h}}\weightsK{\ell,h}$, which we treat as
    the learnable attention parameter of head $(\ell,h)$.
\end{assumption}

The attention scores
$\tp{\hidden{\ell-1}}\tp{\weightsQ{\ell,h}}\weightsK{\ell,h}\hidden{\ell-1}$ depend
on queries and keys only through this product, the ``QK circuit'' of mechanistic
analyses \citep{elhage2021mathematical}. Treating it as a single matrix removes
the gauge freedom $(\weightsQ{}, \weightsK{}) \mapsto (\mathbf{R}\weightsQ{},
    \mathbf{R}\weightsK{})$ and is exactly what the basis weights $\mathcal{W}_\ell$
describe. Relaxing it to independently learnable factors is discussed in
\cref{sec:circuit-folding}.

\begin{assumption}[Fixed orthogonal output--value maps]
    \label{asm:ortho-values}
    Each output--value map $\values{\ell,h} = \weightsP{\ell,h}\weightsV{\ell,h}$ is
    a fixed isometry onto a write subspace that is orthogonal to the residual stream
    entering layer $\ell$ and to the write subspaces of the other heads of layer
    $\ell$, so that $\tp{\values{\ell,h}}\values{\ell,h'} = \delta_{hh'}\id$ on that
    input space.
\end{assumption}

This is the multi-head form of the \emph{disentangled
    residual stream} \citep{friedman2023learning,nichani2024transformers}, again
justified by the associative-memory view: large random projections write into
near-orthogonal subspaces. It is the main simplification of our analysis. The
detailed proof in \cref{sec:invariant-space-proof} establishes the theorem under
it, and \cref{sec:circuit-folding} sketch (without a complete
proof) how learnable value and projection factorizations might be accommodated.

\subsection{Transformer Architecture}
\label{sec:architecture}

In this section, we recap our transformer architecture in a bit more detail.

First, we denote our model's vocabulary as $\vocab$, and we map every token $\token \in \vocab$ to a non-learnable embedding vector $\emb{\token} \in \embedSpace$. 
We use $\embedSymbol : \vocab \rightarrow \embedSpace$ to denote the mapping
$\embed{\token} = \emb{\token}$. Applying this embedding to our input sequence
and target output, we obtain the inputs $\seqEmbed \in \R^{\embedDim \times
        \seqLen}$ and output $\modelOutput \in \outputSpace$.

Second, we add \emph{sinusoidal} position embeddings $\posEmbed \in \posEmbedSpace$ to our model.
As discussed above, we assume each block can be separated by an arbitrary number of irrelevant words, which we model by adding random offsets for each block.
Further, we use $\numFreqs$ frequencies, which gives us:
\begin{equation}
    \posEmbed_{:, b\tupleSize + i} = \sum_{j=1}^{\numFreqs}
    \Big(
    \sin\big(\freq_j i + \offset_{j,b} \big) \, \sinBasis_j +
    \cos\big(\freq_j i + \offset_{j,b} \big) \, \cosBasis_j
    \Big)
\end{equation}
where $\sinBasis_j, \cosBasis_j \in \embedSpace$ are fixed orthogonal basis vectors, $\freq \in \R^\numFreqs$ are fixed frequencies of order $\tupleSize$ (i.e. $\freq_j = 2 \pi N_j / \tupleSize$), and $\offset \in \R^{\numFreqs \times \seqPairs}$ are block offsets sampled i.i.d. for every sequence during training.

Given these token position embeddings, we add them together to get our model's inputs $\hidden{0}
    = \seqEmbed + \posEmbed$. 
We then use an attention-only transformer with
$\numLayers$ attention layers and $\numHeads$ attention heads per layer.
The representations at each layer $\ell \in [\numLayers]$ are given by:
\begin{equation}
    \hidden{\ell} \;=\; \; \hidden{\ell-1} \; + \;\; \sum_{h=1}^{\numHeads} \; \weightsP{\ell,h} \weightsV{\ell,h} \, \hidden{\ell-1} \softmax{\tp{\hidden{\ell-1}} \tp{\weightsQ{\ell,h}} \, \weightsK{\ell,h} \; \hidden{\ell-1}},
\end{equation}
where $\softmax$ denotes the column-wise softmax function with autoregressive masking, $\weights{\ell,h} \in \weightSpace$ denote the merged key-query attention weights for head $h \in [\numHeads]$ in layer $\ell \in [\numLayers]$, $\values{\ell, h} \in \valueSpace$ denote the merged output-value projection weights, and $\hidden{\ell} \in \hiddenSpace$ denote the residual stream at layer $\ell \in [\numLayers]$.

Finally, we get predicted outputs by passing the hidden activations corresponding
to the last input token through a final linear layer:
\begin{equation}
    \outputPredicted \;=\; \weightsOut \, \hidden{\numLayers}_{:, \seqLen}
\end{equation}
where we denote the output weights as $\weightsOut \in \weightSpace$ and the final output as $\outputPredicted \in \embedSpace$.
We train the model to predict the correct token using the cross-entropy loss with a tied unembedding matrix:
\begin{equation}
    \loss \;=\;
    -\tp{\outputPredicted} \modelOutput
    \;+\;
    \log \left(\;
    \sum_{\token \in \vocab} \; e^{\;\tp{\outputPredicted}\emb{\token}}
    \;\right).
\end{equation}

\subsection{Proof of \texorpdfstring{\cref{thm:gradient-confinement}}{Theorem 1}}
\label{sec:invariant-space-proof}

\thmGradientConfinement*

\begin{proof}    
The proof proceeds in six steps. \emph{First} (\cref{proofsec:core_concepts}),
we set up the objects the argument acts on: the orthogonal token and position
embedding subspaces $\tokSpace,\posSpace$, the selection matrices
$\tokActions,\posActions$ that generate the basis weights, and the
representation spaces reachable at each layer. \emph{Second}
(\cref{sec:proof-simplified}), we isolate a simplified single-head subspace with
identity queries and fixed orthogonal value matrices, which already exposes the
mechanism; the learnable, multi-head case is recovered at the end. \emph{Third}
(\cref{proofsec:input-transformations}), we encode the two data symmetries of a
block-list task (permuting associated common tokens and shifting block
positional offsets) as a pair of orthogonal rotations acting on
$\tokSpace,\posSpace$. \emph{Fourth} (\cref{proofsec:forward-pass}), we show that
these rotations leave all attention scores invariant, so the hidden
representations and the model output are simply rotated by the same
transformation; consequently the loss is unchanged. \emph{Fifth}
(\cref{proofsec:backward-pass}), we propagate this
equivariance backward to find how each weight's gradient transforms, and observe
that, because the data distribution is invariant, the \emph{population} gradient
must itself be invariant under the transformation. \emph{Sixth}
(\cref{proofsec:gradient-blocks}), we analyze which matrices can be invariant: a
commutant argument forces every gradient block to vanish across concepts and, within
a concept, to be a polynomial in the selection matrices, exactly the basis-weight
structure $\vecSpan(\mathcal{W}_\ell),\vecSpan(\bar{\mathcal{W}})$ defining
$\manifold$. This establishes $\EV[\nabla_{\weightsAll}\loss]\in\manifold$ and
hence \cref{thm:gradient-confinement}.
\end{proof}

We note that in \cref{sec:circuit-folding} we then lift our argument to independently learnable factorizations via circuit folding.

\subsubsection{Core Concepts}
\label{proofsec:core_concepts}

We begin by introducing several key concepts necessary for our proof.
\paragraph{Embedding Spaces.}
We use $\tokSpace, \,\posSpace \subset \embedSpace$ to denote the vector spaces
containing the \emph{token} and \emph{position} embeddings:
\begin{align}
    \tokSpace \;=\; \vecSpan\big(\{ \, \emb{\token} : \token \in \vocabC\,\}\big),
     &  &
    \posSpace \;=\; \vecSpan\big(\{\, \sinBasis_j, \cosBasis_j: \, j\in [\numFreqs] \,\}\big)
\end{align}
Two remarks are in order. First, $\tokSpace$ is spanned by \emph{common}-token embeddings only: rare tokens have no fixed embedding that could span a subspace, since their embeddings are re-sampled for every sequence under our lexinvariance assumption (\cref{sec:assumptions}); they are instead handled distributionally in the arguments below. Second, $\posSpace$ is spanned by the $2\numFreqs$ sinusoidal basis vectors $\sinBasis_j, \cosBasis_j$ because every position embedding is a linear combination of them (\cref{sec:architecture}): a position corresponds to a rotation angle within each frequency plane, rather than to a separate direction, which is what allows offset shifts to act as rotations of $\posSpace$.

\paragraph{Identity Matrices.}
We use $\idT,\, \idP \in \weightSpace$ to denote the identity transformations
acting only on the subspaces $\tokSpace$ and $\posSpace$, respectively, while collapsing all other directions to zero. 
Formally, they are the orthogonal projections onto these subspaces; under our orthonormality assumptions (\cref{sec:assumptions}), they can be written as
\begin{align}
    \idT \;=\; \sum_{\token \in \vocabC} \emb{\token}\,\tp{\emb{\token}},
     &  &
    \idP \;=\; \sum_{j=1}^{\numFreqs} \Big(\sinBasis_j\, \tp{\sinBasis_j} + \cosBasis_j\, \tp{\cosBasis_j}\Big).
\end{align}
Note that these definitions are basis-independent: the subspaces $\tokSpace$ and $\posSpace$ need not be aligned with the neuron axes.

\mymacro[1]{\repSpace}{\mathbb{U}^{(#1)}}
\paragraph{Representation Spaces.} We use $\repSpace{\ell} \subset \embedSpace$ to denote the vector space
containing the representations at layer $\ell \in [\numLayers]$ such that
$\hidden{\ell} \in \repSpace{\ell}$ always holds. Therefore, we have that
\begin{align}
     & \repSpace{0} = \tokSpace + \posSpace,
     & \repSpace{\ell} = \repSpace{\ell-1} \,+\, {\sum}_{h=1}^{\numHeads} \, \weightsP{\ell,h} \weightsV{\ell,h} \repSpace{\ell-1}, \;\; \forall \ell \in [\numLayers] %
\end{align}
with the sum of vector spaces defined as $\mathcal{A} + \mathcal{B} = \{\, \mathbf{a} + \mathbf{b} \;:\; \mathbf{a} \in \mathcal{A},\, \mathbf{b} \in \mathcal{B} \,\}$ and products of matrices and spaces defined as $\mathbf{A} \mathcal{B} = \{\, \mathbf{A} \mathbf{b} \;:\; \mathbf{b} \in \mathcal{B} \,\}$.
This recursion follows directly from the architecture in \cref{sec:architecture}. The first term, $\repSpace{\ell-1}$, accounts for the residual connection, which lets each layer's representations span the same directions as the previous layer's. Each remaining term, $\weightsP{\ell,h} \weightsV{\ell,h} \repSpace{\ell-1}$, is the ``writing space'' of attention head $h$: the head's update is a convex combination of the previous layer's representations mapped through $\weightsP{\ell,h} \weightsV{\ell,h}$, so its output--value matrices constrain the set of directions in which it can write information.

\subsubsection{Proving the Invariance of a Simplified Space}
\label{sec:proof-simplified}

We will prove the invariance of a simplified transformer where:
(i) only the key matrices $\weightsK{\ell,h}$ and the output matrix $\weightsOut$ are learnable and the others are considered fixed orthogonal matrices; (ii) with a single attention head per layer ($\numHeads = 1$); and (iii) full-dimension head ($\headDim = \embedDim$). 
Extending the proof to multi-head attention is straightforward, and we later informally sketch a `circuit folding' technique that generalizes our proof to a larger class of invariant spaces with arbitrary $\headDim$ and where all matrices are learnable.\looseness=-1

For this simplified calculation thus, we consider full-dimension single-head
attention layers ($\numHeads = 1$ and $\headDim = \embedDim$) and we take the
query matrix to be an identity matrix $\weightsQ{\ell, 1} = \id$. 
(Note that this is mathematically equivalent to considering a merged key--query matrix.)
Regarding projection and value matrices, we assume that they always project their input representations into orthogonal subspace, i.e. $\repSpace{\ell-1} \perp
    \weightsP{\ell, 1}\weightsV{\ell, 1}\repSpace{\ell-1}$. We further simplify
notation by denoting $\weights{\ell} = \weightsK{\ell,1}$ and $\values{\ell} =
    \weightsP{\ell,1}\weightsV{\ell,1}$. This leads to the following simplified
self-attention formulation:
\begin{equation}
    \hidden{\ell} \;=\; \; \hidden{\ell-1} \; + \; \values{\ell} \, \hidden{\ell-1} \softmax{\tp{\hidden{\ell-1}} \, \weights{\ell} \; \hidden{\ell-1}},
\end{equation}

These simplifications are in line with the \emph{associative memory} view of
\citet{bietti2024birth}, where the large, randomly initialized matrices project
their inputs into nearly-orthogonal representations. A related approach was
also used by \citet{nichani2024transformers}, who analyzed a transformer with a ``disentangled''
residual stream \citep{friedman2023learning} where the outputs of different
layers are concatenated rather than added together. Note that this requires an
exponentially growing residual dimensionality $\embedDim =
    \Theta(\numHeads^\numLayers)$.

\subsubsection{Input Transformations}
\label{proofsec:input-transformations}

\mymacro{\rotR}{\mathbf{E}^{(r)}}
\mymacro{\rotT}{\mathbf{E}^{(t)}}
\mymacro{\rotC}{\mathbf{E}^{(c)}}
\mymacro{\rotP}{\mathbf{E}^{(p)}}
\mymacro{\automorph}{\mathrm{Aut}}

We now define two likelihood-preserving rotations to the token and position embeddings. 
For token embeddings, we will apply an association-preserving permutation $\perm \in \permGroup{\vocabC}$ of the common tokens. 
For position embeddings, we apply a frequency-wise change in position encoding offsets $\offset'_{j,b} = \offset_{j,b} + \Delta \mathbf{\offset}_j$, with $\Delta \mathbf{\offset} \in \R^\numFreqs$, which also preserves associations between positions.

We define the transformations of common embeddings $\rotT \in
    \orthoGroup(\tokCommSpace)$ and position embeddings $\rotP \in
    \orthoGroup(\posSpace)$ as:
\begin{align}
     & \rotT \;=   %
    \sum_{\token \,\tiny\in \vocabC}
    \Big( \| \emb{\token} \|_2 \, \| \emb{\perm(\token)} \|_2 \Big)^{-1}
    \,
    \emb{\token} \, \tp{\emb{\perm(\token)}}
    \\
     & \rotP \;=\;
    \sum_{j=1}^{\numFreqs} \;
    \left(\cos\Delta \offset_j \right) \left(\sinBasis_j \, \tp{\sinBasis_j} + \cosBasis_j \, \tp{\cosBasis_j} \right)
    + \left(\sin\Delta \offset_j \right) \left(\cosBasis_j \, \tp{\sinBasis_j} - \sinBasis_j \, \tp{\cosBasis_j} \right)
\end{align}

Together, they define the transformed embeddings
\begin{align}
    \seqEmbedT = \rotT \seqEmbed,
     &  & \posEmbedT = \rotP \posEmbed.
\end{align}

Since our permutation of common tokens preserves the associations and our
offset shift maintains position correlations, we have the following
commutations:
\begin{align}
    \rotC \commCorrel \;=\; \commCorrel \rotC,
     &  &
    \rotP \posCorrel \;=\; \posCorrel \rotP.
\end{align}

\mymacro[1]{\rotH}{\Lambda^{(#1)}}

Note that the token and position subspaces are orthogonal by assumption, $\tokSpace \perp \posSpace$ (\cref{sec:assumptions}), so our transformations do not act outside their respective subspaces, i.e. $\rotP \seqEmbed = \rotT \posEmbed = \zero$.
This allows us to define a global transformation $\rotH{0} \in \orthoGroup(\repSpace{0})$ of the inputs:
\begin{align}
    \hiddenT{0} = \rotH{0} \hidden{0},
     &  & \rotH{0} = \rotT + \rotP.
\end{align}

\subsubsection{Forward Pass}
\label{proofsec:forward-pass}

In this section, we aim to show that applying a transformation $\rotH{0} = \rotT + \rotP$ to our inputs $\hidden{0}$: (i) leaves attention scores unchanged; (ii) rotates the output $\outputPredicted$ by the same rotation $\rotT$; (iii) leaves the loos unchanged.

Before that, however, recall the definition of the invariant subspace, adapted to the simplified setup:
\begin{align}
     & \weights{\ell, *} \;\in\; \weightsSubSpace{\ell} \;=\; \vecSpan\Big(
    \valActions{1:\ell-1} \,
    \big( \tokActions \cup \posActions \big) \,
    \valActions{\ell-1:1}
    \Big),
     & \weightsOut \;\in\; \weightsOutSubSpace \;=\; \vecSpan\Big(
    \tokActions \, \valActions{\numLayers:1}
    \Big),
\end{align}
where we remind the reader that we defined the selection matrices and action spaces as:
\begin{align}
     &
    \tokActions \;=\; \{\; \idT, \commCorrel  \;\},
     &   &
    \valActions{1:\ell} = \left\{\; \id, \,
    \values{\ell, h}
    \;:\;
    h \in [\numHeads]
    \;\right\} \; \valActions{1:\ell-1},
    \\
     &
    \posActions \;=\; \{\; \idP, \posCorrel, \ldots, \posCorrel^{\tupleSize - 1} \;\},
     &   &
    \valActions{\ell:1} = \left\{\;
    \tp{\mathbf{V}}
    \;:\;
    \mathbf{V} \in \valActions{1:\ell}
    \;\right\}.
\end{align}
and that we have the following token and position association matrices
\begin{align}
     & \commCorrel \;=\;
    \sum_{\{ \tok{a}, \tok{b} \} \in \setAssoc}
    \Big( \| \emb{\tok{a}} \|_2 \, \| \emb{\tok{b}} \|_2 \Big)^{-1}
    \,
    \emb{\tok{a}} \, \tp{\emb{\tok{b}}}
    \\
     & \posCorrel \;=\;
    \sum_{j=1}^{\numFreqs} \;
    \left(\cos\freq_j \right) \left(\sinBasis_j \, \tp{\sinBasis_j} + \cosBasis_j \, \tp{\cosBasis_j} \right)
    + \left(\sin\freq_j \right) \left(\sinBasis_j \, \tp{\cosBasis_j} - \cosBasis_j \, \tp{\sinBasis_j} \right)
\end{align}

Note that $\valActions{\ell:1}$ is defined via \emph{transposes}, whereas the main-text definition in \cref{sec:IMIR-theory} uses \emph{pseudo-inverses}. The two coincide in this simplified setup where the fixed $\values{\ell}$ are isometries onto orthogonal subspaces, i.e. $\mathbf{V}^{+} = \tp{\mathbf{V}}$.

\paragraph{Attention scores are unchanged.}
We now aim to establish that attention scores remain unchanged under these transformations. 
Let $\rotH{\ell} \in \orthoGroup(\repSpace{\ell})$ be a layer-wise transformation matrix, for which it holds that:\looseness=-1
\begin{align}
    \hiddenT{0} = \rotH{0} \hidden{0} \implies \hiddenT{\ell} = \rotH{\ell} \hidden{\ell}
\end{align}
Note that for any layer, $\rotH{\ell-1}$ commutes with any $\weights{\ell} \in \weightsSubSpace{\ell}$.
This implies that attention scores are invariant to our input transformations:
\begin{equation}
    \tp{\hiddenT{\ell-1}} \weights{\ell} \hiddenT{\ell-1} =
    \tp{\hidden{\ell-1}} \tp{\rotH{\ell-1}} \weights{\ell} \rotH{\ell-1} \hidden{\ell-1} =
    \tp{\hidden{\ell-1}} \weights{\ell} \hidden{\ell-1}
\end{equation}

\paragraph{Outputs are rotated.}
We can now show the layer-wise transformation matrices can be defined recursively as $\rotH{\ell} = \rotH{\ell-1} + \values{\ell} \rotH{\ell-1} \tp{\values{\ell}}$.
Expanding the definition of $\hiddenT{\ell}$ we get: 
\begin{subequations}    
\begin{align}
    \hiddenT{\ell}
     & = \hiddenT{\ell-1} + \values{\ell} \hiddenT{\ell-1} \softmax{\tp{\hiddenT{\ell-1}} \weights{\ell} \hiddenT{\ell-1}}
    & \mathcomment{definition of $\hiddenT{\ell}$} \\
     & =  \rotH{\ell-1} \hidden{\ell-1} + \values{\ell} \rotH{\ell-1} \hidden{\ell-1} \softmax{\tp{\hidden{\ell-1}} \weights{\ell} \hidden{\ell-1}} 
     \!\!\!\!\!\!\!\!\!\!\!\!\!\!\!\! 
     \\
     & =
    \left( \rotH{\ell-1} + \values{\ell} \rotH{\ell-1} \tp{\values{\ell}} \right)
    \hidden{\ell} 
    & \mathcomment{by $\repSpace{\ell-1} \perp \values{\ell} \repSpace{\ell-1}$} \\
    &= \rotH{\ell} \hidden{\ell}
\end{align}
\end{subequations}
which shows that $\rotH{\ell} = \rotH{\ell-1} + \values{\ell} \rotH{\ell-1} \tp{\values{\ell}}$.
The final layer $\weightsOut \in \weightsOutSubSpace$ projects $\hidden{L}$ back into $\tokSpace$, 
which gives $\hiddenOutT = \rotT \hiddenOut$ and $\outputPredictedT = \rotT \outputPredicted$.

\paragraph{Loss is unchanged.}
Finally, we note that, under our transformations, the target output $\modelOutput$ are rotated by the same transformation as the inputs: $\modelOutputT = \rotT \modelOutput$.
This fact, combined with the previous result imply the loss is unchanged $\lossT = \loss$.

\subsubsection{Backward Pass}
\label{proofsec:backward-pass}

\mymacro[1]{\grad}{\frac{\partial \loss}{ \partial #1}}
\mymacro[1]{\gradT}{\frac{\partial \lossT}{ \partial #1}}
\mymacro[1]{\eGrad}{\frac{\partial \EV \big[\loss \big]}{ \partial #1}}
\mymacro[1]{\eGradT}{\frac{\partial \EV \big[\lossT \big]}{ \partial #1}}

We now backward propagate through our model, and show how the previous transformation affects our gradients.
We also show here that the population gradients are preserved under such a transformation.

\paragraph{Transformed Gradients.}
We start by the final softmax and unembedding layers.
Propagating the loss gradient through them, we get
\begin{equation}
    \gradT{\hiddenOutT} \;=\; \rotT \grad{\hiddenOut}
\end{equation}

Second, the gradient of the final linear layer weights will be:
\begin{equation}
    \gradT{\weightsOut}
    \;=\;
    \hiddenT{\numLayers} \tp{\left( \gradT{\hiddenOutT} \right)}
    \;=\;
    \rotH{\numLayers}\, \hidden{\numLayers} \, \grad{\hiddenOut} \, \tp{\rotT}
    \;=\;
    \rotH{\numLayers}\, \grad{\weightsOut} \, \tp{\rotT}
\end{equation}

A key ingredient of the proof is that attention scores are invariant to our
input transformations: association-preserving permutations don't affect the
attention between a token and itself or its associated token, which are the
only allowed token actions; likewise, the attention between different positions
is not affected by a global offset change. Since the attention scores are
invariant to our transformations, their gradients are also invariant. For
layers $\ell \in [\numLayers]$, we get:
\begin{subequations}    
\begin{align}
    \gradT{\weights{\ell}} & \;=\;
    \hiddenT{\ell-1} \;
    \gradT{ \;\softmax{ \tp{\hiddenT{\ell-1}} \weights{\ell} \hiddenT{\ell-1}} } \;
    \tp{\hiddenT{\ell-1}}
    \\
                           & \;=\;
    \rotH{\ell-1} \;
    \hidden{\ell-1} \;
    \gradT{ \;\softmax{ \tp{\hidden{\ell-1}} \weights{\ell} \hidden{\ell-1}} } \;
    \tp{\hidden{\ell-1}} \;
    \tp{\rotH{\ell-1}}
    \\
                           & \;=\;
    \rotH{\ell-1} \;
    \grad{\weights{\ell}} \;
    \tp{\rotH{\ell-1}}
\end{align}
\end{subequations}

\paragraph{Population Gradients.}
We now note that, since our transformations (by construction) preserve the likelihood of the input--output pairs, they result in an identical data distribution. 
As the loss is also preserved under transformation, this implies that the population gradient is invariant to our transformations:
\begin{equation}
    \EV \Bigg[ \gradT{\weights{\ell}} \Bigg]
    \;=\;
    \eGradT{\weights{\ell}}
    \;=\;
    \eGrad{\weights{\ell}}
    \;=\;
    \EV \Bigg[ \grad{\weights{\ell}} \Bigg]
    .
\end{equation}

\mymacro[1]{\basis}{\mathbf{B}^{(#1)}}

\subsubsection{Gradient Blocks}
\label{proofsec:gradient-blocks}

Finally, in this section we conclude our proof that our model's gradient will live inside the constructed manifold.
To that end, we first analyze which matrices are invariant with respect to the studied transformations.
Rather than analyzing the gradient directly, we decompose it into blocks indexed by pairs of concepts (tokens or positions) and pairs of layer subsets, i.e., by the possible interaction paths through the network.
This decomposition is convenient because each block is easier to treat in isolation, and all blocks follow the same symmetry argument.
Moreover, it loses nothing: the residual-stream activations are supported entirely on the concept subspaces and their images under the composite value maps, so any gradient block not considered below is zero, since the corresponding activations are zero.

For a subset of layers $S \subseteq [\numLayers]$, we define the product of the corresponding value matrices, taken in descending layer order (later layers on the left), matching the composites in $\valActions{1:\numLayers}$:
\begin{equation}
    \values{S} \;=\; \prod_{\ell \;\in\, S} \values{\ell}
\end{equation}

For any layer $\weights{\ell}$ or $\weightsOut$, any two layer subsets $S_1, S_2 \subseteq [\numLayers]$, and any two embedding `concepts' $\tok{x_1}, \tok{x_2} \in \{ \mathrm{t, p} \}$ (i.e., for either tokens or positions), the corresponding gradient block satisfies:
\begin{equation}
    \mathbf{I}^{(\tok{x_1})} \;
    \tp{\values{S_1}} \;
    \eGrad{\weights{\ell}} \;
    \values{S_2} \;
    \mathbf{I}^{(\tok{x_2})}
    \;=\;
    \tp{\mathbf{E}^{(\tok{x_1})}}
    \left(
    \mathbf{I}^{(\tok{x_1})} \;
    \tp{\values{S_1}} \;
    \eGrad{\weights{\ell}} \;
    \values{S_2} \;
    \mathbf{I}^{(\tok{x_2})}
    \right)
    \mathbf{E}^{(\tok{x_2})}
\end{equation}

Recall that our embedding transformations, $\rotT$, and $\rotP$, can be chosen
independently. Moreover, by shifting every position offset by $\Delta \offset'
    = \Delta \offset + \pi$, we obtain $\rotP' = -\rotP$. This implies that any
gradient block mapping different concepts must be zero:
\begin{equation}
    \idT \;
    \tp{\values{S_1}} \;
    \eGrad{\weights{\ell}} \;
    \values{S_2} \;
    \idP
    \; = \;
    \zero
\end{equation}

We are left with the gradient blocks mapping two rotations of the same
embedding space. Because the expected loss $\EV[\loss]$ is computed over a data
distribution that is invariant to our embedding transformations, the expected
gradients must commute with those transformations.

Let $\mathbf{G}^{(p)} = \idP \tp{\values{S_1}} \eGrad{\weights{\ell}}
    \values{S_2} \idP$ and $\mathbf{G}^{(t)} = \idT \tp{\values{S_1}}
    \eGrad{\weights{\ell}} \values{S_2} \idT$. By the chain rule and the invariance
of the loss, the gradient blocks must commute with any valid transformation
matrix:
\begin{align}
    \mathbf{G}^{(p)} \rotP = \rotP \mathbf{G}^{(p)} \quad \forall \rotP \in \orthoGroup(\posSpace), \\
    \mathbf{G}^{(t)} \rotT = \rotT \mathbf{G}^{(t)} \quad \forall \rotT \in \orthoGroup(\tokCommSpace).
\end{align}
We can analyze these two constraints using standard linear algebra and symmetry properties.

\paragraph{Token Subspace.}
The transformation $\rotT$ represents permutations of common tokens that
strictly preserve associations (e.g., swapping one pair of associated tokens
with another, or swapping tokens within a pair).

Let us view $\mathbf{G}^{(t)}$ as a weighted adjacency matrix connecting tokens
in $\tokCommSpace$. For $\mathbf{G}^{(t)}$ to commute with all valid
permutations $\rotT$, the weights must be symmetric and uniform across all
interchangeable token relationships. Because any token pair can be swapped with
any other pair, the diagonal entries of $\mathbf{G}^{(t)}$ (self-connections)
must all be identical. Similarly, because tokens within a pair can be swapped,
the entries connecting a token to its associated pair must all be identical.
Any connection between tokens belonging to \emph{different} pairs would break
the permutation symmetry unless all such connections were identical.

Therefore, $\mathbf{G}^{(t)}$ must have all entries equal except for the
diagonal and the off-diagonal entries corresponding to valid associations, i.e.
$\mathbf{G}^{(t)} = c_1 \mathbf{1}^{(t)} + c_2 \idT + c_3 \commCorrel$. The
bias term $\mathbf{1}^{(t)}$ of the gradient vanishes ($c_1 \approx 0$) when
all tokens have the same L1 norm in the basis of common tokens. In our setup,
all common tokens have the same L2 norm (i.e. $\| \emb{\token} \|_2 =
    \text{const.}$), which ensures equal L1 norm in the token basis. Furthermore,
lex-invariant rare token embeddings are generated with the same mean and
variance as common tokens, which ensures nearly-equal L1 norm when the
dimensionality is large.

\paragraph{Positional Subspace.}
The transformation $\rotP$ represents a continuous shift in the positional
offsets. By definition of our sinusoidal embeddings, $\posSpace$ is spanned by
orthogonal 2D subspaces corresponding to the fundamental frequencies. In this
basis, any rotation $\rotP(\Delta \offset)$ is a block-diagonal matrix where
each $2 \times 2$ block is a standard 2D rotation matrix.

For $\mathbf{G}^{(p)}$ to commute with $\rotP(\Delta \offset)$ for \emph{all}
continuous values of $\Delta \offset$, standard matrix algebra dictates that
$\mathbf{G}^{(p)}$ must also be block-diagonal in the exact same basis, and its
$2 \times 2$ blocks must be scaled 2D rotation matrices themselves. Recall that
we considered that all frequencies have order $\tupleSize$, i.e. $\freq_j =
    2\pi N_j/ \tupleSize$, but there are fundamentally only $\tupleSize$ such
distinct frequencies. Permuting the offsets of equivalent frequencies preserves
the output and loss of the model, hence the gradient sub-block rotations
corresponding to equivalent frequencies must be equal. Therefore, the gradient
contains at most $\tupleSize$ distinct sub-blocks, each of which is a $2 \times
    2$ scaled rotation.

A polynomial of order $\tupleSize$ in $\posCorrel$ (i.e., $\mathbf{G}^{(p)} =
    c_0 \idP + c_1 \posCorrel + \ldots + c_{\tupleSize - 1}
    \posCorrel^{\tupleSize-1}$) requires one additional constraint: sub-blocks
corresponding to opposite frequencies (i.e., $\omega$ and $\omega' = 2\pi
    -\omega$) must be the transpose of each other (i.e., scaled rotations with the
same scaling factors but opposite rotations). To establish this, we must
examine the structure of the gradient in more detail. Note that
$\mathbf{G}^{(p)}$ is the gradient block of an attention matrix
$\weights{\ell}$, which is obtained as:
\begin{align}
    \grad{\weights{\ell}} \;=\;
    \hidden{\ell-1} \, \grad{A} \; \tp{\hidden{\ell-1}},
     &  &
    \text{where } A = \softmax{\tp{\hidden{\ell-1}} \, \weights{\ell} \; \hidden{\ell-1}}.
\end{align}

In particular, our gradient block will be $\mathbf{G}^{(p)} = \EV[ {\sum}_{ij}
        (\partial \loss / \partial A_{ij}) \, p_i \tp{p_j}]$, where $p_i,p_j \in
    \posSpace$ are the positional components at two locations $i,j \in [\seqPairs
        \tupleSize]$ in the sequence. Since attention score gradients $\partial \loss /
    \partial A_{ij}$ are shared across frequencies, it suffices to focus on a
single attention path. Isolating the $2 \times 2$ gradient sub-block
corresponding to a frequency $\omega$, we get $\mathbf{G}_{ij\omega} = \EV[
        p_{\omega i} \tp{p_{\omega j}} ]$. Further using the fact that $p_{\omega i} =
    \posCorrel_{\omega}^{i-j} p_{\omega j} = \tp{\posCorrel_{-\omega}^{i-j}}
    p_{\omega j}$ and taking the expectation over global offsets, which preserve
attention scores and their gradients, we obtain that opposite-frequency
sub-blocks are transpositions of each other.

By combining the structural constraints of both the token and positional
subspaces, we recover the desired geometry for all gradient blocks.

\paragraph{Multi-Head Attention.}
\label{sec:multi-head}
The above proof can be easily extended to multi-head attention:
\begin{equation}
    \hidden{\ell} \;=\; \; \hidden{\ell-1} \; + \; \sum_{h=1}^{\numHeads} \; \values{\ell,h} \, \hidden{\ell-1} \softmax{\tp{\hidden{\ell-1}} \, \weights{\ell,h} \; \hidden{\ell-1}}.
\end{equation}

By applying the same input transformations as before, we obtain slightly more
complex global transformations $\rotH{\ell} = \rotH{\ell-1} +
    \sum_{h=1}^{\numHeads} \values{\ell,h} \rotH{\ell-1} \tp{\values{\ell,h}}$,
where the orthogonality of the per-head write subspaces (\cref{asm:ortho-values})
keeps the heads' contributions independent. Looking at the gradient blocks indexed
by concepts and sets of heads, the desired structure can be established using
identical arguments.

\subsection{Circuit Folding: Invariant Dynamics with Interacting Subspaces}
\label{sec:circuit-folding}

\Cref{thm:gradient-confinement} assumes fixed orthogonal output--value maps
(\cref{asm:ortho-values}); relaxing this is more delicate, and we record the
extension here as a sketch rather than a proven claim. Allowing the value maps
$\values{\ell,h}$ to be learned lets concept subspaces rotate or mix, and allowing
the key--query and output--value factorizations to be learned independently
(rather than only through their merged products, \cref{asm:merged-kq}) requires
tracking the intermediate representations rather than just those products. We
expect the same block-level arguments to extend to factorizations in which
subspaces interact only within a single concept (tokens or positions) and are
transformed only by their fundamental actions; this appendix sets up
the circuit-folding formalism needed to make this precise. We emphasize that this
is a sketch, and leave a complete treatment to future work.

\mymacro{\w}{\mathbf{W}}
\mymacro{\h}{\mathbf{H}}
\mymacro{\B}{\mathbf{B}}
\mymacro{\bTI}{\mathcal{B}_{t}^i}
\mymacro{\bTO}{\mathcal{B}_{t}^o}
\mymacro{\bPI}{\mathcal{B}_{p}^i}
\mymacro{\bPO}{\mathcal{B}_{p}^o}

To illustrate the idea, consider any learnable matrix $\w \in \R^{d_2 \times
        d_1}$ of the network. For any orthogonal input and output concept spaces, we
can establish the following invariant manifold:
\begin{equation}
    \w \;\in\; \vecSpan(\tp{\bTO} \tokActions \bTI + \tp{\bPO} \posActions \bPI)
\end{equation}
where $\bTI \subset \R^{d_t \times d_1}, \bTO \subset \R^{d_t \times d_2}, \bPI \subset \R^{d_p \times d_1}, \bPO \subset \R^{d_p \times d_2}$ denote the sets of orthogonal basis vectors for the orthogonal input/output token/position spaces.

Consider a sequence of representations $\h \in \R^{d_i \times n}$ that we
transform by applying the same input transformation as before. Each token space
is transformed with $\rotT \in \orthoGroup(\tokCommSpace)$ and each position
space is transformed with $\rotP \in \orthoGroup(\posSpace)$. We define the
global input transformation:
\begin{equation}
    \Lambda_i = \sum_{\B \in \bTI}  \tp{\B} \rotT \B \;+\; \sum_{\B \in \bPI}  \tp{\B} \rotP \B.
\end{equation}

Since output subspaces are connected only to input subspaces of the same
concept via transformations that commute with the input transformations, we
obtain the following output transformation:
\begin{align}
    \w \Lambda_i \h \;=\; \Lambda_o \w \h,
     &  &
    \Lambda_o = \sum_{\B \in \bTO}  \tp{\B} \rotT \B \;+\; \sum_{\B \in \bPO}  \tp{\B} \rotP \B.
\end{align}

Similarly, if the output gradient is transformed via $\Lambda_o$, the
propagated input gradient will be transformed via $\Lambda_i$. Computing the
transformed gradient blocks of $\w$ and applying the same reasoning as in the
original proof, we can show that the gradient of $\w$ fits the desired
structure.

\subsection{Beyond Attention Only: Feed-Forward Networks \& Layer Norm}
\label{sec:beyond-attention-only}

We extend \cref{thm:invariant-manifold} to standard Transformer layers, which
interleave multi-head attention with two further operations: position-wise
feed-forward networks (FFN) and layer normalization (LN). Both operations
commute with the embedding-space symmetries used in
\cref{sec:invariant-space-proof}, provided their parameters fit the natural
extension of the IMIR structure described below. The argument is the same in
spirit as the multi-head extension of \cref{sec:multi-head}: one writes down a
global residual-stream rotation $\rotH{\ell}$, checks that the new operation
commutes with it, and propagates the block-level analysis through the new
gradient blocks.

We adopt the notation of \cref{sec:invariant-space-proof}: $\rotT \in
    \orthoGroup(\tokCommSpace)$ is the orthogonal map induced by an
association-preserving permutation of common tokens, $\rotP \in
    \orthoGroup(\posSpace)$ is the rotation on the position subspace induced by a
global offset shift, and $\rotH{\ell}$ is the composite rotation on the
residual stream at layer~$\ell$.

\paragraph{Layer Normalization.}

A LayerNorm at layer $\ell$ rescales each column of the residual stream:
\begin{equation}
    \mathrm{LN}(\hidden{\ell})_{:, i}
    \;=\; \mathbf{g} \odot \frac{\hidden{\ell}_{:, i} - \mu_i \mathbf{1}}{\sigma_i}
    \;+\; \mathbf{b},
\end{equation}
where $\mu_i, \sigma_i \in \R$ are the column mean and standard
deviation and $\mathbf{g}, \mathbf{b} \in \embedSpace$ are learnable
gain and bias vectors. Two facts make LN compatible with
$\rotH{\ell-1}$.
\begin{enumerate}[leftmargin=*]
    \item The column statistics $\mu_i, \sigma_i$ depend on the residual column only
          through inner products with the all-ones direction $\mathbf{1}$ and with
          itself. Both $\rotT$ and $\rotP$ act on $\tokCommSpace$ and $\posSpace$
          respectively, and these subspaces are orthogonal to $\mathbf{1}$ in our setup
          (common-token and sinusoidal position embeddings are zero-mean by
          construction). Hence $\mu_i$ and $\sigma_i$ are unchanged by $\rotH{\ell-1}$.
    \item Provided the gain and bias decompose into the same fundamental actions that
          span the IMIR --- i.e.\ $\mathbf{g}, \mathbf{b} \in \vecSpan(\tokActions \cup
              \posActions)$ applied to $\mathbf{1}$, with the special case $\mathbf{g} =
              \mathbf{1}$, $\mathbf{b} = \zero$ recovering RMSNorm --- the elementwise
          multiplication and addition commute with $\rotH{\ell-1}$ in the same
          block-diagonal sense as $\commCorrel$ and $\posCorrel$ in the proof of
          \cref{thm:invariant-manifold}.
\end{enumerate}
Together,
\begin{equation}
    \mathrm{LN}(\rotH{\ell-1} \hidden{\ell-1})
    \;=\;
    \rotH{\ell-1}\,\mathrm{LN}(\hidden{\ell-1}),
\end{equation}
and the gradient of $\mathrm{LN}$ with respect to its input and to
$\mathbf{g}, \mathbf{b}$ inherits the block structure of
\cref{sec:invariant-space-proof} verbatim.

\paragraph{Feed-Forward Networks.}

Each FFN is a position-wise two-layer MLP,
\begin{equation}
    \mathrm{FFN}(\hidden{\ell})_{:, i}
    \;=\; W^{(\mathrm{out})}\,
    \sigma\!\big(W^{(\mathrm{in})} \hidden{\ell}_{:, i}\big),
    \qquad W^{(\mathrm{in})} \in \R^{d_h \times \embedDim},\;
    W^{(\mathrm{out})} \in \R^{\embedDim \times d_h},
\end{equation}
with $\sigma$ an elementwise non-linearity (e.g.\ GELU). Biases are
omitted for brevity; they are handled like the LN bias above.

Define the FFN counterpart of the IMIR by applying the circuit-folding
construction of \cref{sec:circuit-folding} to each FFN matrix.
$W^{(\mathrm{in})}$ maps the residual stream into the FFN-hidden space: let
$\bTO, \bPO$ denote the orthogonal basis matrices of the token-only and
position-only blocks of the FFN-hidden axes (with the FFN-hidden dimension
partitioned as $d_h = d_h^{(t)} + d_h^{(p)}$). $W^{(\mathrm{out})}$ maps back,
with the same FFN-hidden bases playing the role of input. We require
\begin{align}
    W^{(\mathrm{in})}
     & \;\in\;
    \vecSpan\!\big(\tp{\bTO}\,\tokActions\,\valActions{\ell:1}\;+\;\tp{\bPO}\,\posActions\,\valActions{\ell:1}\big),
    \\
    W^{(\mathrm{out})}
     & \;\in\;
    \vecSpan\!\big(\valActions{1:\ell}\,\tokActions\,\bTO\;+\;\valActions{1:\ell}\,\posActions\,\bPO\big),
\end{align}
which is exactly the circuit-folding form
$\w\in\vecSpan(\tp{\bTO}\tokActions\bTI+\tp{\bPO}\posActions\bPI)$ of
\cref{sec:circuit-folding}, instantiated with $\bTI = \valActions{\ell:1}$
and $\bPI = \valActions{\ell:1}$ for $W^{(\mathrm{in})}$, and with the
roles of input and output swapped for $W^{(\mathrm{out})}$.

\emph{Forward invariance.} Define the FFN-hidden rotation
$\rotH{\ell}^{(\mathrm{hid})} = \tp{\bTO}\rotT\bTO + \tp{\bPO}\rotP\bPO$,
matching $\rotH{\ell-1}$ on the residual side. By the circuit-folding
lemma (\cref{sec:circuit-folding}), $W^{(\mathrm{in})}\rotH{\ell-1} =
    \rotH{\ell}^{(\mathrm{hid})} W^{(\mathrm{in})}$ and, by the symmetric
statement for $W^{(\mathrm{out})}$, $W^{(\mathrm{out})}
    \rotH{\ell}^{(\mathrm{hid})} = \rotH{\ell-1} W^{(\mathrm{out})}$. The
non-linearity $\sigma$ is applied coordinate-wise within each
FFN-hidden sub-block. Because $\rotH{\ell}^{(\mathrm{hid})}$ permutes
coordinates only within the token sub-block (via $\rotT$, which is a
permutation of the orthonormal token basis vectors) and only within
the position sub-block (via the block-diagonal $\rotP$), and because
$\sigma$ is shared across coordinates of the same sub-block, $\sigma$
commutes with $\rotH{\ell}^{(\mathrm{hid})}$. Composing,
\begin{equation}
    \mathrm{FFN}(\rotH{\ell-1}\hidden{\ell-1})
    \;=\;
    \rotH{\ell-1}\,\mathrm{FFN}(\hidden{\ell-1}),
\end{equation}
so that the residual update $\hidden{\ell} = \hidden{\ell-1} +
    \mathrm{FFN}(\hidden{\ell-1})$ also commutes with $\rotH{\ell-1}$.

\emph{Backward invariance.} The gradients
\begin{align}
    \grad{W^{(\mathrm{out})}}
    \; & =\;
    \grad{\hidden{\ell}} \;
    \tp{\big(\sigma(W^{(\mathrm{in})} \hidden{\ell-1})\big)},
    \\
    \grad{W^{(\mathrm{in})}}
    \; & =\;
    \Big(\tp{W^{(\mathrm{out})}} \, \grad{\hidden{\ell}} \;\odot\;
    \sigma'(W^{(\mathrm{in})} \hidden{\ell-1})\Big)\;
    \tp{\hidden{\ell-1}}
\end{align}
are outer products in which one factor lives in the residual stream
(rotated by $\rotH{\ell-1}$) and the other in the FFN-hidden space
(rotated by $\rotH{\ell}^{(\mathrm{hid})}$). The same gradient-block
analysis as in \cref{sec:invariant-space-proof} (cross-concept blocks
vanish under $\rotP \mapsto -\rotP$; same-concept blocks commute with
$\rotT$ resp.\ $\rotP$) implies that the expected gradient blocks are
polynomials in the fundamental actions and therefore fit the
prescribed FFN-IMIR structure.

\paragraph{Putting it together.}

Inserting any number of FFN and LN layers, in any order, between the attention
layers does not break \cref{thm:invariant-manifold}. The extended manifold
consists of the original attention-weight constraints of
\cref{sec:invariant-space-proof}, plus the FFN-weight constraints above, plus
the LN gain/bias constraints. The forward pass commutes with $\rotH{\ell}$ at
every layer (attention, FFN, or LN); the loss is therefore invariant under our
embedding rotations; the expected gradient is invariant too; and the same
block-by-block argument forces each gradient block --- attention, FFN, or LN
--- into the prescribed polynomial form. The full Transformer thus retains the
same low-dimensional invariant manifold as the attention-only model.

\newpage
\section{Arbitrary Relational Structure}
\label{sec:relational-structure}

In the main text we specialized the discussion to irreflexive symmetric binary
relations, which already captures the bulk of the inductive tasks studied in
the literature. The framework, however, applies to arbitrary relational
structures on the vocabulary. In this appendix we describe the general
construction and then work out the directed pairwise case as a concrete
extension. The two ingredients are always the same: a relation $\setAssoc$
induces a token symmetry group, the symmetry group induces orthogonal
transformations on representations, and the corresponding commutant pins down
the admissible operators that learning dynamics can produce.

\subsection{Automorphism Groups of General Relations}
\label{sec:relational-structure-aut}

Let
\begin{align}
    \setAssoc \subseteq \vocab^\arity
\end{align}
be an $\arity$-ary relation on the vocabulary $\vocab = \{1,\dots,m\}$, characterizing
the basic relational skills required by the task. Examples include binary
associations, equivalence relations, transition structures of a Markov chain,
or higher-order dependencies. The relation $\setAssoc$ induces a subgroup of
the symmetric group, the \emph{automorphism group} of $\setAssoc$:
\begin{align}
    \operatorname{Aut}(\vocab,\setAssoc)
    &:=
    \bigl\{
        \perm \in \permGroup{m}
        \;:\;
        \tokSeq \in \setAssoc \Longleftrightarrow \perm \circ \tokSeq \in \setAssoc
    \bigr\}, \\
    \perm \circ (w_1,\dots,w_\arity)
    &:= (\perm(w_1),\dots,\perm(w_\arity)). \nonumber
\end{align}
Its elements are precisely the permutations of $\vocab$ that preserve the
relational structure of $\setAssoc$.

\subsection{Orthogonal Action on Representations}

To translate vocabulary symmetries into representation-space symmetries we
exploit the orthogonality of the embedding matrix. As in the main text we
assume $\tp{\embedMatrix}\embedMatrix = \id$ and, for simplicity of
exposition, $\embedDim = m$, so that no embedding dimensions are redundant.
Every permutation $\perm \in \operatorname{Aut}(\vocab,\setAssoc)$ induces a
permutation matrix $\mb P_\perm \in \R^{m\times m}$ acting on basis vectors by
$\mb P_\perm \mb e_i = \mb e_{\perm(i)}$, and the corresponding orthogonal
action on representations is
\begin{align}
    \mb U_\perm
    \;=\;
    \embedMatrix \mb P_\perm \tp{\embedMatrix},
    \qquad
    \mb U_\perm \emb{i} \;=\; \emb{\perm(i)}.
\end{align}
The induced symmetry group on representation space is
\begin{align}
    \mc U
    :=
    \bigl\{
        \mb U_\perm
        \;:\;
        \perm \in \operatorname{Aut}(\vocab, \setAssoc)
    \bigr\}
    \;\subseteq\;
    O(\embedDim),
\end{align}
and the admissible operators are exactly those that commute with every
element of $\mc U$:
\begin{align}
    \mc C(\mc U)
    :=
    \bigl\{
        \mb W \in \R^{\embedDim\times\embedDim}
        \;:\;
        \mb W \mb U = \mb U \mb W
        \;\;\forall\, \mb U \in \mc U
    \bigr\}.
\end{align}
By the gradient-invariance argument of \cref{sec:IMIR-theory}, the population
gradient flow preserves $\mc C(\mc U)$, so $\mc C(\mc U)$ provides the
elementary admissible building blocks of any reasoning circuit, regardless
of the arity or structure of $\setAssoc$.

\subsection{Directed Pairwise Relations}

To illustrate the construction beyond the symmetric pairwise case treated in
the main text, consider an \emph{ordered} (directed) binary relation. Write
\begin{align}
    \setAssoc
    =
    \{(\itm_1,\lbl_1),\dots,(\itm_\vocabPairs,\lbl_\vocabPairs)\}
    \;\subseteq\;
    \vocab \times \vocab,
\end{align}
where every vocabulary item appears at most once, and let
\begin{align}
    A = \{\itm_1,\dots,\itm_\vocabPairs\},
    \qquad
    B = \{\lbl_1,\dots,\lbl_\vocabPairs\},
    \qquad
    D = \vocab \setminus (A \cup B)
\end{align}
denote the sources, the targets, and the unrelated tokens. The automorphisms
of $\setAssoc$ jointly permute the directed pairs while preserving
orientation, and independently permute the unrelated tokens:
\begin{align}
    \itm_i \mapsto \itm_{\sigma(i)},
    \qquad
    \lbl_i \mapsto \lbl_{\sigma(i)},
    \qquad
    d_j \mapsto d_{\tau(j)},
    \qquad
    \sigma \in S_\vocabPairs,
    \;
    \tau \in S_{m-2\vocabPairs}.
\end{align}
Because sources and targets cannot be exchanged, the admissible operators now
distinguish four pair-local roles. Define the role-preserving projectors and
role-switching couplings
\begin{align}
    \id_A
    &:=
    \sum_{i=1}^\vocabPairs \emb{\itm_i} \tp{\emb{\itm_i}},
    \qquad
    \id_B
    :=
    \sum_{i=1}^\vocabPairs \emb{\lbl_i} \tp{\emb{\lbl_i}},
    \qquad
    \id_D
    :=
    \sum_{d \in D} \emb{d} \tp{\emb{d}}, \\
    \commCorrel_{A\to B}
    &:=
    \sum_{i=1}^\vocabPairs \emb{\lbl_i} \tp{\emb{\itm_i}},
    \qquad
    \commCorrel_{B\to A}
    :=
    \sum_{i=1}^\vocabPairs \emb{\itm_i} \tp{\emb{\lbl_i}}.
\end{align}
A direct orbit count shows that, after centering the embeddings within each
symmetry class so that the global rank-one couplings drop out, the commutant
reduces to the pair-local algebra
\begin{align}
    \mc C(\mc U)
    \;=\;
    \operatorname{span}
    \bigl\{
        \id_A,\; \id_B,\; \id_D,\;
        \commCorrel_{A\to B},\; \commCorrel_{B\to A}
    \bigr\}.
\end{align}
This generalizes the unordered case in two ways. First, the identity
$\id = \id_A + \id_B + \id_D$ splits into separate role projectors, so
attention heads can now condition on whether a token plays the role of a
source, of a target, or of an unrelated filler. Second, the symmetric
association operator $\commCorrel = \commCorrel_{A\to B} + \commCorrel_{B\to A}$
of the main text splits into two oriented operators, allowing the circuit
algebra to encode directional relations such as $\itm \to \lbl$ without
collapsing onto its inverse.

Before centering, the full commutant additionally contains uniform coupling
operators of the form $\mb J_{RS} := \sum_{r\in R}\sum_{s\in S}\emb{r}\tp{\emb{s}}$,
for role types $R,S \in \{A,B,D\}$. These operators correspond to nonzero
mean directions of the symmetry classes; centering the embeddings within each
class removes them, and the local algebra above is recovered as the essential
symmetry-compatible structure. The construction of the basis weights
$\mathcal{W}_\ell$ in \cref{sec:IMIR-theory} then proceeds verbatim, with the
token selection set $\tokActions = \{\idT, \commCorrel\}$ replaced by the
directed selection set $\{\id_A, \id_B, \id_D, \commCorrel_{A\to B},
\commCorrel_{B\to A}\}$, and the invariant manifold theorem
(\cref{thm:invariant-manifold}) carries over without further change.

\subsection{General Recipe}

The two examples treated so far -- symmetric pairwise relations in the main
text and directed pairwise relations above -- follow the same recipe, and the
recipe extends to arbitrary $\setAssoc$:
\begin{enumerate}[leftmargin=*]
    \item Compute the automorphism group
    $\operatorname{Aut}(\vocab,\setAssoc)$ of the relation, and partition
    $\vocab$ into orbits of role-equivalent tokens.
    \item Form the induced orthogonal symmetry group $\mc U \subseteq O(\embedDim)$
    via $\mb U_\perm = \embedMatrix \mb P_\perm \tp{\embedMatrix}$.
    \item Compute the commutant $\mc C(\mc U)$. Its dimension equals the number
    of orbits of $\operatorname{Aut}(\vocab,\setAssoc)$ acting on $\vocab\times\vocab$,
    and a basis is given by the orbit-indicator operators
    $\sum_{(\token,\token')\in \mathcal{O}}\emb{\token}\tp{\emb{\token'}}$,
    indexed by orbits $\mathcal{O}$.
    \item Use $\mc C(\mc U)$ as the token selection set in place of
    $\tokActions$ and form the basis weights $\mathcal{W}_\ell$ as in
    \cref{sec:IMIR-theory}.
\end{enumerate}
For binary $\setAssoc$ defining a directed graph, this recipe identifies the
isomorphic connected components of the graph; the symmetric pair and the
directed pair are the simplest such components. For higher-arity relations
the orbit count grows accordingly, and the commutant captures progressively
richer relational structure (such as equivalence classes, $n$-cycles, or
transition kernels). In every case the same gradient-invariance argument
applies and the corresponding basis weights remain invariant under population
gradient flow.

\newpage
\section{The Geometry of Reasoning Circuits}
\label{sec:geometry}

In this appendix, we discuss the geometry of the Invariant Manifold of
Inductive Reasoning (IMIR) in more detail. We first give an intuitive
description of its coordinate frame, and then we give a theoretical motivation
for its origin based on representational symmetries.

\subsection{What is an Attention Head?}
\label{sec:IMIR-intuitive}

Now that our key theoretical result has been established, we give an intuitive
description of its geometry. Using a toy attention example, we introduce the
mechanisms of the circuits that exist in the Invariant Manifold of Inductive
Reasoning. Consider the input embeddings $\mathbf{h}^{0}_i = \emb{\tok{a}} +
    \mathbf{p}_i$ (where $\mathbf{h}^{0}_i \in \embedSpace$) of token $\tok{a} \in
    \vocab$ at position $i \in [\seqLen]$. What can an attention head do with this
input?

\paragraph{Token Actions.}
In our setup, we consider that the only correlations present in the data are
between non-overlapping pairs of associated tokens in $\setAssoc$, as formalized by the $\setAssoc$-invariance assumption of \cref{sec:assumptions}.
Thus, for each token, the model can learn only two possible mappings: to the
same token or to the associated token. Applying this to attention heads, we get
two possible token actions. First, an attention head may look for the exact
token it has, constructing key and query vectors that align when the input
tokens match, i.e. $\weightsK{1,1} \,\emb{\tok{a}} \approx \weightsQ{1,1} \,
    \emb{\tok{b}}$ when $\tok{a} = \tok{b}$. This is achieved when the key-query
circuit is approximately a multiple of the identity transformation over the
subspace of token embeddings, i.e. $\tp{\weightsK{1,1}} \weightsQ{1,1} \approx
    \alpha \idT$, where $\idT = \sum_{\token \in \vocab} \emb{\token}
    \tp{\emb{\token}}$. The second possible action for the attention head is to
look for the associated token, i.e. $\weightsQ{1,1} \, \emb{\tok{a}} \approx
    \weightsK{1,1} \, \emb{\tok{b}}$ when $\{\tok{a}, \tok{b}\} \in \setAssoc$. In
this case, the key-query circuit must act as a mapping between associated
tokens, i.e. $\tp{\weightsK{1,1}} \weightsQ{1,1} \approx \alpha \commCorrel$,
where $\commCorrel = \sum_{\{\tok{a}, \tok{b}\} \in \setAssoc} (\emb{\tok{a}}
    \tp{\emb{\tok{b}}} + \emb{\tok{b}} \tp{\emb{\tok{a}}})$.

\paragraph{Position Actions.}
In our setup, each block has a random positional offset, simulating an
arbitrary (unknown) number of imaginary tokens separating consecutive blocks.
This breaks positional correlations between separate blocks, which means that
positional embeddings may only be used to attend to a previous position in the
same block (up to $\tupleSize - 1$ positions back). Indeed, since block offsets are sampled independently, the relative position between tokens of different blocks is uniformly random and carries no usable information. 
This is achieved by constructing key and query vectors that align for a certain position distance, i.e. $\weightsK{1,1}\, \textbf{p}_{i}
    \approx \weightsQ{1,1} \, \textbf{p}_{i-j}$ for a distance $j \in [\tupleSize -
        1]$. Using translation-invariant sinusoidal embeddings, this is achieved when
the key-query circuit encodes a specific rotation $\tp{\weightsK{1,1}}
    \weightsQ{1,1} \approx \alpha \posCorrel^j$,
where $\posCorrel = \EV[\mathbf{p}_{i-1} \tp{\mathbf{p}_i}]$ and $\posCorrel^j$ denotes the $j$-th matrix power of $\posCorrel$. The power arises because $\posCorrel$ rotates every frequency plane by its unit-position angle, mapping each position embedding to that of the preceding position ($\posCorrel\, \mathbf{p}_i = \mathbf{p}_{i-1}$); applying it $j$ times thus shifts positions back by $j$, i.e. $\posCorrel^j\, \mathbf{p}_i = \mathbf{p}_{i-j}$. 
Deep circuits might also benefit from attending to the exact same position using $\tp{\weightsK{1,1}} \weightsQ{1,1} \approx \alpha \idP$, where $\idP = \EV[\mathbf{p}_{i} \tp{\mathbf{p}_i}]$.

To see how this extends to stacked attention heads, assume that the first
attention head
has attended entirely to a token $\tok{b} \in \vocab$ at position $j \in
    [\seqLen]$. The residual stream at position $i$ has become
$\mathbf{h}^{1}_i = \emb{\tok{a}} + \mathbf{p}_i + \weightsP{1,1}
    \weightsV{1,1} (\emb{\tok{b}} + \mathbf{p}_j)$. How could a second attention
head leverage these representations?

\paragraph{Output-Value Projections.}
While the previously described actions remain perfectly valid, a whole new
range of possibilities is unlocked. A second head can now attend using the
information retrieved by the first head. For example, the second head can
attend to positions containing the token retrieved by the first head if
$\weightsK{2,1} (\weightsP{1,1} \weightsV{1,1} \emb{\tok{b}}) \approx
    \weightsQ{2,1} \emb{\tok{b}}$ or $\tp{\weightsK{2,1}} \weightsQ{2,1} \approx
    \weightsP{1,1} \weightsV{1,1} \idT$. The information retrieved by the first
head can also be used by the second head when constructing queries. There are
twice as many ways to construct keys and queries, so the number of total
possible behaviors has increased four-fold.

\paragraph{Combined Projections.} To get an intuition for even deeper circuits, assume that the second attention
head at position $i$ has attended entirely to a token $\tok{a}' \in \vocab$ at
position $i' \in [\seqLen]$. Moreover, assume that the first head
at position $i'$ has attended entirely to a token $\tok{b}' \in \vocab$ at
position $j' \in [\seqLen]$, so that $\mathbf{h}^{1}_{i'} =
    \emb{\tok{a}'} + \mathbf{p}_{i'} + \weightsP{1,1} \weightsV{1,1}
    (\emb{\tok{b}'} + \mathbf{p}_{j'})$. Therefore, after the second head, the
residual stream at position $i$ will be $\mathbf{h}^{2}_i = \emb{\tok{a}} +
    \mathbf{p}_i + \weightsP{1,1} \weightsV{1,1} (\emb{\tok{b}} + \mathbf{p}_j) +
    \weightsP{2,1} \weightsV{2,1} (\emb{\tok{a}'} + \mathbf{p}_{i'}) +
    \weightsP{2,1} \weightsV{2,1} \weightsP{1,1} \weightsV{1,1} (\emb{\tok{b}'} +
    \mathbf{p}_{j'})$. Note that the residual stream is now a sum of eight embedding terms --- the four token embeddings $\emb{\tok{a}}, \emb{\tok{b}}, \emb{\tok{a}'}, \emb{\tok{b}'}$ and the four position embeddings $\mathbf{p}_i, \mathbf{p}_j, \mathbf{p}_{i'}, \mathbf{p}_{j'}$ --- each rotated into its own subspace by the output--value projections of the heads that wrote it.
    Each of them
    could be used by a third head to construct queries and keys by accessing the
    corresponding subspaces.\footnote{The number of combined projections across
        layers is analogous to the number of possible code calls when converting a
        transformer into D-RASP, a programming language proposed to abstract
        transformers \citep{huang2026discovering}.}

    \subsection{Representational Symmetry of Inductive Tasks}
    \label{sec:IMIR-symmetry}

    \paragraph{Token Embeddings.} We use a data representation with fixed, non-learnable embeddings, mapping
    every $\mb x \in \mc V$ to a vector $\mb e_{\mb x} \in \R^d$, resulting in the
    embedding matrix $\mb E \in \R^{m \times d}$. However, in the context of
    induction head learning, the model should not memorize specific embeddings in
    its parameters. Rather, it should learn to associate tokens with certain roles,
    as is the case in the class of block-list tasks described above. Hence, we want
    the model weights to be invariant under transformation symmetries $S$ of the
    data representation $\mb E$: What is learned with $\mb E$ should generalize to
    all input representations $\mb E' = \mb U \mb E$, with $\mb U \in S$.

    Specifically, natural transformation symmetries to consider are subgroups of
    orthogonal transformations $S \subseteq O(d)$. As is well established, such
    orthogonal invariances can be tied to the commutant of the symmetry group
    \begin{align}
        \mc C(S) := \{ \mb W:  \mb W \mb U = \mb U \mb W ,\; \forall\, \mb U \in S\} \,.
    \end{align}
    It can be shown easily that, for instance, equivariance of linear operators and invariance of bi-linear operators constrains them to be in $\mc C(S)$. The larger $S$, the more constrained are operators, for instance, $\mc C(O(d)) = \operatorname{span}(\mb I)$ and invariant operators can only re-scale data.

    What is then a suitable choice for $S$? Our focus is on retaining associative
    pairwise relations between vocabulary items. Hence, we assume that the task
    implies a symmetric, irreflexive binary relation $\sim$ on $\mc V$ such that
    each element is related to \emph{exactly} one other element (a perfect matching). 
    This induces a subgroup $S\subseteq O(d)$ consisting of permutation matrices $\mb U_\pi$ corresponding
    to permutations $\pi$ of $\mc V$ that preserve the relation, i.e.~$\mb x \sim
\mb y \Longleftrightarrow \pi( \mb x) \sim \, \pi(\mb y)$. The commutation
    relation implies that operators in $\mc C(S)$ are constant on orbits of index
    pairs under $S$. There are three such orbits
    \begin{align}
        \mc O_1  = \{ (\mb x, \mb x): \mb x \in \mc V \}, \quad
        \mc O_2 = \{ (\mb x, \mb y): \mb x \sim \mb y, \; \mb x, \mb y \in \mc V\}, \quad
        \mc O_3 = \overline{\mc O}_1 \cap \overline{\mc O}_2\,.
    \end{align}
    and the commutant is given by
    \begin{align}
        \mc C(S) =
        \operatorname{span}\{ \mb I, \mb C, \mb 1 \mb 1^\top\}, \quad
        \mb C := \mb 1\{x \sim y\}\,.
    \end{align}
    Note that by centering embeddings such that $\mb 1^\top \mb E= \mb 0$ one can eliminate $\mb 1 \mb 1^\top$. The interpretation is very clear: For each token, a (bi-)linear map can learn only one of two possible things (and superpositions thereof): auto-association of a token with itself (i.e.~$\mb I$) or hetero-association of a token with its partner (i.e.~$\mb C$).

    \paragraph{Positional Embeddings.}
    Token positions are encoded via sinusoidal embeddings as in
    \citep{vaswani2017attention}. To model the fact that relevant blocks may appear
    at arbitrary positions, we add random offsets for each block. Let $\omega_i \in
\R_+$, $i\in[k]$, denote the embedding frequencies. For a token at relative
    position $t$ inside block $b$, with offset $\phi_b$, we define
    \begin{align*}
        p_{b,t}
        =
        \bigl(
        \sin(\omega_1(t+\phi_b)),
        \cos(\omega_1(t+\phi_b)),
        \dots,
        \sin(\omega_k(t+\phi_b)),
        \cos(\omega_k(t+\phi_b))
        \bigr)^\top
        \in \R^{2k}.
    \end{align*}

    In the spirit of the previous paragraph, we require that proper induction does
    not learn or memorize absolute token positions, but only relative ones. We thus
    require invariance under relative position shifts. It is a key property of
    sinusoidal embeddings that shifts in position correspond to rotations within
    each subspace $V_j$. More abstractly, the above positional embeddings admit a
    decomposition into orthogonal two-dimensional subspaces corresponding to
    different frequencies, i.e.~$\R^{2k} = \bigoplus_{j} V_j$. The corresponding
    subgroup $P \subseteq O(d)$ consists of block-diagonal orthogonal matrices of
    the form
    \begin{align}
        \mb U = \operatorname{diag}(\mb U_1, \dots, \mb U_k),
        \quad \mb U_j \in SO(2).
    \end{align}
    As before, symmetry of the data distribution implies that admissible operators must commute with all $\mb U \in P$, which already constrains them substantially. Commutation alone, however, does not fully determine the structure used in the main text: additional properties of the sinusoidal embeddings, unrelated to this symmetry, further reduce the admissible operators to the desired polynomial structure in the positional shift operator (see \cref{sec:invariant-space-proof}).

    \newpage
    \section{In-Context Learning vs.\ In-Weights Learning}

    \subsection{Training details}
    \label{sec:icl-vs-iwl-training}

    We train a 2-layer, 1-head-per-layer attention-only transformer of width $128$
    on an in-context bigram prediction task with a vocabulary of $16$ common tokens
    and $64$ rare tokens, plus sinusoidal positional embeddings with 4 frequencies
    (8 features). Each sequence consists of $4$ bigram demonstrations followed by a
    query, giving a context length of $9$. Tokens within each sequence are sampled
    as rare with probability $0.9$, with single-shot demonstrations within context
    and zero per-sequence frequency variation. To facilitate interpretability, we
    train the simplified architecture described in \cref{sec:proof-simplified},
    with merged key-query matrix, fixed output-value projections, and orthogonal
    representations. The width of $128$ is chosen so that the model exactly admits
    the geometric construction of the bigram circuit on the manifold (base
    dimension $32$, branching $(1+H)^{L}=4$, hence $4\cdot32=128$).
    We optimize
    with vanilla SGD at learning rate $0.1$ and batch size $256$ for $250$ steps;
    the random seed is fixed at $0$. Every $10$ steps we estimate top-1 accuracy on
    two probe distributions, each averaged over $8$ fresh batches: an \emph{ICL
        probe} in which every token is rare, so that the query bigram can only be
    solved by attending to its in-context demonstration, and an \emph{IWL probe} in
    which every token is common and the context is shuffled so the target pair is
    absent, forcing reliance on the in-weights bigram statistics. For the
    constrained training, we use the same optimization hyperparameters, while
    ablating all model parameters except for $\alpha$, $\beta$, $\gamma$, and
$\delta$ after every step of gradient descent.

    \paragraph{Compute resources.}
    The model has approximately $0.05\,$M learnable parameters; one full training
    run (250 SGD steps at batch size 256, plus the periodic ICL and IWL probes)
    takes well under a minute on a single consumer-grade GPU and uses less than
$100\,$MB of GPU memory. Constrained training adds a single per-step IMIR
    projection, which is dominated by a constant-size matrix multiply and does not
    change the wall-clock cost in any noticeable way.

    \subsection{Circuit Ablations}
    \label{sec:icl-vs-iwl-ablations}

    In this appendix, we causally test the circuit interpretation given in \cref{sec:icl-vs-iwl}. We take the model trained as described in \cref{sec:icl-vs-iwl-training} and ablate all of its parameters except for a chosen subset of the four highlighted IMIR directions; we then re-evaluate the training loss and the ICL and IWL probe accuracies (defined in \cref{sec:icl-vs-iwl-training}). Each row of \cref{tab:ablation-results} corresponds to one such subset. The results confirm our interpretation: keeping only $(\alpha, \beta, \gamma)$ preserves ICL accuracy ($0.75$, vs.\ $0.77$ for the unablated model) while destroying IWL; keeping only $\delta$ yields near-perfect IWL accuracy and no ICL; and removing any single direction from $(\alpha, \beta, \gamma)$ destroys ICL entirely, confirming that these three directions jointly implement the induction head.

    \begin{table}[h]
        \centering
        \caption{Ablation study confirming that the $\alpha\beta\gamma$-circuit implements In-Context Learning (ICL), while the $\delta$-circuit accounts for In-Weights Learning (IWL).}
        \label{tab:ablation-results}
        \begin{tabular}{lccc}
            \toprule
            \textbf{Ablated except}         & \textbf{Train Loss} & \textbf{ICL Accuracy} & \textbf{IWL Accuracy} \\
            \midrule
            all (original)                  & 1.41                & \textbf{0.77}         & 0.22                  \\
            $\alpha, \beta, \gamma, \delta$ & \textbf{1.06}       & \textbf{0.76}         & 0.14                  \\
            $\alpha, \beta, \gamma$         & \textbf{1.03}       & \textbf{0.75}         & 0.10                  \\
            $\delta$                        & 4.34                & 0.01                  & \textbf{0.99}         \\
            $\alpha, \beta$                 & 4.38                & 0.00                  & 0.06                  \\
            $\alpha, \gamma$                & 3.99                & 0.05                  & 0.00                  \\
            $\beta, \gamma$                 & 7.47                & 0.03                  & 0.00                  \\
            \bottomrule
        \end{tabular}
    \end{table}

    \newpage
    \subsection{All circuits during training}
    \label{sec:icl-vs-iwl-all-circuits}

    As described in \cref{sec:icl-vs-iwl}, the IMIR of our two-layer, single-head model is 28-dimensional. We additively decompose the model's key--query and output weights into these 28 basis directions, and we track the resulting coefficients throughout training. \cref{fig:circuits-all} plots every coefficient, across the data regimes considered in \cref{sec:icl-vs-iwl}. Only the four highlighted directions ($\alpha$, $\beta$, $\gamma$, and $\delta$) grow appreciably during training; this is the evidence behind our claim in \cref{sec:icl-vs-iwl} that learning concentrates on these four directions.

    \begin{figure}[h]
        \centering
        \captionsetup{width=0.9\linewidth}
        \includegraphics[width=\linewidth]{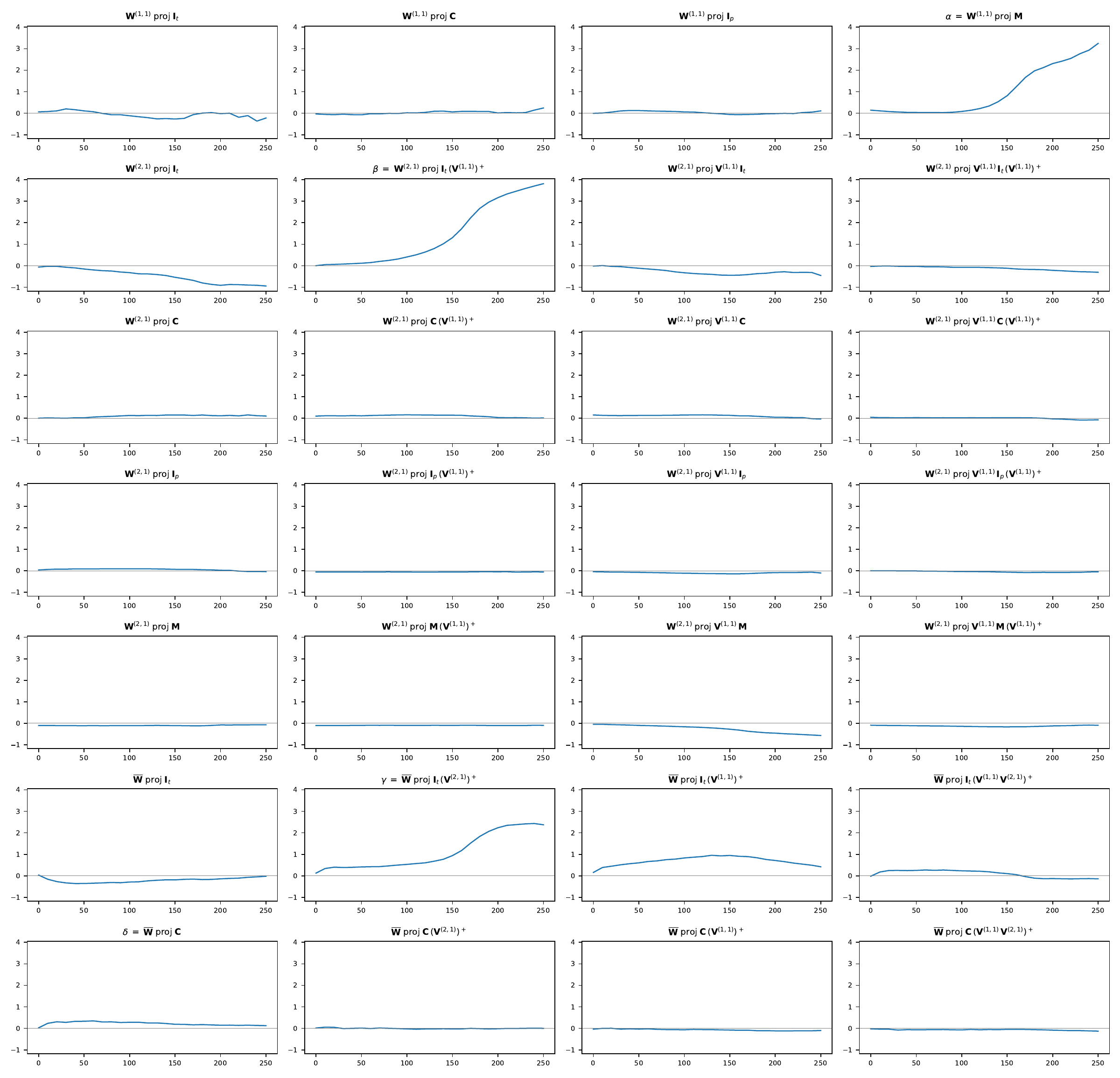}
        \caption{All 28 IMIR coefficients of a 2-layer, single-head Transformer trained on in-context bigrams, tracked over 250 training steps: the 4 first-layer directions (top row), the 16 second-layer directions (middle rows), and the 8 output directions (bottom rows). Only the four highlighted directions ($\alpha$, $\beta$, $\gamma$, $\delta$) grow appreciably during training.}
        \label{fig:circuits-all}
    \end{figure}

    \newpage
    \subsection{Proof of \texorpdfstring{\cref{thm:icl-rare-tokens}}{Theorem 2}}
    \label{sec:proof-rare-tokens}

    \thmICLrareTokens*

    We give the detailed proof of \cref{thm:icl-rare-tokens}, following the proof
    sketch: we partition the data distribution into the part on which the IWL
    solution succeeds and the remainder, and we show that the IWL solution kills
    the ICL gradient on the former while leaving it unchanged on the latter.

    We work in the simplified architecture of \cref{sec:proof-simplified} with the
    in-context bigram task ($\tupleSize = 2$); the argument extends to arbitrary
$\tupleSize$ without modification. The ICL induction circuit is specified by
    three pseudo-parameters $\alpha, \beta, \gamma$ via
    \begin{align}
        \weights{1} \;=\; \alpha\, \posCorrel,
         &  &
        \weights{2} \;=\; \beta\, \idT\, \tp{\values{1}},
         &  &
        \weightsOut \;=\; \gamma\, \idT\, \tp{\values{2}},
    \end{align}
    and the IWL solution by a single pseudo-parameter $\delta$ via
$\weightsOut = \delta\, \commCorrel$.
    We write $\weightsAll = (\alpha,\beta,\gamma,0)$ for an ICL-only configuration
    and $\weightsAll + \weightsAll_\mathrm{IWL} = (\alpha,\beta,\gamma,\delta)$ for
    the configuration with the IWL solution superimposed on top, so that
$\weightsOut = \gamma\, \idT\, \tp{\values{2}} + \delta\, \commCorrel$.

    \paragraph{Setup: the ICL gradient, the regime, and the disjointness assumption.}
    We pin down the three objects the comment flags as under-specified.

    \emph{\textbf{(A1) The ICL gradient is a scalar.}} Throughout we use the single
    definition of \cref{sec:proof-burstiness}: $\nabla_\mathrm{ICL}\loss$ is the
    directional derivative of the loss along the unit ICL tuple direction
    $\hat e_\mathrm{ICL} \propto (\weightsICL{1},\weightsICL{2},\weightsOutICL)$ under
    the Frobenius inner product,
    \begin{equation}
        \begin{aligned}
            \nabla_\mathrm{ICL}\loss
            \;&=\; \big\langle \weightsICL{1}, \grad{\weights{1}}\big\rangle
                + \big\langle \weightsICL{2}, \grad{\weights{2}}\big\rangle
                + \big\langle \weightsOutICL, \grad{\weightsOut}\big\rangle
            \;=\; \big( \mathbf{p}(\mathbf{z}) - \modelOutput \big)^{\!\top} J,
            \\[2pt]
            J \;&:=\; \frac{\mathrm{d} \mathbf{z}_\mathrm{ICL}}{\mathrm{d} t}\Big|_{t=0} \in \R^{|\vocab|},
        \end{aligned}
        \label{eq:icl-grad-def}
    \end{equation}
    where $J$ is the logit-space directional Jacobian along $\hat e_\mathrm{ICL}$ and
    $J_x := \langle \emb{x}, J\rangle$ is its coordinate in the (orthonormal) embedding
    basis; the chain rule gives the second equality because the loss sees the ICL
    direction only through the logits $\mathbf{z}$. This collapses the
    three pseudo-parameter partials $\partial\mathbf{z}_\mathrm{ICL}/\partial\theta$
    of the sketch into one logit-space vector $J$, reconciling the notation with
    \cref{sec:proof-burstiness}.

    \emph{\textbf{(A2) Geometry.}} Token embeddings are orthonormal (the
    associative-memory idealization; lex-invariant rare embeddings share the common
    tokens' norm and are pairwise orthogonal up to $O(|\vocab|^{-1})$ fluctuations),
    under the standing assumptions of \cref{sec:proof-simplified}, and all gradients
    are population (data-averaged) gradients, consistent with
    \cref{thm:invariant-manifold}.

    \emph{\textbf{(A3) Asymptotic regime.}} We work in the \emph{saturated-IWL,
    large-vocabulary} regime: the sequence length $\seqLen$, the number of blocks
    $\seqPairs$, the common-vocabulary size $|\vocabC|$, and the ICL readout scale
    $\gamma = O(1)$ are all held \emph{fixed}, while $|\vocab| \to \infty$ (the rare
    vocabulary grows) and the IWL scale $\delta \to \infty$ jointly, fast enough that
    $|\vocab|\,e^{-\delta} \to 0$ (equivalently $\delta - \log|\vocab| \to \infty$).
    Because $\gamma = O(1)$, the ICL logits are $O(1)$ and the softmax partition
    function is $\Theta(|\vocab|)$, so every off-target probability is
    $O(|\vocab|^{-1})$ --- the fact that drives both error terms.

    \emph{\textbf{(A4) Distractor disjointness.}} For a common query with canonical
    partner $\token_q'$, no distractor block carries $\token_q'$ as an item or a
    label; i.e.\ the query's canonical association is token-disjoint from the
    distractor context. Under (A4) the induction circuit never reads $\token_q'$, so
    its logit-space contribution --- and hence $J_{\token_q'}$ --- is pure
    orthonormality fluctuation, $J_{\token_q'} = O(|\vocab|^{-1})$. This is the clean
    fix for the cross-talk flagged above; \cref{rem:finite-vocab-crosstalk} quantifies
    the deviation of the actual finite-vocabulary sampler from~(A4).

    \medskip
    We summarize the residual error of the IWL solution by
    \begin{equation}
        \epsilon_{\mathrm{IWL}} \;:=\;
        \EV_{\mathcal{D}_{c,\mathrm{can}}}\!\big[\, 1 - \mathbf{p}(\mathbf{z})_{\token_q'} \,\big]
        \;=\; O\!\big(|\vocab|\, e^{-\delta}\big),
    \end{equation}
    the probability mass that the saturated model fails to place on the canonical
    token on the part of the data the IWL solution covers; under (A3),
    $\epsilon_{\mathrm{IWL}} \to 0$, and the rate is derived in \cref{lem:iwl-solvable}.

    \paragraph{Output decomposition.}
    Let $q$ denote the query position. The simplified architecture yields the
    orthogonal decomposition
    \begin{equation}
        \hiddenlast{2} \;=\;
        \underbrace{\emb{\token_q}}_{\in\, \tokSpace}
        \;+\; \underbrace{\mathbf{p}_q}_{\in\, \posSpace}
        \;+\; \underbrace{\values{1} \mathbf{u}^{(1)}}_{\in\, \values{1} \repSpace{0}}
        \;+\; \underbrace{\values{2} \mathbf{u}^{(2)}}_{\in\, \values{2} \repSpace{1}},
    \end{equation}
    with the four addends in pairwise orthogonal subspaces (cf.\
    \cref{sec:proof-simplified}). The map $\commCorrel$ acts only on the common-token
    subspace and vanishes on the other three subspaces, so
    \begin{equation}
        \commCorrel\, \hiddenlast{2} \;=\; \commCorrel\, \emb{\token_q}
        \;=\;
        \begin{cases}
            \emb{\token_q'} & \text{if } \token_q \in \vocabC \text{ with canonical partner } \token_q', \\
            \zero           & \text{if } \token_q \in \vocabR.
        \end{cases}
    \end{equation}
    The model logits at $q$ therefore split into an ICL part and an IWL part,
    \begin{equation}
        \mathbf{z}
        \;=\; \weightsOut\, \hiddenlast{2}
        \;=\; \mathbf{z}_\mathrm{ICL}(\alpha,\beta,\gamma)
        \;+\; \delta\, \commCorrel\, \emb{\token_q},
        \qquad
        \mathbf{z}_\mathrm{ICL} \;=\; \gamma\, \idT\, \tp{\values{2}}\, \hiddenlast{2},
    \end{equation}
    where the IWL part depends on $\delta$ but not on $\alpha,\beta,\gamma$.

    \begin{lemma}[The ICL Jacobian is blind to the IWL solution]
        \label{lem:icl-blind}
        The directional Jacobian $J$ of \eqref{eq:icl-grad-def} is independent of the
        IWL scale $\delta$, so the only $\delta$-dependence of $\nabla_\mathrm{ICL}\loss$
        is funneled through the probabilities:
        \begin{equation}
            \nabla_\mathrm{ICL} \loss
            \;=\;
            \big( \mathbf{p}(\mathbf{z}) - \modelOutput \big)^{\!\top} J,
            \qquad
            \mathbf{p}(\mathbf{z})_x \;=\; \frac{e^{\tp{\emb{x}} \mathbf{z}}}{\sum_{x'} e^{\tp{\emb{x'}} \mathbf{z}}}.
            \label{eq:icl-grad-template}
        \end{equation}
        Moreover $J$ obeys two structural bounds:
        \emph{(J1)} an $\ell_1$ bound $\|J\|_1 = \sum_x |J_x| = O(1)$; and
        \emph{(J2)} under \emph{(A4)}, $|J_{\token_q'}| = O(|\vocab|^{-1})$.
    \end{lemma}
    \begin{proof}
        For each $\theta \in \{\alpha,\beta,\gamma\}$, $\partial \hiddenlast{2} /
        \partial \theta$ lives entirely in $\values{1} \repSpace{0} + \values{2}
        \repSpace{0} + \values{2} \repSpace{1}$ (it is built from differentiating
        the layer-$1$/layer-$2$ attention scores and attended content, all produced
        by the value rotations of those layers). All three subspaces are orthogonal
        to $\tokSpace$, so $\commCorrel$ annihilates them, $\commCorrel\,
        \partial \hiddenlast{2} / \partial \theta = \zero$, and hence
        $\partial \mathbf{z}/\partial\theta = \partial\mathbf{z}_\mathrm{ICL}/
        \partial\theta$ is independent of $\delta$; assembling the three partials into
        the directional derivative \eqref{eq:icl-grad-def} makes $J$ $\delta$-independent.
        \emph{(J1):} $J = \mathrm{d}\mathbf{z}_\mathrm{ICL}/\mathrm{d}t$ is, by the
        induction-circuit construction, a combination of in-context token embeddings
        weighted by layer-$2$ attention weights (which sum to $1$) and the bounded
        value rotations, scaled by $\gamma = O(1)$; in the orthonormal embedding basis
        its coordinates therefore have $\ell_1$ mass $O(1)$.
        \emph{(J2):} the induction circuit reads only in-context tokens; under (A4)
        the canonical partner $\token_q'$ appears in no block, so $J_{\token_q'}$
        collects only the $O(|\vocab|^{-1})$ orthonormality fluctuations.
    \end{proof}

    \paragraph{Splitting the data into ``IWL-solvable'' and ``IWL-unhelpful''.}
    The data distribution $\mathcal{D}(\freqR, \freqWCV, \burstiness)$ partitions
    into three sub-distributions, $\mathcal{D}_r$ (rare query, frequency $\freqR$),
$\mathcal{D}_{c,\mathrm{can}}$ (common query whose in-context partner equals
    the canonical partner $\token_q'$, frequency $(1-\freqR)(1-\freqWCV)$), and
$\mathcal{D}_{c,\mathrm{scr}}$ (common query whose in-context partner differs
    from $\token_q'$, frequency $(1-\freqR) \freqWCV$). We bundle the latter two
    according to whether the IWL solution succeeds at solving them.

    \begin{lemma}[IWL-solvable part]
        \label{lem:iwl-solvable}
        On $\mathcal{D}_{c,\mathrm{can}}$ (mass $(1-\freqR)(1-\freqWCV)$),
        \begin{equation}
            \nabla_\mathrm{ICL} \loss_{c,\mathrm{can}}\big(\weightsAll + \weightsAll_\mathrm{IWL}\big)
            \;=\; O(\epsilon_{\mathrm{IWL}})
            \;=\; O\!\big(|\vocab|\, e^{-\delta}\big)
            \;\xrightarrow[\;\delta \to \infty\;]{}\; 0 .
        \end{equation}
    \end{lemma}
    \begin{proof}
        Here the target equals the canonical partner, $\modelOutput = \emb{\token_q'}$,
        and both the IWL term $\delta\, \emb{\token_q'}$ and the ICL readout
        $\gamma\, \emb{\token_q'}$ add to the \emph{correct} logit, while every competing
        logit stays $O(1)$ by orthonormality. Hence the correct-token probability is
        \begin{equation}
            \mathbf{p}(\mathbf{z})_{\token_q'}
            \;=\;
            \frac{e^{\gamma + \delta}}{e^{\gamma + \delta} + \sum_{x \neq \token_q'} e^{\tp{\emb{x}}\mathbf{z}}}
            \;=\;
            1 - O\!\big(|\vocab|\, e^{-(\gamma+\delta)}\big)
            \;=\;
            1 - \epsilon_{\mathrm{IWL}},
        \end{equation}
        the second equality because the denominator's $|\vocab|-1$ off-target logits are
        each $O(1)$ by orthonormality. So the residual is $\|\mathbf{p}(\mathbf{z}) - \modelOutput\| =
        O(\epsilon_{\mathrm{IWL}})$. By \cref{lem:icl-blind}, $J = O(1)$ is independent
        of $\delta$ (J1), so \eqref{eq:icl-grad-template} gives a \emph{suppressed} (not
        exactly cancelled) ICL gradient of order $\epsilon_{\mathrm{IWL}} = O(|\vocab|\,
        e^{-\delta})$, decaying exponentially in the IWL margin $\delta$.
    \end{proof}

    \begin{lemma}[IWL-unhelpful part]
        \label{lem:iwl-unhelpful}
        On $\mathcal{D}_r \cup \mathcal{D}_{c,\mathrm{scr}}$ (mass $\freqR +
        (1-\freqR)\freqWCV$), under \emph{(A4)},
        \begin{equation}
            \nabla_\mathrm{ICL} \loss_r\big(\weightsAll + \weightsAll_\mathrm{IWL}\big)
            = \nabla_\mathrm{ICL} \loss_r\big(\weightsAll\big),
            \quad
            \nabla_\mathrm{ICL} \loss_{c,\mathrm{scr}}\big(\weightsAll + \weightsAll_\mathrm{IWL}\big)
            = \nabla_\mathrm{ICL} \loss_{c,\mathrm{scr}}\big(\weightsAll\big)
            + O\!\big(|\vocab|^{-1}\big).
        \end{equation}
    \end{lemma}
    \begin{proof}
        \emph{Rare queries.} On $\mathcal{D}_r$ the IWL term is identically zero
        ($\commCorrel\, \emb{\token_q} = \zero$), so $\mathbf{z}$ and hence
        $\mathbf{p}(\mathbf{z})$ are unchanged by adding $\weightsAll_\mathrm{IWL}$;
        with $J$ also $\delta$-independent (\cref{lem:icl-blind}) the gradient is
        \emph{exactly} unchanged.

        \emph{Scrambled common queries.} The target is the in-context partner
        $\token^* \neq \token_q'$, and the IWL contribution $\delta\, \emb{\token_q'}$
        adds logit mass to a \emph{wrong} token. By \cref{lem:icl-blind} only
        $\mathbf{p}(\mathbf{z})$ carries the $\delta$-dependence, so with
        $\Delta\mathbf{p} := \mathbf{p}(\mathbf{z} + \delta\, \emb{\token_q'}) -
        \mathbf{p}(\mathbf{z})$,
        \begin{equation}
            \nabla_\mathrm{ICL} \loss_{c,\mathrm{scr}}\big(\weightsAll + \weightsAll_\mathrm{IWL}\big)
            - \nabla_\mathrm{ICL} \loss_{c,\mathrm{scr}}\big(\weightsAll\big)
            \;=\; \Delta\mathbf{p}^{\!\top} J
            \;=\; \sum_x \Delta p_x\, J_x .
            \label{eq:scr-perturbation}
        \end{equation}
        Adding $\weightsAll_\mathrm{IWL}$ only rescales the softmax normalizer, so for
        every $x \neq \token_q'$,
        \begin{equation}
            \mathbf{p}(\mathbf{z} + \delta\, \emb{\token_q'})_x
            \;=\;
            \frac{\mathbf{p}(\mathbf{z})_x}{1 + (e^{\delta} - 1)\, \mathbf{p}(\mathbf{z})_{\token_q'}}
            \;\le\; \mathbf{p}(\mathbf{z})_x
            \qquad (x \neq \token_q') .
        \end{equation}
        Hence $\max_{x \neq \token_q'} |\Delta p_x| \le \max_{x \neq \token_q'}
        \mathbf{p}(\mathbf{z})_x = O(|\vocab|^{-1})$ by (A3): under orthonormal
        embeddings and $O(1)$ logits the partition function is $\Theta(|\vocab|)$, so
        each off-target mass is $O(|\vocab|^{-1})$. The single coordinate $x =
        \token_q'$ instead absorbs
        $\Theta(1)$ of mass, $\Delta p_{\token_q'} = \Theta(1)$. We therefore split it
        off and bound the remainder by H\"older against the Jacobian $\ell_1$ norm:
        \begin{equation}
            \Big| \sum_x \Delta p_x\, J_x \Big|
            \;\le\;
            \underbrace{\big|\Delta p_{\token_q'}\big|\,\big|J_{\token_q'}\big|}_{\Theta(1)\,\cdot\, O(|\vocab|^{-1}) \text{ by (J2)}}
            \;+\;
            \underbrace{\Big(\max_{x \neq \token_q'} |\Delta p_x|\Big) \sum_{x \neq \token_q'} |J_x|}_{O(|\vocab|^{-1})\,\cdot\, O(1) \text{ by (J1)}}
            \;=\; O\!\big(|\vocab|^{-1}\big).
            \label{eq:holder-split}
        \end{equation}
        This is where \emph{(A4)} is load-bearing: it gives $|J_{\token_q'}| =
        O(|\vocab|^{-1})$ (J2), without which the first term would be $\Theta(1)$ on
        the event that $\token_q'$ appears as a distractor token
        (\cref{rem:finite-vocab-crosstalk}). Note that the per-coordinate bound
        $|\Delta p_x| = O(|\vocab|^{-1})$ does \emph{not} suffice on its own --- summed
        over $|\vocab|$ coordinates it gives $O(1)$ --- which is why the bound is
        carried as $\ell_\infty(\Delta\mathbf{p}) \cdot \ell_1(J)$.
    \end{proof}

    \begin{lemma}[IWL-free gradients agree across sub-distributions]
        \label{lem:subdist-agree}
        Under \emph{(A4)}, the IWL-free per-sub-distribution ICL gradients agree up to
        $O(|\vocab|^{-1})$:
        \begin{equation}
            \nabla_\mathrm{ICL} \loss_r(\weightsAll)
            \;=\;
            \nabla_\mathrm{ICL} \loss_{c,\mathrm{can}}(\weightsAll)
            \;=\;
            \nabla_\mathrm{ICL} \loss_{c,\mathrm{scr}}(\weightsAll)
            \;=\;
            \nabla_\mathrm{ICL} \loss(\weightsAll \,|\, \freqR = 1)
            \;+\; O\!\big(|\vocab|^{-1}\big).
        \end{equation}
    \end{lemma}
    \begin{proof}
        At $\weightsAll$ (no IWL) the logits are $\mathbf{z}_\mathrm{ICL}$ alone, which
        by construction of the induction circuit equals (a multiple of) the in-context
        partner embedding $\emb{\token^*}$ in all three sub-distributions --- the
        canonical association is never read, only the in-context demonstration is, and
        (A4) ensures the canonical partner does not re-enter as a distractor token.
        Together with the lex-invariance of rare embeddings and the orthonormality of
        common-token embeddings, this gives a per-sequence loss depending on the target
        only through $\gamma$ and $|\vocab|$, not on whether the target is rare or
        common. The three gradients therefore agree, exactly in the orthonormal
        idealization and up to $O(|\vocab|^{-1})$ finite-size corrections otherwise; by
        linearity of expectation the data-averaged $\nabla_\mathrm{ICL}\loss(\weightsAll)$
        equals this common value up to the same order.
    \end{proof}

    \begin{remark}[Finite-vocabulary cross-talk: deviation from (A4)]
        \label{rem:finite-vocab-crosstalk}
        The sampler of \cref{sec:data-properties} draws the $\Theta(\seqPairs)$
        distractor associations independently and does \emph{not} enforce~(A4): a
        distractor common block coincides with the query's canonical pair with
        probability $\Theta(1/|\vocabC|)$ per block, so the canonical partner
        $\token_q'$ surfaces in the context with probability $\Theta(\seqPairs\,
        \freqC/|\vocabC|)$ --- governed by the \emph{common}-vocabulary size $|\vocabC|$
        (e.g.\ $16$ in our experiments), not the total $|\vocab|$. On that event the
        induction circuit \emph{does} read $\token_q'$, so $|J_{\token_q'}| = \Theta(1)$
        and the first term of \eqref{eq:holder-split} is $\Theta(1)$ rather than
        $O(|\vocab|^{-1})$. Carrying it explicitly replaces the $O(|\vocab|^{-1})$ error
        of \cref{lem:iwl-unhelpful,lem:subdist-agree} (and of \cref{thm:icl-rare-tokens})
        by $O\!\big(|\vocab|^{-1} + \seqPairs\, \freqC/|\vocabC|\big)$. This term does
        not vanish in the orthonormal idealization; (A4) removes it by fiat, and it is
        the only place where the finite common vocabulary enters the bound.
    \end{remark}

    \paragraph{Combining the two parts.}
    Summing the per-sub-distribution contributions weighted by their frequencies and
    substituting \cref{lem:iwl-solvable,lem:iwl-unhelpful} ($O(\epsilon_{\mathrm{IWL}})$ on
$\mathcal{D}_{c,\mathrm{can}}$, unchanged-up-to-$O(|\vocab|^{-1})$ on
$\mathcal{D}_{c,\mathrm{scr}}$, exactly unchanged on $\mathcal{D}_r$),
    \begin{subequations}
    \begin{align}
        \nabla_\mathrm{ICL} \loss\big(\weightsAll + \weightsAll_\mathrm{IWL}\big)
        \; & =\;
        \freqR \, \nabla_\mathrm{ICL} \loss_r\big(\weightsAll\big)
        \\
           & \quad
        \;+\; (1-\freqR)(1-\freqWCV) \, O(\epsilon_{\mathrm{IWL}})
        \nonumber \\
           & \quad
        \;+\; (1-\freqR)\, \freqWCV \, \Big( \nabla_\mathrm{ICL} \loss_{c,\mathrm{scr}}\big(\weightsAll\big) + O\!\big(|\vocab|^{-1}\big) \Big)
        \nonumber \\
           & \;=\;
        \big( \freqR + (1-\freqR)\, \freqWCV \big)
        \, \nabla_\mathrm{ICL} \loss(\weightsAll)
        \;+\; O(\epsilon_{\mathrm{IWL}}) \;+\; O\!\big(|\vocab|^{-1}\big),
    \end{align}
    \end{subequations}
    where the final step used \cref{lem:subdist-agree} (equality of the IWL-free
    per-sub-distribution gradients). Taking expectations gives the statement of
    \cref{thm:icl-rare-tokens}: the ICL gradient is suppressed by the data-coverage
    factor $\freqR + (1-\freqR)\freqWCV$, exactly in the saturated-IWL,
    large-vocabulary limit and up to the two stated error terms otherwise. Without the
    disjointness assumption (A4), the $O(|\vocab|^{-1})$ term is replaced by
    $O\!\big(|\vocab|^{-1} + \seqPairs\,\freqC/|\vocabC|\big)$, per
    \cref{rem:finite-vocab-crosstalk}.
    \hfill$\square$

    \subsection{Proof of \texorpdfstring{\cref{thm:icl-burstiness}}{Theorem~3}}
    \label{sec:proof-burstiness}

    \thmICLburstiness*

    The proof rests on a single observation: the ICL gradient decomposes
    into a sum of \emph{per-block} contributions in which every relevant
    block carries the same expected signal and every distractor block
    carries zero. With $\burstiness$ identical relevant copies in context,
    the resulting sum is $\burstiness$ times the corresponding
$\burstiness = 1$ quantity.

    \paragraph{Setup.}

    We work in the simplified two-attention-layer architecture of
    \cref{sec:proof-simplified} (single head per layer, fixed orthogonal
    output--value matrices $\values{\ell}$, merged key--query matrices
    $\weights{\ell}$). The all-rare regime $\freqR = 1$ is assumed; the $(\freqR +
(1 - \freqR)\freqWCV)$ pre-factor of \cref{thm:icl-rare-tokens} is independent
    of $\burstiness$ and multiplies the conclusion below without altering its
    derivation. The extension to deeper or multi-head models follows from
    \cref{sec:multi-head,sec:circuit-folding}.

    A single training sequence has $\seqPairs$ blocks of size $\tupleSize = 2$
    followed by the query item at the last position $\queryPos = N$. Block $r \in
[\seqPairs]$ has its item at position $i_r = 2r - 1$ and its label at position
$j_r = 2r$. The block indices split into the \emph{relevant} set $\mathcal{R} =
\{r : \token_{i_r} = \itm_q\}$ of size $|\mathcal{R}| = \burstiness$ and the
    \emph{distractor} set $\mathcal{D} = [\seqPairs] \setminus \mathcal{R}$, whose
    items and labels are drawn from the rare vocabulary independently of $(\itm_q,
\lbl_q)$. We invoke the never-repeat (lexinvariance) assumption of
    \cref{sec:data-properties}: distractor tokens are drawn from the effectively
    infinite rare vocabulary, so no distractor item coincides with $\itm_q$. (With
    a finite rare vocabulary $\vocabR$ this would fail with probability
    $\sim 1/|\vocabR|$ per block, contributing an extra $O(\seqPairs/|\vocabR|)$
    term; lexinvariance removes it.)

    \paragraph{Layer-$2$ attention at the query.}

    Write $s^{(\ell)}_{ik}$ for the pre-softmax attention logit from position $i$
    to position $k$ at layer $\ell$, and $A^{(\ell)}_{ik} = e^{s^{(\ell)}_{ik}} /
    Z^{(\ell)}_i$ for the corresponding attention weight, with $Z^{(\ell)}_i =
    \sum_k e^{s^{(\ell)}_{ik}}$. The query at position $N$ reads out through the
    layer-$2$ denominator
    \begin{equation}
        Z_q \;\coloneqq\; Z^{(2)}_N \;=\; \sum_{k} e^{s^{(2)}_{N k}}
        \;=\; \underbrace{\textstyle\sum_{r \in \mathcal{R}} e^{s^{(2)}_{N j_r}}}_{\burstiness\;\text{relevant keys}}
        \;+\; \underbrace{\textstyle\sum_{r \in \mathcal{D}} e^{s^{(2)}_{N j_r}}}_{\Theta(\seqPairs)\;\text{distractor keys}} ,
        \label{eq:burst-Zq}
    \end{equation}
    which couples all blocks. At a fixed ICL configuration the relevant logits
    take a common value $s_\mathrm{rel}$ and the distractor logits average to
    $\bar s_\mathrm{dis}$, so $Z_q = \burstiness\,e^{s_\mathrm{rel}} +
    (\seqPairs - \burstiness)\,e^{\bar s_\mathrm{dis}} = \Theta(\seqPairs)$ when
    $\seqPairs \to \infty$ with $\burstiness$ fixed. Consequently each relevant
    attention weight $A^{(2)}_{N j_r} = e^{s_\mathrm{rel}}/Z_q = \Theta(1/\seqPairs)$,
    and, since every per-block signal below carries exactly one such factor, the
    gradient is of order $\burstiness/\seqPairs$. This is precisely why the error
    must be controlled \emph{relative} to the leading term.

    The ICL direction parameterizes the canonical induction circuit:
    \begin{equation}
        \weightsICL{1} \,\propto\, \posCorrel,
        \qquad
        \weightsICL{2} \,\propto\, \tp{\values{1}}\,\idT,
        \qquad
        \weightsOutICL \,\propto\,
        \idT\,\tp{\values{2}}.
    \end{equation}
    The ICL gradient is the directional derivative of the loss along this
    tuple under the Frobenius inner product
$\langle \mathbf{A}, \mathbf{B}\rangle = \mathrm{tr}(\tp{\mathbf{A}}\mathbf{B})$:
    \begin{equation}
        \nabla_\mathrm{ICL}\loss
        \;=\;
        \big\langle \weightsICL{1}, \grad{\weights{1}}\big\rangle
        \;+\;
        \big\langle \weightsICL{2}, \grad{\weights{2}}\big\rangle
        \;+\;
        \big\langle \weightsOutICL, \grad{\weightsOut}\big\rangle .
    \end{equation}

    \paragraph{Per-block decomposition.}

    Each of the three components decomposes as a sum over context blocks,
    \begin{equation}
        \nabla_\mathrm{ICL}\loss
        \;=\; \sum_{r=1}^{\seqPairs}
        \Big( g^{(1)}_r \,+\, g^{(2)}_r \,+\, g^{(\mathrm{o})}_r \Big),
        \label{eq:burst-decomp}
    \end{equation}
    with per-block signals
    \begin{align}
        g^{(2)}_r          & \;=\;
        \frac{\partial \loss}{\partial s^{(2)}_{N j_r}}\,
        \big\langle \hiddenlast{1}, \;
        \weightsICL{2}\,\hiddeni{1}{j_r}\big\rangle ,
        \\
        g^{(1)}_r          & \;=\;
        \frac{\partial \loss}{\partial s^{(1)}_{j_r i_r}}\,
        \big\langle \hiddeni{0}{j_r}, \;
        \weightsICL{1}\,\hiddeni{0}{i_r}\big\rangle ,
        \\
        g^{(\mathrm{o})}_r & \;=\;
        A^{(2)}_{N\, j_r}\,
        \big\langle \outputPredicted - \modelOutput,\;
        \weightsOutICL\,\values{2}\,\hiddeni{1}{j_r}\big\rangle .
    \end{align}
    Three observations justify the decomposition. \emph{(i)} Only the
    residual stream at the last position feeds the readout, so
$\partial \loss / \partial s^{(2)}_{ij}$ is nonzero only on row
$i = N$. \emph{(ii)} $\weightsICL{1} \propto \posCorrel$
    shifts positions by one, so its inner product with
$\grad{\weights{1}}$ is supported on adjacent within-block pairs
$(j_r, i_r)$. \emph{(iii)} The residual stream at $N$ decomposes
    additively over context positions through the layer-$2$ attention
    sum, and within each block the only term that contributes through
$\weightsOutICL$ is the one at the label position $j_r$.

    \paragraph{Two facts about the per-block signals.}

    The per-block signals satisfy a clean dichotomy.

    \textbf{Distractor blocks contribute zero in expectation.} For every
$r \in \mathcal{D}$ and every $\bullet \in \{1, 2, \mathrm{o}\}$, the
    inner product inside $g^{(\bullet)}_r$ contracts the query embedding
$\emb{\itm_q}$ against an embedding of $\token_{i_r}$ or
$\token_{j_r}$, both of which are rare tokens drawn independently of
$\itm_q$. By the same orthogonality argument used to exclude
    cross-concept gradient blocks in \cref{sec:invariant-space-proof},
    distinct rare-token inner products vanish in expectation, so
    \begin{equation}
        \EV\big[g^{(\bullet)}_r \,\big|\, r \in \mathcal{D}\big] \;=\; 0.
    \end{equation}

    \textbf{Relevant blocks contribute identical signals.} For every
$r \in \mathcal{R}$ the item is $\itm_q$ and the label is $\lbl_q$
    by definition; the only quantity that varies with $r$ is the i.i.d.\
    block positional offset, which the analysis of
    \cref{sec:invariant-space-proof} shows to integrate out from the
    expected gradient. Consequently
    \begin{equation}
        \EV_{\dataDist(\burstiness)}\big[g^{(\bullet)}_r \,\big|\, r \in \mathcal{R}\big]
        \;=\; \tilde g^{(\bullet)}(\burstiness)
    \end{equation}
    takes a single value $\tilde g^{(\bullet)}(\burstiness)$ that does not depend on
    the choice of $r \in \mathcal{R}$. Each signal carries exactly one factor of the
    relevant layer-$2$ attention weight $A^{(2)}_{N j_r} = e^{s_\mathrm{rel}}/Z_q$:
    explicitly in $g^{(\mathrm{o})}_r$, and through the chain rule in
    $g^{(2)}_r$ (the softmax derivative $\partial\loss/\partial s^{(2)}_{N j_r} =
    A^{(2)}_{N j_r}\big(\partial_A\loss|_{j_r} - \langle A^{(2)}_{N\cdot},
    \partial_A\loss\rangle\big)$) and in $g^{(1)}_r$ (the layer-$1$ output
    $\hiddeni{1}{j_r}$ reaches the readout only via the query's layer-$2$ attention
    to $j_r$). We therefore write
    \begin{equation}
        \tilde g^{(\bullet)}(\burstiness) \;=\; \frac{\kappa^{(\bullet)}(\burstiness)}{Z_q(\burstiness)},
        \label{eq:burst-factor}
    \end{equation}
    isolating the dominant $\burstiness$-channel, the shared denominator $Z_q$ of
    \cref{eq:burst-Zq}. The residual numerator $\kappa^{(\bullet)}(\burstiness)$ is
    built from the query residual $\outputPredicted - \modelOutput$ and the softmax
    cross-term; it depends on $\burstiness$ only through the \emph{aggregate}
    relevant attention mass $\burstiness\,e^{s_\mathrm{rel}}/Z_q =
    \Theta(\burstiness/\seqPairs)$, which the query absorbs into $\modelOutput$.
    Hence $\kappa^{(\bullet)}(\burstiness) = \kappa^{(\bullet)}(1)\,\big(1 +
    O(\burstiness/\seqPairs)\big)$: the numerator is $\burstiness$-independent to
    leading order, with a correction of the same relative order as the denominator
    ratio computed below.

    \paragraph{Conclusion.}

    Taking the population expectation of \cref{eq:burst-decomp} under
    $\dataDist(\burstiness)$ and substituting the two facts, the
    $\seqPairs - \burstiness$ distractor blocks drop out and each of the
    $\burstiness$ relevant blocks contributes $\sum_{\bullet} \tilde g^{(\bullet)}(\burstiness)$;
    using the factorization \cref{eq:burst-factor},
    \begin{equation}
        \EV_{\dataDist(\burstiness)}\big[\nabla_\mathrm{ICL}\loss(\weightsAll)\big]
        \;=\;
        \burstiness \sum_{\bullet} \tilde g^{(\bullet)}(\burstiness)
        \;=\;
        \frac{\burstiness}{Z_q(\burstiness)} \sum_{\bullet} \kappa^{(\bullet)}(\burstiness).
        \label{eq:burst-exact}
    \end{equation}
    This decomposition is exact; the additive remainder of the original bound is
    gone. The residual $\burstiness$-dependence beyond the explicit prefactor sits
    in the two factors $Z_q(\burstiness)$ and $\kappa^{(\bullet)}(\burstiness)$,
    both of which we now control \emph{relative} to their $\burstiness = 1$ values.

    For the denominator, write $\bar e_\mathrm{dis} = e^{\bar s_\mathrm{dis}}$ for
    the mean distractor exponential; from \cref{eq:burst-Zq},
    \begin{equation}
        \frac{Z_q(1)}{Z_q(\burstiness)}
        \;=\;
        \frac{e^{s_\mathrm{rel}} + (\seqPairs - 1)\,\bar e_\mathrm{dis}}
             {\burstiness\, e^{s_\mathrm{rel}} + (\seqPairs - \burstiness)\,\bar e_\mathrm{dis}}
        \;=\;
        1 \;+\; \frac{(\burstiness - 1)\,(\bar e_\mathrm{dis} - e^{s_\mathrm{rel}})}
                     {Z_q(\burstiness)}
        \;=\;
        1 + O\!\big(\burstiness/\seqPairs\big),
    \end{equation}
    since $Z_q(\burstiness) = \Theta(\seqPairs)$ (distractor-dominated) and the
    numerator of the correction is $O(\burstiness)$. The numerator factor obeys the
    same relative bound, $\kappa^{(\bullet)}(\burstiness) =
    \kappa^{(\bullet)}(1)\,(1 + O(\burstiness/\seqPairs))$, by Fact~2. Specializing
    \cref{eq:burst-exact} to $\burstiness = 1$ gives
    $\EV_{\dataDist(1)}[\nabla_\mathrm{ICL}\loss(\weightsAll)] =
    \big(\sum_{\bullet}\kappa^{(\bullet)}(1)\big)/Z_q(1)$, so dividing the two
    instances of \cref{eq:burst-exact} and inserting both relative bounds yields
    \begin{equation}
        \begin{aligned}
            \EV_{\dataDist(\burstiness)}\big[\nabla_\mathrm{ICL}\loss(\weightsAll)\big]
            \;&=\;
            \burstiness\;
            \frac{Z_q(1)}{Z_q(\burstiness)}\;
            \frac{\sum_{\bullet}\kappa^{(\bullet)}(\burstiness)}{\sum_{\bullet}\kappa^{(\bullet)}(1)}\;
            \EV_{\dataDist(1)}\big[\nabla_\mathrm{ICL}\loss(\weightsAll)\big]
            \\[2pt]
            \;&=\;
            \burstiness\;
            \EV_{\dataDist(1)}\big[\nabla_\mathrm{ICL}\loss(\weightsAll)\big]
            \;\big(1 + O(\burstiness/\seqPairs)\big).
        \end{aligned}
    \end{equation}
    The remainder is now \emph{multiplicative}: it scales the leading term rather
    than adding to it, so it is genuinely lower-order and vanishes as
    $\seqPairs \to \infty$ with $\burstiness$ fixed, the distractor-dominated
    regime assumed by the theorem. The ICL gradient is therefore proportional to
    $\burstiness$ in this limit.
    \hfill$\square$

    \newpage
    \section{Competing Induction Heads}
    \label{sec:competing-induction-heads-appendix}

    This appendix supports \cref{sec:competing-induction-heads}, which studies the
    dependence of the winning ICL circuit on initialization in a 3-layer
    transformer.

    \subsection{Training details}
    \label{sec:competing-induction-heads-training}

    We train a 3-layer, 1-head-per-layer attention-only transformer of width $256$
    on the in-context bigram task of \cref{sec:icl-vs-iwl-training} restricted to
    the all-rare regime ($\freqR = 1.0$), so the only solution to the task is an
    in-context induction circuit. The vocabulary is $16$ common and $16$ rare
    tokens, sinusoidal positional embeddings with 4 frequencies (8 features), and
    each sequence consists of $4$ bigram demonstrations followed by a query
    (context length $9$); both burstiness and within-class variability are set to
    their neutral values ($\burstiness = 1$, $\freqWCV = 0$). We train the
    simplified architecture of \cref{sec:proof-simplified}, with merged key-query
    matrices, fixed output-value projections, and orthogonal representations. The
    width $256 = 32 \times (1 + H)^{L}$ with $H = 1$ and $L = 3$ is chosen so that
    the model exactly admits the geometric construction of all three competing
    induction circuits on the manifold (base dimension $32$, branching $(1+H)^{L} =
8$). Optimization uses vanilla SGD at learning rate $0.1$, micro-batch size
$64$, gradient-accumulation factor $2$ (effective batch $128$), for at most
$12{,}000$ steps, with early-stop once the training cross-entropy drops below
$0.5$. We constrain the dynamics to the IMIR by projecting $\weightsAll$ onto
    the seven-dimensional sub-manifold spanned by the named circuit units $\alpha,
\alpha', \beta, \beta', \beta'', \gamma, \gamma'$ after every SGD step. For the
    phase-transition scan of \cref{fig:icl-competition}, each seed samples
$(s_\beta, s_{\beta''}) \sim \mathrm{Uniform}(0,\,0.6)^{\!2}$ i.i.d., zeroes
    every other IMIR direction, and sets the initial strengths of $\beta$ and
$\beta''$ to $s_\beta$ and $s_{\beta''}$ respectively; the seed is then plotted
    at $(s_\beta, s_{\beta''})$ and colored by the winning circuit. The winning
    circuit is identified post-training by causal ablation on a fixed evaluation
    batch: we rebuild three reduced transformers from the trained IMIR coordinates,
    retaining only one of $\{\alpha, \beta, \gamma\}$, $\{\alpha, \beta',
\gamma'\}$, or $\{\alpha', \beta'', \gamma'\}$ at a time, and pick the one with
    the lowest cross-entropy at the query position. The same fixed batch is used
    across all seeds so the three ablation losses are directly comparable.

    \paragraph{Compute resources.}
    Each individual seed trains a 3-layer transformer of width $256$
    (${\sim}0.2\,$M learnable parameters) for at most $12{,}000$ SGD steps with
    effective batch size $128$, taking a few minutes on a single GPU and using
    under $0.5\,$GB of GPU memory. The phase-transition scan parallelizes trivially
    across seeds: we run it on an 8-GPU node via per-seed sharding, with one
    persistent worker process per GPU pulling seeds from a shared queue. The scan
    reported in \cref{fig:icl-competition} aggregates several thousand seeds and
    completes in a few hours of wall-clock on the 8-GPU node.

    \newpage
    \subsection{Training Dynamics}
    \label{sec:competing-induction-heads-dynamics}

    \begin{figure}[h]
        \centering
        \captionsetup{width=0.9\linewidth}
        \includegraphics[width=\linewidth]{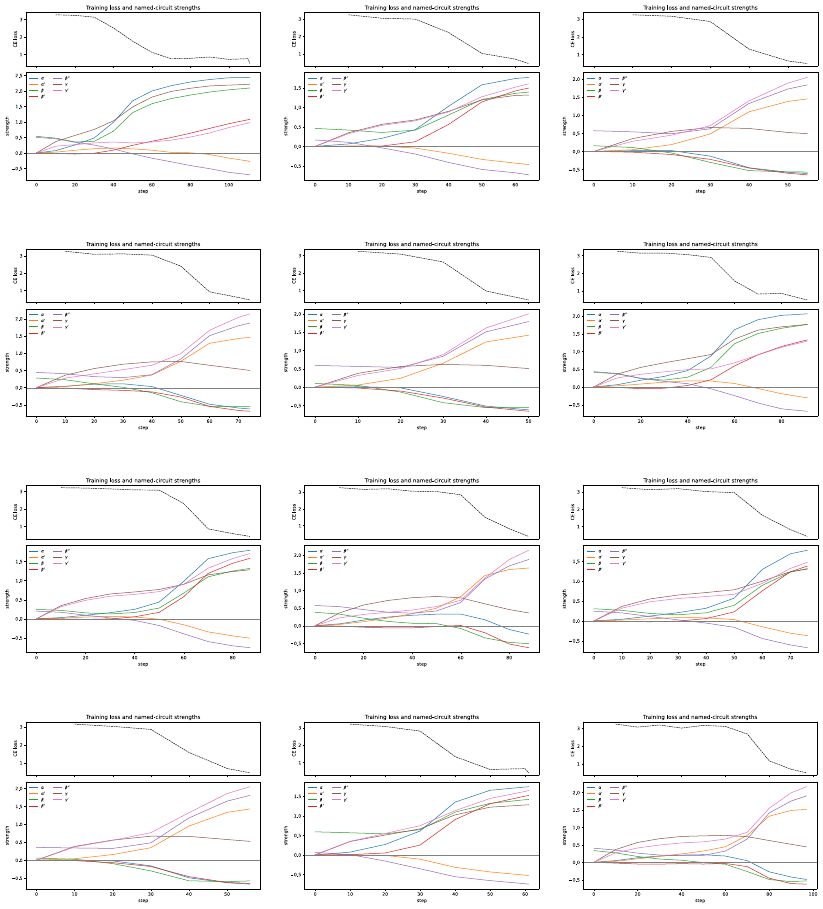}
        \caption{Training dynamics of 12 random seeds of a 3-layer, single-head Transformer trained on in-context bigrams (cf.\ \cref{sec:competing-induction-heads}). For each seed, the top panel shows the training cross-entropy and the bottom panel the strengths of the seven named circuit units ($\alpha$, $\alpha'$, $\beta$, $\beta'$, $\beta''$, $\gamma$, $\gamma'$); the units that grow together indicate which of the three competing induction circuits wins.}
        \label{fig:winning-tickets}
    \end{figure}

    \newpage
    \section{Towards Automated Interpretability}
    \subsection{Training Details}
    \label{sec:autointerp-details}

    We train a 5-layer, 1-head-per-layer attention-only transformer of width $1024$
    on the $k$-hop induction task with $k = 2$ hops, vocabulary of $16$ common and
$32$ rare tokens, sinusoidal positional embeddings with 4 frequencies (8
    features), and 9 context blocks per sequence. The all-rare regime $\freqR =
1.0$ is used; an implicit-curriculum loss samples the target uniformly from the
$k$ hop endpoints $\{t_1, \ldots, t_k\}$ along the query chain, and
    within-chain blocks are emitted in reverse-hop order so that the target-bearing
    hop sits furthest from the query position. The width $1024 = 32 \times (1 +
H)^{L}$ with $H = 1$ and $L = 5$ admits the orthogonal-subspace construction of
    all IMIR directions exactly. Optimization uses AdamW at learning rate
$10^{-3}$, micro-batch size $1$, gradient-accumulation factor $32$ (effective
    batch $32$), for $50{,}000$ optimizer steps; the random seed is fixed at $6$.

    \paragraph{Compute resources.}
    The model has approximately $32\,$M learnable parameters; the bulk of the cost
    comes from the layer-wise attention matrices over the relatively long context.
    We launch training via \texttt{torchrun} on a single 8-GPU node, with the model
    wrapped in DistributedDataParallel and gradients all-reduced on the final
    micro-batch of each accumulation window. One full training run (50{,}000
    optimizer steps, effective batch $32 \cdot 8 = 256$ across ranks) takes a few
    hours of wall-clock on this hardware. The circuit-detection routine of
    \cref{sec:autointerp-algo} runs post-training on a single GPU and completes in
    minutes: each greedy round evaluates one forward pass per remaining IMIR
    direction on a fixed eval batch, so the total cost is quadratic in the number
    of non-zero IMIR directions of the trained model.

    \newpage
    \subsection{Circuit Detection Algorithm}
    \label{sec:autointerp-algo}

    Our circuit-detection routine performs greedy backward elimination on the
    manifold coordinates of the trained model. Starting from the full set of IMIR
    directions, it repeatedly drops the single direction whose ablation increases
    the loss the least, and stops when even the cheapest available ablation would
    push the loss on a fixed evaluation batch above a user-supplied budget $\tau$.
    The surviving directions, sorted by their per-circuit ablation cost, are the
    algorithm's identification of the model's essential circuits.

    \begin{algorithm}[h]
        \caption{Greedy backward circuit elimination on the IMIR.}
        \label{alg:auto-interp}
        \begin{algorithmic}[1]
            \Require trained transformer $f_\weightsAll$, fixed eval batch
            $\mathcal{B}$, loss budget $\tau$, free-ablation slack
            $\rho \ge 1$
            \Ensure essential circuit indices $\mathcal{A}$ and strengths
            $\{s_i\}_{i \in \mathcal{A}}$
            \Statex
            \State $\{(\mathbf{e}_i, s_i)\}_i \gets
                \Call{ReduceToManifold}{\weightsAll}$
            \Comment{basis directions on $\manifold$ and their strengths}
            \State $\weightsAll \gets \sum_i s_i\, \mathbf{e}_i$
            \Comment{project the trained weights onto $\manifold$}
            \State $\mathcal{A} \gets \{i : s_i \ne 0\}$;\;\;
            $\loss_{\mathrm{cur}} \gets \loss(f_\weightsAll;\, \mathcal{B})$
            \Statex
            \While{$\loss_{\mathrm{cur}} < \tau$ \textbf{and}
                $\mathcal{A} \ne \emptyset$}
            \State $i_\star \gets \textsc{None}$;\;\;
            $\loss_\star \gets +\infty$
            \ForAll{$i \in \mathcal{A}$ in random order}
            \State $\loss_i \gets
                \loss\big(f_{\weightsAll - s_i \mathbf{e}_i};\, \mathcal{B}\big)$
            \Comment{tentative ablation of direction $i$}
            \If{$\loss_i < \loss_\star$}
            \State $\loss_\star \gets \loss_i$;\;\;
            $i_\star \gets i$
            \EndIf
            \If{$\loss_\star \le \rho \cdot \loss_{\mathrm{cur}}$}
            \State \textbf{break}
            \Comment{free ablation found; skip rest of scan}
            \EndIf
            \EndFor
            \If{$\loss_\star \ge \tau$}
            \State \textbf{break}
            \Comment{even the cheapest ablation crosses the budget}
            \EndIf
            \State $\weightsAll \gets \weightsAll - s_{i_\star}\,
                \mathbf{e}_{i_\star}$
            \Comment{commit the ablation}
            \State $\mathcal{A} \gets
                \mathcal{A} \setminus \{i_\star\}$;\;\;
            $\loss_{\mathrm{cur}} \gets \loss_\star$
            \EndWhile
            \Statex
            \State \Return $\mathcal{A}$, sorted by per-circuit ablation cost
            $\loss(f_{\weightsAll - s_i \mathbf{e}_i};\, \mathcal{B}) -
                \loss_{\mathrm{cur}}$ in descending order.
        \end{algorithmic}
    \end{algorithm}

    The free-ablation slack $\rho$ is a small optimization: once the inner scan
    finds an ablation whose cost is within a factor $\rho$ of the current loss, we
    accept it without finishing the round. This keeps each round at
$O(|\mathcal{A}|)$ in the worst case but typically much cheaper when several
    directions are obviously prunable. The randomized scan order avoids the bias of
    always probing the same prefix first. The fixed evaluation batch $\mathcal{B}$
    and the deterministic RNG seed used by the embedding layer ensure that two
    evaluations of the same model produce identical losses, which is required for
    the greedy comparisons to be meaningful. We use the values $\tau = 100$ and
$\rho = 1.01$ in our experiments.

    \newpage
    \subsection{Discovered Circuits}
    \label{sec:autointerp-circuits}

    We run the greedy backward-elimination procedure of \cref{alg:auto-interp} on
    two trained 5-layer transformers solving the $k = 2$ hop induction task, and
    report the discovered circuits in
    \cref{tab:autointerp-circuit1,tab:autointerp-circuit2}. In addition to the four
    units one would expect from a canonical induction-head construction --- a
    previous-token head, two item-match heads chained through value rotations, and
    a readout that inverts the chain --- the algorithm consistently recovers a
    small number of \emph{aiding} units that further lift accuracy without being
    part of any standard induction template. We separate the two groups in the
    tables for clarity.

    The aiding units are interesting on their own: removing them drops the isolated
    subcircuit's accuracy from $\sim 0.92$ down to $\sim 0.6$
    (\cref{tab:autointerp-circuit1}, $0.922 \to 0.594$), which is much more than a
slight refinement. They appear to compose with the canonical induction units to
form a slightly redundant ensemble that exploits multiple paths through the
residual stream simultaneously. This kind of structure would be invisible to
existing automated circuit discovery algorithms with granularity at the
attention head level. We list the four canonical units under \emph{Essential}
and the recovered extras under \emph{Aiding} in each circuit's table.

\begin{table}[h]
    \centering
    \caption{Discovered subcircuit $C_1$ in a trained 5-layer transformer
        ($k = 2$ hops). The accuracy table (left) reports top-1 accuracy on
        a fixed evaluation batch at four levels of restriction; the
        circuit-unit table (right) lists the seven IMIR directions kept by
        the greedy elimination, separated into the four canonical induction
        units (\emph{Essential}) and three non-canonical helpers
        (\emph{Aiding}).}
    \label{tab:autointerp-circuit1}
    \begin{minipage}[t]{0.30\linewidth}
        \vspace{0pt}
        \centering
        \begin{tabular}{@{}lc@{}}
            \toprule
            \textbf{Configuration}      & \textbf{Acc.} \\
            \midrule
            Trained model               & $0.938$       \\
            Full IMIR ($1428$ dirs.)    & $0.953$       \\
            Recovered $C_1$ ($7$ units) & $0.922$       \\
            Essentials ($4$ units)      & $0.594$       \\
            \bottomrule
        \end{tabular}
    \end{minipage}
    \hfill
    \begin{minipage}[t]{0.65\linewidth}
        \vspace{0pt}
        \centering
        \begin{tabular}{@{}cl l r@{}}
            \toprule
            \textbf{Layer} & \textbf{Direction}                           & \textbf{Type} & \textbf{Strength} \\
            \midrule
            $1$            & $\posCorrel$                                 & Essential     & $+2.94$           \\
            $3$            & $\idT\,(\values{1,1})^+$                     & Essential     & $+9.24$           \\
            $4$            & $\values{3,1}\,\posCorrel\,(\values{1,1})^+$ & Essential     & $+3.68$           \\
            out            & $\idT\,(\values{3,1}\,\values{4,1})^+$       & Essential     & $+1.42$           \\
            \midrule
            $2$            & $\idT\,(\values{1,1})^+$                     & Aiding        & $+3.81$           \\
            $3$            & $\values{1,1}\,\values{2,1}\,\idP$           & Aiding        & $-1.33$           \\
            $4$            & $\values{3,1}\,\idT$                         & Aiding        & $+1.14$           \\
            \bottomrule
        \end{tabular}
    \end{minipage}
\end{table}

\begin{table}[h]
    \centering
    \caption{Discovered subcircuit $C_2$ in a second trained 5-layer
        transformer ($k = 2$ hops). Compared with $C_1$, $C_2$ lives
        entirely on layers $1, 2, 3$, uses canonical token-identity
        directions on every essential unit, and recovers a single aiding
        unit.}
    \label{tab:autointerp-circuit2}
    \begin{minipage}[t]{0.30\linewidth}
        \vspace{0pt}
        \centering
        \begin{tabular}{@{}lc@{}}
            \toprule
            \textbf{Configuration}      & \textbf{Acc.} \\
            \midrule
            Trained model               & $0.906$       \\
            Full IMIR ($1428$ dirs.)    & $0.906$       \\
            Recovered $C_2$ ($5$ units) & $0.844$       \\
            Essentials ($4$ units)      & $0.719$       \\
            \bottomrule
        \end{tabular}
    \end{minipage}
    \hfill
    \begin{minipage}[t]{0.65\linewidth}
        \vspace{0pt}
        \centering
        \begin{tabular}{@{}cl l r@{}}
            \toprule
            \textbf{Layer} & \textbf{Direction}                     & \textbf{Type} & \textbf{Strength} \\
            \midrule
            $1$            & $\posCorrel$                           & Essential     & $+2.35$           \\
            $2$            & $\idT\,(\values{1,1})^+$               & Essential     & $+2.68$           \\
            $3$            & $\idT\,(\values{1,1})^+$               & Essential     & $+5.79$           \\
            out            & $\idT\,(\values{2,1}\,\values{3,1})^+$ & Essential     & $+2.94$           \\
            \midrule
            $2$            & $\idT$                                 & Aiding        & $-0.47$           \\
            \bottomrule
        \end{tabular}
    \end{minipage}
\end{table}

\end{document}